\documentclass[onefignum,onetabnum]{siamart171218}

\usepackage{microtype}
\usepackage{graphicx}
\graphicspath{{../}}
\usepackage{wrapfig}
\usepackage{subcaption}
\usepackage{booktabs} 
\usepackage[table]{xcolor}

\usepackage[noend]{algorithmic}

\usepackage[utf8]{inputenc} 
\usepackage[T1]{fontenc}    
\usepackage{hyperref}       
\usepackage{url}            
\usepackage{amsfonts}       
\usepackage{nicefrac}       

\usepackage{amsmath}
\usepackage{amssymb}
\usepackage{mathtools}
\usepackage{natbib}
\usepackage{enumitem}

\usepackage{etoc}
\etocsettagdepth{main}{none}
\etocsettagdepth{appendix}{subsection}
\etocsettocstyle{\noindent{\Large\bfseries Appendix contents}\par\medskip}{}
\makeatletter
\etocsetstyle{section}{}{}{%
  \begingroup\bfseries
  \@dottedtocline{1}{0em}{0em}{\etocname}{\etocpage}%
  \endgroup}{}
\etocsetstyle{subsection}{}{}{%
  \@dottedtocline{2}{1.5em}{2.8em}%
    {\numberline{\etocnumber}\etocname}{\etocpage}}{}
\makeatother

\definecolor{mypink}{rgb}{0.858, 0.188, 0.478}
\definecolor{mygreen}{RGB}{0, 128, 0}
\definecolor{myblue}{RGB}{0, 0, 255}
\definecolor{myred}{RGB}{255, 0, 0}
\definecolor{mybrown}{RGB}{150, 75, 0}
\definecolor{myorange}{RGB}{255,69,0}
\definecolor{mypurple}{RGB}{128,0,128}
\definecolor{cbpink}{RGB}{220,38,127}
\definecolor{cborange}{RGB}{254,97,0}
\definecolor{cbpurple}{RGB}{120,94,240}

\newsiamthm{assumption}{Assumption}
\newsiamremark{remark}{Remark}

\makeatletter
\AtBeginDocument{%
  \def\refstepcounter@optarg[#1]#2{%
    \cref@old@refstepcounter{#2}%
    \cref@constructprefix{#2}{\cref@result}%
    \protected@edef\cref@currentlabel{%
      [#2][\arabic{#2}][\cref@result]%
      \csname p@#2\endcsname\csname the#2\endcsname}%
    \@ifundefined{cref@#1@alias}%
      {\def\@tempa{#1}}%
      {\def\@tempa{\csname cref@#1@alias\endcsname}}%
    \protected@edef\cref@currentlabel{%
      \expandafter\cref@override@label@type%
        \cref@currentlabel\@nil{\@tempa}}}%
}
\makeatother

\newcommand{\pspace}{\mathcal{P}_2(\mathcal{M})}
\newcommand{\ppspace}{\mathcal{P}_2(\mathcal{P}_2(\mathcal{M}))}
\newcommand{\dvolm}{d\mathrm{vol}} 

\usepackage{xparse}

\definecolor{rankone}{HTML}{9DC183}   
\definecolor{ranktwo}{HTML}{88C1D7}   
\definecolor{rankthree}{HTML}{FFBB7D} 
\newcommand{\rA}[1]{\cellcolor{rankone}#1}
\newcommand{\rB}[1]{\cellcolor{ranktwo}#1}
\newcommand{\rC}[1]{\cellcolor{rankthree}#1}
\newcommand{\lgd}[2]{\textcolor{#1}{#2}}

\usepackage[textsize=tiny]{todonotes}

\title{When Riemann flows with Wasserstein:\\ Generative Modeling of Probability Distributions on Manifolds}

\author{Doron Haviv$^{*,1}$%
\and Edward De Brouwer$^{*,1}$%
\and Rishabh Anand$^{2}$%
\and Rex Ying$^{2}$%
\and Aïcha Bentaieb$^{1}$%
\and Gabriele Scalia$^{1}$%
\and Hector Corrada Bravo$^{1}$%
}

\begin{document}
\etocdepthtag.toc{main}

\maketitle

\renewcommand{\thefootnote}{\fnsymbol{footnote}}
\footnotetext[1]{Equal contribution.}
\renewcommand{\thefootnote}{\arabic{footnote}}
\footnotetext[1]{Genentech Inc., South San Francisco, CA. \texttt{havivd@gene.com, debroue1@gene.com}.}
\footnotetext[2]{Yale University, New Haven, CT.}

\begin{abstract}
Many scientific datasets, such as molecular conformational ensembles or single-cell tissue measurements, are naturally modeled as meta-distributions: distributions over probability measures on non-Euclidean domains. Existing generative methods largely assume Euclidean geometry and fail to capture this structure. We introduce Riemannian Wasserstein Entropic Flow Matching (RWEFM), a generative framework on the Wasserstein space $\mathcal{P}_2(\mathcal{M})$ of a Riemannian manifold $(\mathcal{M},g)$. RWEFM is trained by regressing a neural vector field onto Riemannian optimal transport velocities, using McCann displacement interpolations as conditional paths. We confirm theoretically that this construction leads to a valid flow matching approach on $\mathcal{P}_2(\mathcal{M})$ and introduce the Riemannian Entropic Map, a GPU-efficient approximation of the optimal transport map on manifolds. Our experiments show that by respecting the intrinsic geometry of the data, RWEFM can generate whole single-cell samples in hyperspherical latent spaces and protein conformational ensembles on the torus. As RWEFM requires only a geodesic distance and a projection operator, it is not restricted to manifolds with closed-form geometry, which we demonstrate by generating distributions on a general triangulated mesh. Code and tutorials are available at \href{https://github.com/DoronHav/WassersteinFlowMatching}{RWEFM}.
\end{abstract}

\section{Introduction}
\label{introduction}

Many modern scientific datasets are most naturally represented as collections of \emph{distributions} of data points living on structured non-Euclidean spaces. Molecular conformational ensembles~\citep{axelrod2022geom} describe each molecule as a distribution over rotations and translations in 3D space; climate records~\citep{abatzoglou2018terraclimate} can be viewed as distributions of weather variables over the sphere; and tissue-level single-cell RNA-seq~\citep{czi2025cz} represents each sample as an empirical distribution of cell states living on a biologically meaningful manifold (\emph{e.g.} spherical or hyperbolic). In these settings, the object of interest is a distribution over probability measures on a manifold.

Generative modeling provides a principled way to summarize such data and to enable scientific tasks, including in-silico design~\citep{gruver2023protein}, hypothesis testing \citep{candes2018panning}, and causal discovery in the underlying physical process~\citep{zhu2023sample}. While recent generative models have achieved impressive results in Euclidean data (images \citep{lin2024sdxl}, videos \citep{jin2025pyramidal}, single-cell \citep{klein2025cellflow}), and have increasingly incorporated non-Euclidean structures (\emph{e.g.} molecules \citep{schneuing2024structure}), most of this progress concerns generating individual data points defined as finite-dimensional vectors. In contrast, generative modeling of \emph{distributions} requires operating in an infinite-dimensional space: the Wasserstein space of probability measures $\mathcal{P}_2(\mathcal{M})$. Recent work has explored generative models directly on Wasserstein space~\citep{haviv2024wasserstein,atanackovic2024meta,piening2026generalized} but these methods assume the underlying domain to be Euclidean (\emph{i.e.} $\mathcal{M}\equiv \mathbb{R}^d$), thereby failing to faithfully model datasets whose support lies on curved geometries.

\newpage
\begin{wrapfigure}{r}{0.46\textwidth}
\centering
\includegraphics[width=\linewidth]{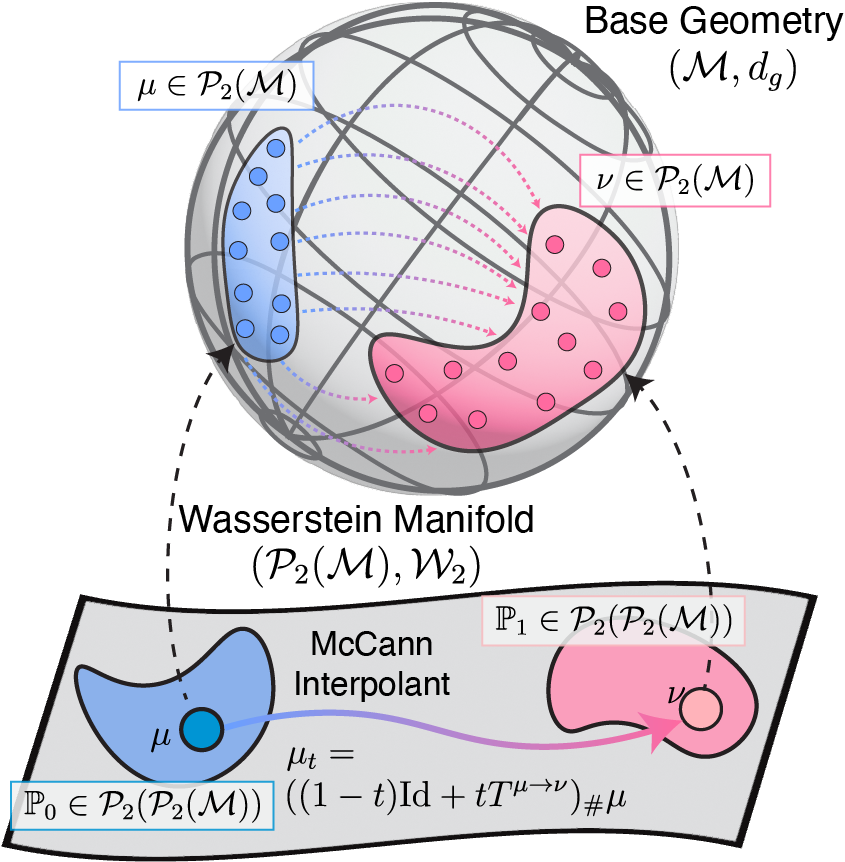}
\caption{\textbf{Riemannian Wasserstein Entropic Flow Matching.} \textbf{Top:} Base Riemannian manifold $(\mathcal{M}, d_g)$. Each blob is a probability
distribution $\mu,\nu \in \mathcal{P}_2(\mathcal{M})$ and a single training example is one blob, not one point. The dashed line is the Riemannian OT path between $\mu$ and $\nu$ which is estimated via point-cloud samples from each blob using the entropic map. \textbf{Bottom:} Wasserstein manifold $(\mathcal{P}_2(\mathcal{M}), \mathcal{W}_2)$, where a \textit{blob} is a distribution-over-distributions and each point is a distribution in the top figure. The same dashed path lifts here as the McCann interpolant $\mu_t$, the geodesic along which RWEFM learns its flow.}
\vspace{-4mm}
\end{wrapfigure}

In this work, we introduce Riemannian Wasserstein Entropic Flow Matching (RWEFM), a flow-matching (FM) framework \citep{lipman2022flow} for generative modeling on $\mathcal{P}_2(\mathcal{M})$. RWEFM lifts conditional FM to the Wasserstein space using McCann displacement interpolations as conditional paths between paired measures $(\mu,\nu)$. We establish that this yields a valid FM objective on $\mathcal{P}_2(\mathcal{M})$ and that the induced marginal dynamics satisfy a weak continuity equation and admit a path-space representation via the superposition principle \citep{pinzi2025nested}.

A key practical bottleneck is the computation of optimal transport maps between empirical distributions on Riemannian manifolds. To this end, we introduce the Riemannian Entropic Map, a GPU-efficient approximation of the Monge map that extends the entropic map~\citep{pooladian2021entropic} to general geometries. In essence, our estimator is obtained as the exponential map of the barycentric projection of tangent vectors. We further provide an error bound that quantifies the approximation induced by regularization and finite sampling.

We evaluate RWEFM across a range of geometries and scientific applications. This includes generation of digits, letters and Kanji characters on the sphere $\mathbb{S}^2$, hyperbolic disk $\mathbb{H}^2$ and torus $\mathbb{T}^2$, as well as whole single-cell sample generation in hyperspherical latent spaces $\mathbb{S}^{128-1}$ and protein conformational ensembles on $\mathbb{T}^2$. To emphasize that our framework is not limited to closed-form geometries, we further generate distributions on a general triangulated mesh. Our experiments show that respecting the intrinsic geometry of the data and using the Riemannian Entropic Map improves generation quality.

\paragraph{Contributions} \textbf{(i)} We propose RWEFM, a framework for generative modeling on the space of probability distributions on Riemannian manifolds---including general geometries without closed-form geodesics, such as triangulated meshes---and show that the resulting marginal dynamics are theoretically well-posed. \textbf{(ii)} We introduce the Riemannian Entropic Map, a fast and accurate approximation of Riemannian optimal transport maps suitable for large-scale GPU training, together with theoretical guarantees. \textbf{(iii)} We benchmark RWEFM against FM baselines on real-world datasets such as generation of single-cell samples and protein conformational ensembles.

\newpage
\section{Background and Related Work}
\label{sec:background}

\subsection{Optimal Transport on Riemannian Manifolds}

Optimal Transport (OT) \citep{villani2008optimal} provides a geometric framework for comparing probability distributions on Riemannian manifolds. Let $(\mathcal{M}, g)$ be a complete Riemannian manifold equipped with a metric $g$ and its corresponding geodesic distance $d_g(\cdot, \cdot)$, and let $\mu, \nu \in \mathcal{P}_{2}(\mathcal{M})$ be two probability measures with finite second moments. The Monge formulation seeks a deterministic map $T: \mathcal{M} \to \mathcal{M}$ pushing $\mu$ onto $\nu$ (denoted $T_\#\mu = \nu$) that minimizes the total transportation cost:
\begin{equation}
\label{eq:monge}
    \inf_{T} \left\{ \int_{\mathcal{M}} c(x, T(x)) d\mu(x) : T_\#\mu = \nu \right\},
\end{equation}

where $c(x,y) = \frac{1}{2}d_g^2(x,y)$ is the cost of moving a unit of mass from $x$ to $y$. Given mild assumptions, such as $\mu$ being absolutely continuous with respect to the volume measure, the Brenier--McCann theorem~\citep{ambrosio2021lectures} states 
that there exists a unique optimal transport map $T_0$, known as the Monge map, of the form $T_0(x) = \text{exp}_{x}(-\nabla\varphi_0(x))$ where $\varphi_0: \mathcal{M} \to \mathbb{R}$ is a c-concave function.

The Kantorovich formulation relaxes the Monge problem by allowing for probabilistic couplings:
\begin{equation*}
    W_2^2(\mu, \nu) = \inf_{\pi \in \Pi(\mu, \nu)} \int_{\mathcal{M} \times \mathcal{M}} d_g^2(x,y) d\pi(x,y).
\end{equation*}
where $\Pi(\mu, \nu)$ is the set of joint distributions on $\mathcal{M} \times \mathcal{M}$ with marginals $\mu$ and $\nu$. The optimal value $W_2(\mu, \nu)$ defines the 2-Wasserstein distance between the measures. Unlike the Monge formulation, the Kantorovich problem does not require continuity of measures to admit a solution. When the Monge map $T_0$ exists, it can be recovered from the Kantorovich problem via the dual formulation: $\frac{1}{2}W_2^2(\mu, \nu) = \sup_{\varphi} \int_{\mathcal{M}} \varphi d\mu + \int_{\mathcal{M}} \varphi^c d\nu,$
with $\varphi^c(y) = \inf_{x\in\mathcal{M}} \{ \frac{1}{2}d(x,y)^2 - \varphi(x)\}$, and the Monge map is then $T_0(x) = \text{exp}_{x}(-\nabla \varphi_0(x))$.


\subsubsection{Statistical Estimation of Optimal Transport}

Closed-form solutions for optimal transport maps are typically limited to very simple distributions so we generally rely on statistical estimation from finite samples $\{x_i\}_{i=1}^m \sim \mu$ and $\{y_j\}_{j=1}^n \sim \nu$. However, the Kantorovich problem on discrete samples is a linear program with cubic complexity in the number of samples, hindering its application to large datasets. Instead, practical approaches rely on entropic OT, which regularizes the (discrete) objective with an entropy term $H(\pi)$. The problem is strictly convex and can be efficiently solved using the Sinkhorn algorithm \citep{cuturi2013sinkhorn}.
\begin{equation*}
    \min_{\pi \in \Pi(\hat\mu, \hat\nu)} \sum_{i=1}^m \sum_{j=1}^n c(x_i, y_j) \pi_{ij} + \varepsilon \sum_{i=1}^m \sum_{j=1}^n \pi_{ij} \log \pi_{ij}.
\end{equation*}
This problem admits the dual formulation:
\begin{equation}
\sup_{\substack{f\in L^1(\mu)\\ g\in L^1(\nu)}}\;
\sum_i f(x_i) + \sum_j g(y_j) -\varepsilon \sum_{i=1}^m \sum_{j=1}^n
  e^{\left(f(x_i)+g(y_j)-\frac12 d_g(x_i,y_j)^2\right)/\varepsilon} + \varepsilon
  \label{eq:dual_EOT}
\end{equation}
While entropic optimal transport is widely used for computing the optimal transport distances, \citet{pooladian2021entropic}~showed that it can be used to compute a tractable estimator of the Monge map $T_0$. Given the optimal entropic coupling $\pi_\varepsilon$, the \textit{entropic map} is the barycentric projection $T_\varepsilon(x_i) = \mathbb{E}_{\pi_\varepsilon}[Y | X = x_i]$ which, under suitable regularity assumptions and an appropriate joint choice of regularization and sample size, consistently estimates the Monge map $T_0$. The optimal entropic potentials $(f_{\varepsilon},g_{\varepsilon})$ are the solutions of Eq~\eqref{eq:dual_EOT}.

\subsubsection{Wasserstein Geometry}

The Wasserstein space over a Riemannian manifold $(\mathcal{M}, g)$, denoted $\mathcal{P}_2(\mathcal{M})$, is the space of probability measures on $\mathcal{M}$ with finite second moments, equipped with the 2-Wasserstein distance $W_2$. While not rigorously a Riemannian manifold due to infinite dimensionality, it can be endowed with a Riemannian-like structure. There, the tangent space at a measure $\mu \in \mathcal{P}_2(\mathcal{M})$ is identified with the closure of the set of gradients of smooth functions in $L^2(\mu,T\mathcal{M})$:
\begin{equation*}
    T_\mu \mathcal{P}_2(\mathcal{M}) = \overline{\{ v= \nabla \phi : \phi \in C_c^\infty(\mathcal{M}) \}}^{L^2(\mu)},
\end{equation*}
and endowed with the norm $\lVert v \rVert^2_{L^2(\mu)} = \int_{\mathcal{M}} \lVert v \rVert^2_g d\mu(x)$. The exponential and logarithm maps in this space are:  $\exp_\mu(v) := \exp_x(v(x))_{\#}\mu, \, \log_\mu(\nu) := -\nabla\varphi(x)\,$, where $T^{\mu\to\nu} = \exp_x(-\nabla\varphi(x))$ is the optimal transport map from $\mu$ to $\nu$. Geodesics in $\mathcal{P}_2(\mathcal{M})$ between measures $\mu$ and $\nu$ are probability paths $(\mu_t)_{t\in[0,1]}$ where:
\begin{equation*}
    \mu_t = \exp_\mu(t\log_\mu(\nu)) = \exp_x(-t\nabla\varphi(x))_{\sharp}\mu
\end{equation*}

\subsection{Flow Matching on Riemannian Manifolds and the Wasserstein Space}
\label{sec:background_RFM}

Flow matching generates samples from a target distribution $\nu$ from samples of a source distribution $\mu$ by learning a time-dependent vector field that generates a probability path $\mu_t$ such that $\mu_0 = \mu$ and $\mu_1=\nu$. Directly learning such a vector field is generally impossible. Instead, flow matching defines conditional probability paths $\mu_t(\cdot|z)$ for which computing a generating vector field is tractable.

\paragraph{Ingredients of flow matching}
This construction relies on three key components: (\textbf{C1}) a condition for when a vector field generates a probability path, (\textbf{C2}) the marginalization of the conditional vector field generates the marginal probability path, and (\textbf{C3}) the losses from regressing the learnable vector field to the conditional or the marginal vector field are equivalent. The continuity equation on $(\mathcal{M},g)$, $\partial_t \mu_t(x) + \nabla_g \cdot (\mu_t(x) v_t(x)) = 0$, plays the role of the first component. The second is satisfied with the marginal vector field defined as:
\begin{equation*}
v_t(x) = \int v_t(x|z) \frac{\mu_t(x|z)}{\mu_t(x)} d\pi(z),
\end{equation*}
Indeed, if $(\mu_t(\cdot|z),v_t(\cdot|z))$ solves the continuity equation, $(\mu_t,v_t)$ is also a solution.  We use $z=(x_0,x_1)\sim \pi$ and $\pi$ is a coupling from $\mu$ and $\nu$, for which multiple options have been identified in the literature~\citep{pooladian2023multisample,lipman2022flow,tong2023conditional}. A typical conditional path is then a dirac centered on a curve $x_t(x_0,x_1)$ that interpolates between $x_0$ and $x_1$. In this case,  the generating vector field is given by $v_t(x|z) = \dot{x}_t = \frac{d}{dt}x_t(x_0,x_1)$. Finally, the equivalence of losses (\textbf{C3}) follows from the second point, as shown in~\citep{lipman2022flow} (see \citet{lipman2024flow} for a comprehensive tutorial). This leads to the celebrated flow matching objective:
\begin{align*}
\mathbb{E}_{t \sim \mathcal{U}[0,1],(x_0,x_1) \sim \pi, x_t \sim \mu_t(\cdot\mid z) }\left[ 
\|v_\theta(x_t,t)-\dot{x}_t\|_{g(x_t)}^2
\right]
\end{align*}

Recently \citet{haviv2024wasserstein} extended this framework to the Wasserstein space over Euclidean domains, termed Wasserstein Flow Matching (WFM). They consider distributions on the Wasserstein space (\emph{i.e.} probability distribution on the space of probability distributions), $\mathbb{P}_0,\mathbb{P}_1 \in \mathcal{P}_2(\mathcal{P}_2(\mathbb{R}^d))$. In practice, samples from $\mathbb{P}_0$ and $\mathbb{P}_1$ are represented as empirical distributions or point clouds $(X_0,X_1)$ and $X_t$ is defined as the McCann interpolant: $X_t = (1-t)X_0 + t \hat{T}^{X_0\to X_1}(X_0)$, with $\hat{T}^{X_0\to X_1}$ the optimal transport map between $X_0$ and $X_1$. One then learns the target vector field $u_{\theta}(X_t,t) \in T\mathcal{P}_2(\mathbb{R}^d)$ by regressing it against $\dot{X}_t =  \hat{T}^{X_0\to X_1}(X_0) - X_0$.



\begin{figure}[t]
\centering
\includegraphics[width=0.95\linewidth]{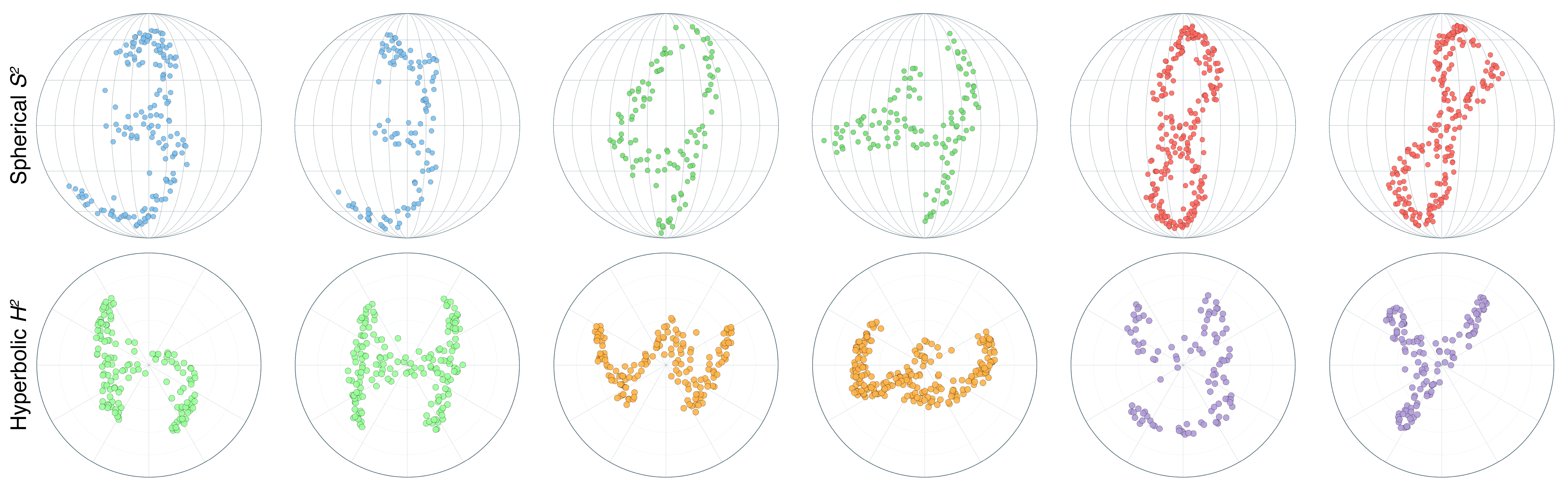}
\caption{\textbf{Generation of distributions on non-Euclidean spaces} Examples of empirical distributions generated by the proposed Riemannian Wasserstein Entropic Flow Matching model on non-Euclidean geometries. The top row displays MNIST digits (3, 4, and 8) generated on the sphere $\mathbb{S}^2$ shown in 2D via Mollweide projections. The bottom row shows EMNIST characters (H, W, and Y) generated on the hyperbolic plane $\mathbb{H}^2$, visualized using the Poincaré disk model.}
\label{fig:mnist_emnist_figures}
\vspace{-4mm}
\end{figure}

\subsection{Related Work}

\paragraph{Optimal Transport on Riemannian Manifolds.} Optimal Transport on Riemannian manifolds has been extensively studied in the mathematical literature, with foundational results on the existence and uniqueness of Monge maps, duality theory, and regularity properties \citep{mccann2001polar, villani2008optimal}. These theoretical insights have paved the way for many applications, particularly in generative modeling for non-Euclidean data \citep{bose2023se,de2022riemannian,huguet2023heat}. Relatedly, \cite{you2026barycentric} study intrinsic and tangential barycentric projections of transport plans on Riemannian manifolds. Recent works have also explored Optimal Transport on Wasserstein spaces themselves, proving existence of continuity equations and geodesics despite the infinite dimensionality, enriching the theoretical framework and expanding its applicability \citep{bonet2025flowing, emami2025optimal,pinzi2025nested}.

\begin{wraptable}{r}{0.54\textwidth}
\centering
\small
\begin{tabular}{lcc}
\toprule
Method & Data Space & Distribution Space \\
\midrule
FM & $(\mathbb{R}^d, L^2)$ & $\times$ \\
RFM & $(\mathcal{M}, g)$ & $\times$ \\
WFM & $(\mathbb{R}^d, L^2)$ & $(\mathcal{P}_2(\mathbb{R}^d), W)$ \\
RWEFM & $(\mathcal{M}, g)$ & $(\mathcal{P}_2(\mathcal{M}), W)$ \\
\bottomrule
\end{tabular}
\caption{\textbf{Comparison of FM approaches.} Only RWEFM respects both the Riemannian structure of the data space and the Wasserstein structure of the distribution space.}
\label{tab:method_comparison}
\vspace{-6mm}
\end{wraptable}
\paragraph{Generative Modeling of Distributions}
Our work relates to recent efforts in defining generative models over spaces of probability measures. Fisher FM \citep{davis2024fisher} and Categorical FM \citep{cheng2024categorical} apply the Flow Matching framework to the simplex $\Delta_d$ equipped with the Fisher--Rao geometry, focusing on categorical data. Similarly, \citet{stark2024dirichlet} utilize the Dirichlet distribution for discrete data generation. For continuous distributions, Meta FM \citep{atanackovic2024meta} and Wasserstein FM \citep{haviv2024wasserstein,piening2026generalized} learn flows directly on the Wasserstein space $\mathcal{P}_2(\mathbb{R}^d)$.  Crucially, these approaches are restricted to Euclidean base spaces, while our framework extends to distributions over general Riemannian manifolds. 

\section{Riemannian Wasserstein Entropic FM}
\label{sec:RWEFM}
\subsection{Flow Matching on $\mathcal{P}_2(\mathcal{M})$}
\label{sec:RWEFM_theory}

RWEFM learns generative flows on $\pspace$ by lifting the conditional FM paradigm to $\ppspace$, respecting the intrinsic geometry of $\mathcal{M}$ and $\pspace$. Following Section~\ref{sec:background_RFM}, this requires three components (\textbf{C1}--\textbf{C3}). We start with (\textbf{C1}), the conditions for when a vector field $V_t\in T\pspace$ generates a probability path $\mathbb{P}_t\in\ppspace$.

\subsubsection{Weak continuity equation and superposition on \texorpdfstring{$\pspace$}{P2(M)} (C1)}

\begin{table}[t]
\centering
\resizebox{\textwidth}{!}{%
\begin{tabular}{@{}lcccccccccccccccccc@{}}
\toprule
& \multicolumn{6}{c}{\textbf{Manifold:} $\mathbb{H}^2$ (EMNIST)} & \multicolumn{6}{c}{\textbf{Manifold:} $\mathbb{S}^2$ (MNIST)} & \multicolumn{6}{c}{\textbf{Manifold:} $\mathbb{T}^2$ (KMNIST)} \\
\cmidrule(lr){2-7} \cmidrule(lr){8-13} \cmidrule(lr){14-19}
& \multicolumn{2}{c}{h} & \multicolumn{2}{c}{w} & \multicolumn{2}{c}{y} & \multicolumn{2}{c}{3} & \multicolumn{2}{c}{4} & \multicolumn{2}{c}{8} & \multicolumn{2}{c}{ki} & \multicolumn{2}{c}{na} & \multicolumn{2}{c}{ma} \\
\cmidrule(lr){2-3} \cmidrule(lr){4-5} \cmidrule(lr){6-7} \cmidrule(lr){8-9} \cmidrule(lr){10-11} \cmidrule(lr){12-13} \cmidrule(lr){14-15} \cmidrule(lr){16-17} \cmidrule(lr){18-19}
Method & CD & EMD & CD & EMD & CD & EMD & CD & EMD & CD & EMD & CD & EMD & CD & EMD & CD & EMD & CD & EMD \\
\midrule
PVD    & 0.32 & 0.35 & 0.43 & 0.38 & \rC{0.25} & \rC{0.30} & 0.45 & 0.42 & 0.40 & 0.40 & 0.47 & 0.42 & $-$ & $-$ & $-$ & $-$ & $-$ & $-$ \\
PSF    & \rB{0.18} & 0.33 & \rA{0.06} & \rB{0.30} & \rB{0.21} & 0.31 & \rB{0.19} & 0.35 & \rA{0.17} & \rC{0.29} & \rB{0.23} & 0.39 & $-$ & $-$ & $-$ & $-$ & $-$ & $-$ \\
\addlinespace
FM     & 0.49 & 0.49 & 0.48 & 0.48 & 0.48 & 0.48 & 0.40 & 0.42 & 0.45 & 0.46 & 0.36 & 0.41 & 0.48 & 0.46 & 0.49 & 0.48 & 0.47 & 0.47 \\
SetFM  & 0.37 & 0.35 & \rC{0.32} & \rC{0.33} & 0.34 & 0.33 & \rC{0.26} & \rB{0.27} & \rC{0.29} & \rC{0.29} & 0.41 & \rB{0.35} & 0.28 & \rC{0.28} & 0.24 & 0.24 & 0.28 & 0.29 \\
WFM    & \rC{0.30} & \rB{0.28} & 0.41 & 0.39 & 0.29 & \rB{0.28} & 0.29 & \rC{0.28} & \rC{0.29} & \rB{0.27} & \rC{0.31} & \rC{0.37} & \rB{0.22} & \rA{0.17} & \rB{0.19} & \rA{0.14} & \rB{0.17} & \rA{0.17} \\
\addlinespace
RFM    & 0.48 & 0.48 & 0.44 & 0.45 & 0.47 & 0.48 & 0.30 & 0.38 & 0.42 & 0.45 & 0.33 & 0.38 & 0.47 & 0.46 & 0.49 & 0.48 & 0.47 & 0.47 \\
SetRFM & \rC{0.30} & \rC{0.31} & 0.34 & \rC{0.33} & 0.31 & 0.31 & 0.38 & 0.38 & 0.38 & 0.38 & 0.39 & 0.40 & \rC{0.26} & \rC{0.28} & \rC{0.23} & \rC{0.23} & \rC{0.26} & \rC{0.28} \\
RWEFM  & \rA{0.15} & \rA{0.09} & \rB{0.23} & \rA{0.10} & \rA{0.16} & \rA{0.11} & \rA{0.18} & \rA{0.18} & \rB{0.19} & \rA{0.18} & \rA{0.16} & \rA{0.16} & \rA{0.17} & \rB{0.18} & \rA{0.16} & \rB{0.16} & \rA{0.15} & \rB{0.18} \\
\bottomrule
\end{tabular}%
}
\caption{\textbf{RWEFM on MNIST, EMNIST, and KMNIST.} Benchmarking of RWEFM against FM variants and other point-cloud generative models on hyperbolic ($\mathbb{H}^2$, EMNIST), spherical ($\mathbb{S}^2$, MNIST), and toroidal ($\mathbb{T}^2$, KMNIST) data. We report the classwise 1-NN deviation (1-NN-D) using Chamfer Distance (CD) and Earth Mover's Distance (EMD) between generated and test point clouds. The score ranges from $0$ to $0.5$, where $0$ corresponds to $50\%$ accuracy in each class; lower is better. Values are averaged over 3 random seeds and 5 samplings per seed. The best three results per column are highlighted (\lgd{rankone}{1st}, \lgd{ranktwo}{2nd}, \lgd{rankthree}{3rd}); ties share a rank. Dashes indicate methods not evaluated on that manifold. Full results with MMD metrics and per-manifold mean $\pm$ std are reported in \cref{tab:mmd_mnist_emnist,tab:nna_std_all,tab:mmd_std_all} (Appendix~\ref{appendix:hyperparameters}). Per-dataset training times for all methods are in \cref{tab:timing}.}
\label{tab:nna_mnist_emnist}
\vspace{-4mm}
\end{table}

The following result shows that if $(\mathbb{P}_t,V_t)$ solves the weak continuity equation on $\pspace$, then $V_t$ \emph{generates} $\mathbb{P}_t$ in the Lagrangian sense:

\begin{theorem}
\label{thm:weakCE_superpostion} 
Let $\mathbb{P}_t$ be an absolutely continuous curve on $\mathcal{P}_2(\pspace)$, and a vector field $V_t \in L^2(\mathbb{P}_t, T\pspace)$ with $\int_{\pspace} \lVert V_t(\mu) \rVert^2_{L^2(\mu)} d\mathbb{P}_t(\mu)<\infty$.

Assume that $(\mathbb{P}_t,V_t)$ solves the weak continuity equation on $\pspace$:
\begin{equation} \label{eq:weak_CE}
    \int_0^T\int_{\pspace} \bigg( \partial_t\varphi_t(\mu) + \langle \nabla_{\mathcal{W}}\varphi_t(\mu), V_t(\mu)\rangle_{L^2(\mu)} \bigg) \, d\mathbb{P}_t(\mu)dt = 0
\end{equation}

for all smooth, compactly supported cylinder functionals $\varphi_t$. Assume moreover that $V_t$ satisfies suitable regularity condition specified in Appendix~\ref{app:lifting_FM}. Then, for $\mathbb{P}_0-a.e$ $\mu \in \pspace$, there exists a unique absolutely continuous curve $\gamma_\mu(t): [0,T]\rightarrow \pspace$ solving 
    \begin{equation*}
    \dot{\gamma}_{\mu}(t) = V_t(\gamma_{\mu}(t)) \quad \text{for a.e. t, with } \gamma_{\mu}(0)=\mu
    \end{equation*}
    and the map $\Phi_t(\mu) :=\gamma_{\mu}(t)$ defines a flow such that $\mathbb{P}_t = (\Phi_t)_{\#} \mathbb{P}_0 \; \text{for all } t\in[0,T]$
\end{theorem}
The proof uses a superposition principle for random measures and is given in Appendix~\ref{app:continuity_p2}. 

\subsubsection{Marginalization of the conditional vector field (C2)}

Let $\Pi(\mu,\nu)$ be a probability measure on $\pspace \otimes \pspace$, with marginals $\int_{\mu}d\Pi(\mu,\nu) = \mathbb{P}_1$ and $\int_{\nu}d\Pi(\mu,\nu) = \mathbb{P}_0$. Given a conditional vector field $V_t(\cdot|\mu_0,\mu_1)$, conditioned on a sample $(\mu_0,\mu_1)\sim\Pi$, we define the marginal vector field $V_t$ as:
\begin{equation*}
V_t(\mu_t) :=  \mathbb{E}_{\mathbb{P}_t(\mu,\nu|\mu_t)}[V_t(\mu_t|\mu_0,\mu_1)]= \iint_{\pspace \times \pspace} V_t(\mu_t|\mu,\nu)\ d\mathbb{P}_t(\mu,\nu|\mu_t)
\end{equation*}
where $d\mathbb{P}_t(\mu,\nu|\mu_t) = \frac{d\mathbb{P}_t(\mu_t|\mu,\nu)d\Pi(\mu,\nu)}{d\mathbb{P}_t(\mu_t)}$ is the conditional probability measure. The following result establishes the weak continuity equation for the marginal pair and, under the hypotheses of Theorem~\ref{thm:weakCE_superpostion}, its Lagrangian interpretation.

\begin{proposition}
Assume $(\mathbb{P}_t(\cdot|\mu,\nu),V_t(\cdot|\mu,\nu))$ solves the weak continuity equation~\eqref{eq:weak_CE}, for $\mu,\nu$ $\Pi-a.s.$. Then $(\mathbb{P}_t,V_t)$ also solves~\eqref{eq:weak_CE}. Under the hypotheses of the superposition theorem (Theorem~\ref{thm:weakCE_superpostion}) for the marginal pair, $V_t$ generates the marginal probability path $\mathbb{P}_t$.
\end{proposition}

\subsubsection{RWEFM Training Objective (C3)}

The ideal, yet intractable, Flow Matching objective is
\begin{align*}
    \mathcal{L}_{FM} = \mathbb{E}_{t,\mu_t\sim \mathbb{P}_t}[\lVert V_t(\mu_t)-V_t^{\theta}(\mu_t)\rVert^2_{L^2(\mu_t)}].
\end{align*}
To bypass this, we optimize the RWEFM objective:
\begin{equation} \label{eq:rwefm_loss}
    \mathcal{L}_{RWEFM} = \mathbb{E}_{\substack{t,\mu,\nu\sim \Pi,\\ \mu_t\sim \mathbb{P}_t(\cdot\mid \mu,\nu)}} \bigg[\big\lVert V_t(\mu_t \mid \mu, \nu) - V_t^{\theta}(\mu_t) \big\rVert^2_{L^2(\mu_t)} \bigg]
\end{equation}
As shown in Appendix~\ref{app:gradloss}, $\nabla_{\theta} \mathcal{L}_{FM} = \nabla_{\theta} \mathcal{L}_{RWEFM}$. Combining \textbf{C1}--\textbf{C3}, we have successfully lifted flow matching to $\ppspace$.

\subsubsection{RWEFM interpolants}

Different choices of conditional probability paths $\mathbb{P}_t(\cdot|\mu_0,\mu_1)$ are possible. In this work, we choose the McCann geodesics as the canonical interpolants between $\mu$ and $\nu$:
\begin{equation}
\label{eq:prob_interpolant}
    \mathbb{P}_t(\cdot|\mu,\nu) = \delta_{\mu_t} \text{ with } \mu_t = (\psi_t)_\# \mu 
        \text{ and }\psi_t(x) = \exp_x(t\log_x(T^{\mu\to \nu}(x))). 
\end{equation}
For such conditional probability path, the conditional vector field that solves~\eqref{eq:weak_CE} is given by:
\begin{align}
    V_t(\cdot|\mu,\nu) = \frac{1}{1-t}\log_{x_t} T^{\mu\to \nu}(\psi_{t}^{-1}(x_t)).\label{eq:cond_vector_field}
\end{align}
We verify that $V_t(\cdot|\mu,\nu) \in T_{\mu_t}\pspace$ with $\log_{x_t}T^{\mu\to \nu}(\psi_{t}^{-1}(x_t)) = \log_{x_t}T^{\mu \to \nu}(x_t) = -\nabla\varphi(x_t)$ for some $\varphi \in C_c^\infty(\mathcal{M})$.

\begin{figure}[b!]
  \centering
  \includegraphics[width=0.8\linewidth]{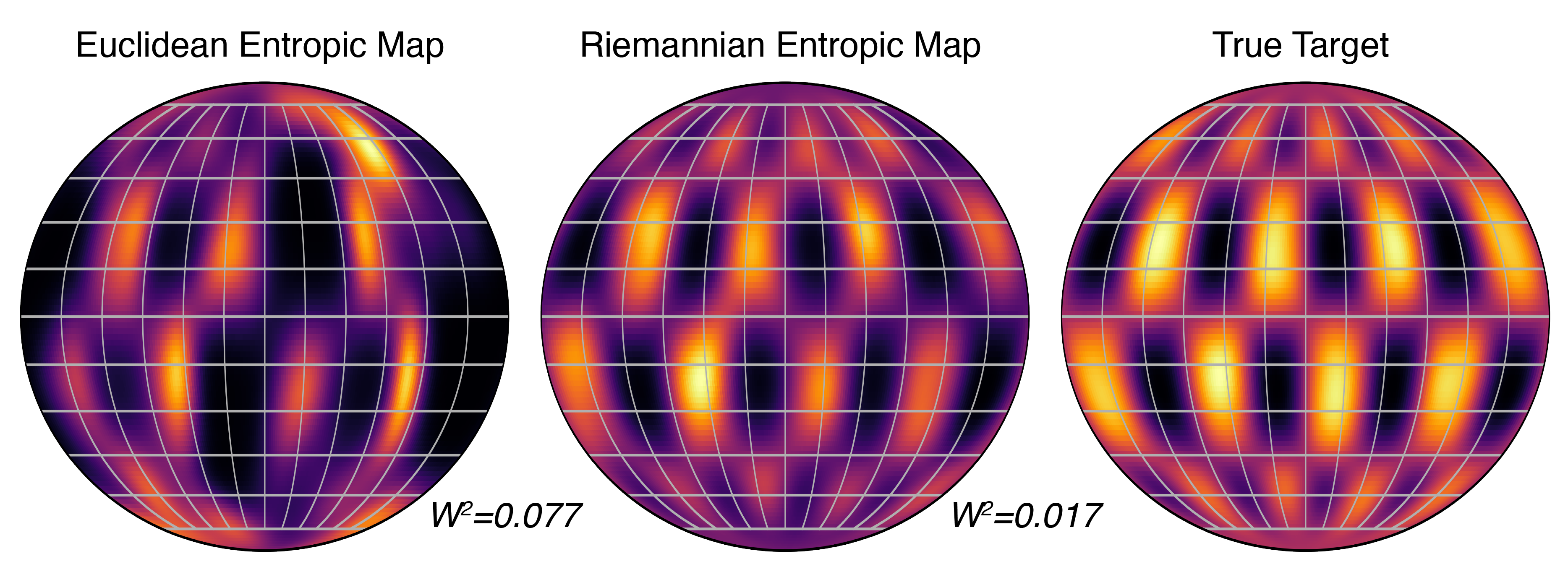}
  \caption{\textbf{Riemannian Entropic Map.} Gaussian-to-checkerboard transport on $\mathbb{S}^2$ using Euclidean (left) and Riemannian (middle) entropic maps. The Riemannian approach respects the underlying geometry, resulting in a more accurate mapping (see \cref{fig:entropic_map_benchmark}).}
  \label{fig:entropic_map_sphere_oos}
  \vspace{-4mm}
\end{figure}

\subsection{The Riemannian Entropic Map}
\label{sec:REM}
A crucial step of computing the RWEFM objective is obtaining the Monge map $T^{\mu\to\nu}$ between two distributions. To this end, we propose the \textbf{Riemannian Entropic Map} $T_\varepsilon(x)$, a scalable estimator for $T^{\mu\to\nu}$ on Riemannian manifolds based on finite samples from $\mu$ and $\nu$ and fast entropic OT solvers.

Like the Euclidean entropic map \citep{pooladian2021entropic}, our estimator leverages the entropic optimal coupling $\pi_\varepsilon$ obtained from the Sinkhorn algorithm, computed using the Riemannian distance cost matrix $C_{ij} = \frac{1}{2} d_g^2(x_i, y_j)$ between samples $\{x_i\}_{i=1}^N \sim \mu$ and $\{y_j\}_{j=1}^N \sim \nu$.

\begin{definition}[Riemannian Entropic Map]
    Let $\pi_\varepsilon$ be the optimal entropic coupling. We define the Riemannian entropic map $T_\varepsilon: \mathcal{M} \to \mathcal{M}$ as:
    \begin{equation} \label{eq:riemannian_map}
        T_\varepsilon(x) := \exp_x\left( \int_{\mathcal{M}} \log_x(y) d\pi_\varepsilon(y|x) \right).
    \end{equation}
\end{definition}

\begin{theorem}\label{thm:main}
    Let $f_\varepsilon$ be the optimal dual potential of the entropic optimal transport problem (Eq~\eqref{eq:dual_EOT}). Under primal feasibility conditions, the Riemannian entropic map arises as the gradient of the dual potential: $T_\varepsilon(x) = \exp_x\left( -\nabla f_\varepsilon(x) \right)$. The proof is in Appendix~\ref{app:EOT}; when $\mathcal{M} = \mathbb{R}^d$, Eq~\eqref{eq:riemannian_map} recovers the Euclidean map $T_\varepsilon = \mathbb{E}_{\pi_\varepsilon}[Y|X=x]$.
\end{theorem}

Computationally, this definition implies a straightforward \textbf{lift-average-retract} procedure for empirical measures. To evaluate the map at a source point $x$: (i) \textbf{lift} by computing the tangent vectors $v_j = \log_x(y_j)$ for all target points $y_j$, mapping the geometry from $\mathcal{M}$ to the vector space $T_x\mathcal{M}$; (ii) \textbf{average} to calculate the weighted mean $\bar{v} = \sum_j \pi_\varepsilon(y_j|x) v_j$; (iii) \textbf{retract} by applying the exponential map to project back to the manifold: $T_\varepsilon(x) = \exp_x(\bar{v})$. We highlight that this procedure only requires access to the Riemannian exponential and logarithm maps, making it applicable to a wide range of manifolds. When these maps are not available, our estimator can still be efficiently computed with only access to the distance function $d_g(x,y)$, as shown in Appendix~\ref{app:entropic_no_exp_no_log}; we exploit this to apply RWEFM on a triangulated mesh with no closed-form geometry (Section~\ref{sec:bunny}). Furthermore, the estimator has the same complexity as the Euclidean entropic map, making it equally scalable to large datasets.

To justify using $T_\varepsilon$ as a reliable proxy for the true Monge map $T_{0}$ in our learning objective, we establish the following statistical bound. This result ensures that the bias introduced by entropic regularization and finite sampling is controlled.


\begin{theorem}[Statistical performance]
\label{thm:statistical-performance}
Let $\widehat T_{\varepsilon,(n,n)}$ be the canonical
out-of-sample entropic map computed from $n$ iid observations
from each of $\mu$ and $\nu$, with the two samples independent.
Under Assumptions~\ref{ass:A7-geometry},
\ref{ass:A7-density}, and~\ref{ass:A7-monge},
there exist constants $C<\infty$ and $\varepsilon_0>0$
such that, for $n\ge2$ and $0<\varepsilon\le\varepsilon_0$,
\[
\begin{aligned}
&\mathbb E\int_\Omega
d_g^2\!\left(
  \widehat T_{\varepsilon,(n,n)}(x),T_0(x)
\right)\,d\mu(x) \le C\left[
  \varepsilon\log\!\left(\frac e\varepsilon\right)
  +\varepsilon^{-s_d}\frac{\log(n+1)}{\sqrt n}
\right],
\end{aligned}
\]
where $s_d=\max\{1,d/2\}$ and $d=\dim\mathcal M$.
The constants depend on the geometry and regularity
assumptions.
The proof is given in
Appendix~\ref{sec:app-statistical-performance}.
\end{theorem}

\subsection{Algorithm and Implementation}
\label{sec:algos}

\begin{figure*}[t]
\centering
\begin{minipage}{\linewidth}
\begin{algorithm}[H]
\small
\setlength{\belowcaptionskip}{-10pt}
   \caption{RWEFM Training Step}
   \label{alg:training}
\begin{algorithmic}
   \STATE {\bfseries Input:} $\mathbb{P}_0, \mathbb{P}_1$
   \STATE Sample $\mu\sim\mathbb{P}_0, \nu \sim \mathbb{P}_1$ \\ Sample $\{x_i\}_{i=1}^N \sim \mu, \{y_j\}_{j=1}^M \sim \nu$
   \STATE Compute $\pi$ via Sinkhorn$(C, \varepsilon)$, 
   Sample $t \sim \mathcal{U}[0,1]$
   \FOR{each $x_i$}
        \STATE $w_{ij} = \pi_{ij}/\sum_k\pi_{ik}$
        \STATE $T_\varepsilon(x_i) =
\exp_{x_i}\!\left(\sum_j w_{ij}\log_{x_i}(y_j)\right)$
       \STATE $z_i = \exp_{x_i}(t \log_{x_i}(T_\varepsilon(x_i)))$\\ $v_i = \log_{z_i}(T_\varepsilon(x_i)) / (1-t)$
       \STATE $\hat{v}_i = v_\theta(z_i, t, \{z_{i}\}_{i=1}^N)$
   \ENDFOR
  \STATE $\mathcal{L} = \frac{1}{N} \sum_i \| \hat{v}_i - v_i \|_{g(z_i)}^2$
  \STATE Update $\theta \leftarrow \theta - \eta \nabla_\theta \mathcal{L}$
\end{algorithmic}
\end{algorithm}
\begin{algorithm}[H]
\small
\setlength{\belowcaptionskip}{-10pt}
   \caption{RWEFM Generation}
   \label{alg:generation}
\begin{algorithmic}
   \STATE {\bfseries Input:} $\mu \sim \mathbb{P}_0$, discretization step $dt$
   \STATE Sample $X_0 = \{x_i\}_{i=1}^N \sim \mu$
   \FOR{$t \gets 0$ \textbf{to} $1-dt$ \textbf{step} $dt$}
       \STATE Compute $v_i = v_\theta(x_i, t, \{x_i\}_{i=1}^N)$
       \FOR{each $x_i$}
           \STATE $x_i \leftarrow \exp_{x_i}(v_i \cdot dt)$
       \ENDFOR
   \ENDFOR
   \STATE {\bfseries Return:} $X_1 = \{x_i\}_{i=1}^N$
\end{algorithmic}
\end{algorithm}
\end{minipage}
\end{figure*}

Here we detail how RWEFM is implemented in practice. The training algorithm is presented in Algorithm~\ref{alg:training}. At each training step, we sample distributions $\mu$ and $\nu$ from the meta-distributions $\mathbb{P}_0$ and $\mathbb{P}_1$, represented as point clouds $X_{0} = \{x_i\}_{i=1}^N, x_{i}\sim \mu$ and $X_{1} = \{y_j\}_{j=1}^M, y_{j}\sim \nu$. We compute the Riemannian entropic map (Eq~\eqref{eq:riemannian_map}) to generate an interpolating point cloud $X_t = \{z_i\}_{i=1}^N$ with $z_i = \exp_{x_i}(t \log_{x_i}(T_\varepsilon(x_i)))$ (Eq~\eqref{eq:prob_interpolant}) and target velocity vectors $V_t =  \{\frac{1}{1-t}\log_{z_i}(T_\varepsilon(x_i))\}_{i=1}^N$ (Eq~\eqref{eq:cond_vector_field}). Since the data is set-structured, we parameterize the neural network $v_\theta(X_{t},t) : \mathcal{M}^N \to T_{X_t}\mathcal{M}^N$ using self-attention architectures, mapping point sets to tangent vectors. The loss is defined as the mean squared geodesic norm between the network prediction $\hat{v}_i$ and the target $v_i$ (Eq~\eqref{eq:rwefm_loss}). For generation (Algorithm~\ref{alg:generation}), we start from $X_0 \sim \mu \sim \mathbb{P}_0$ and integrate $v_\theta$ over time to obtain the final sample $X_1$.


Appendix \ref{appendix:hyperparameters} details hyperparameters and architecture. RWEFM typically trains in a few hours on a single GPU. For large datasets like the single-cell atlas, we trained with 1024 cells sampled from each distribution, still achieving high-quality results in less than 24 hours of training. The size of the point clouds and the entropic regularization $\varepsilon$ are the main hyper-parameters impacting computational complexity, with per-dataset training times for all methods reported in \cref{tab:timing}. Training time can be meaningfully reduced by increasing $\varepsilon$ or decreasing the number of sampled particles, at a marginal cost to generation quality, as studied in detail in Section~\ref{sec:tradeoff}. 




\newpage
\section{Results}
\label{sec:results}

\subsection{The Riemannian Entropic Map is a faithful estimator of the Monge map on $\mathcal{M}$}

We begin by qualitatively demonstrating the effectiveness of the Riemannian entropic map. \Cref{fig:entropic_map_sphere_oos} illustrates the transport error on the sphere $\mathbb{S}^2$, comparing our Riemannian estimator against a Euclidean baseline. By strictly adhering to the underlying geometry, the Riemannian entropic map yields a significantly more accurate coupling, whereas the Euclidean approach incurs high distortion as it ignores the manifold curvature. We provide quantitative evaluation of its performance on constructed examples on the sphere $\mathbb{S}^2$ and hyperbolic space $\mathbb{H}^2$ in Appendix~\ref{app:benchmark_entropic}. The ground truth OT map is generated by taking the exponential map of the gradient of a distance function between a point $x\in\mathcal{M}$ and a fixed attractor $\bar{x}\in\mathcal{M}$. Our results confirm that our estimator is as accurate as the unregularized OT map (Eq~\eqref{eq:monge}), albeit much faster, and outperforms the Euclidean entropic map, which neglects the underlying geometry (\cref{fig:entropic_map_benchmark}).

\begin{figure}[t]
\centering
\includegraphics[width=\linewidth]{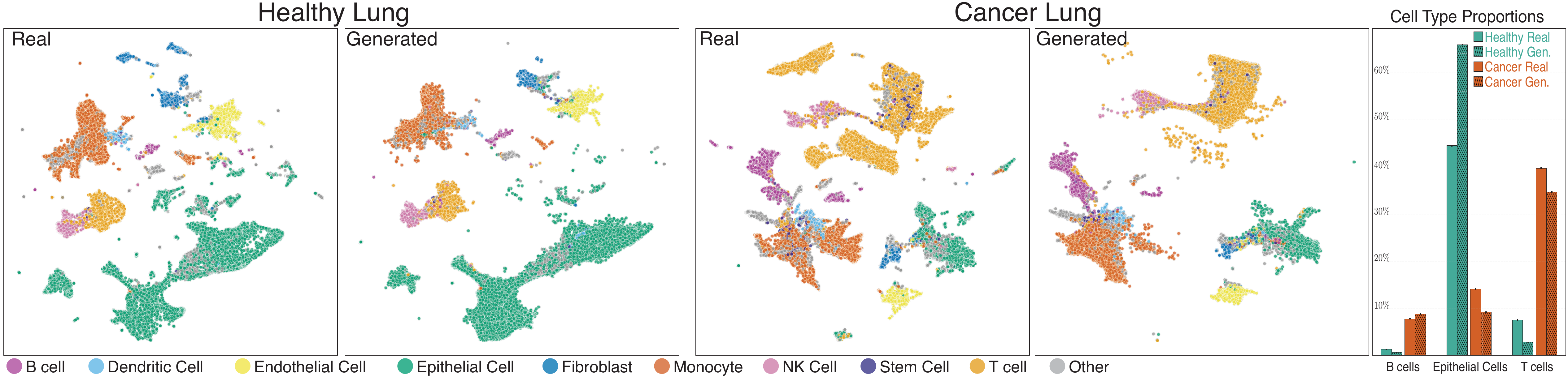}
\caption{\textbf{Generative Modeling of Single-Cell Lung Samples.} UMAP visualization of real and generated scRNA-seq samples from healthy and cancer lung tissues. The generated samples demonstrate high congruence with the real distributions. Furthermore, operating in the foundation model's latent space enables direct analysis of generated data, such as accurate cell typing, highlighting the utility of generating in the (non-Euclidean) latent spaces of foundation models.}
\label{fig:pascientflow}
\vspace{-4mm}
\end{figure}



\subsection{Benchmarking RWEFM on synthetic and real-world scientific datasets}

\subsubsection{Baselines and Evaluation Metrics}

We benchmark RWEFM against several baselines. \textbf{FM}~\citep{lipman2022flow} and \textbf{RFM}~\citep{chen2023flow} operate on individual data points. \textbf{WFM}~\citep{haviv2024wasserstein,atanackovic2024meta} learns flows on the Wasserstein space but is restricted to Euclidean domains. We also introduce \textbf{SetFM}/\textbf{SetRFM}, which apply the FM objective to point clouds without OT couplings. For point cloud data, we include \textbf{PVD}~\citep{zhou20213d} and \textbf{PSF}~\citep{wu2023fast}, trained in Euclidean space and projected onto the manifolds. All set-level methods share the same self-attention architecture, FM/RFM use an MLP and we project generated samples from non-Riemannian methods onto the manifold for fair comparison. We evaluate model performance using classwise 1-NN deviation (1-NN-D) and MMD, with Chamfer Distance (CD) and Earth Mover's Distance (EMD) between point clouds (Appendix~\ref{sec:benchmark-metrics}). For equally sized sets of real and generated point clouds, let $a_{\mathrm{real}}$ and $a_{\mathrm{gen}}$ be their respective classwise nearest-neighbor classification accuracies. We report $\text{1-NN-D}:=\frac12\left(
\left|a_{\mathrm{real}}-\frac12\right|
+\left|a_{\mathrm{gen}}-\frac12\right|
\right)$ This score lies in $[0,0.5]$ and is minimized when both classwise accuracies equal $0.5$. Taking the absolute deviations before averaging prevents opposite classwise deviations from canceling.


\subsubsection{MNIST, EMNIST, and KMNIST on the Sphere, Hyperbolic Space, and Torus}

We evaluate RWEFM on MNIST \citep{lecun1998gradient}, EMNIST \citep{cohen2017emnist}, and KMNIST \citep{kmnist2018} datasets mapped onto Riemannian manifolds: MNIST digits on the sphere $\mathbb{S}^2$, EMNIST letters on hyperbolic space $\mathbb{H}^2$, and KMNIST characters on the torus $\mathbb{T}^2$. RWEFM learns high-quality flows and outperforms baselines in all settings (Figure~\ref{fig:mnist_emnist_figures}, Tables~\ref{tab:nna_mnist_emnist},\ref{tab:mmd_mnist_emnist}). Notably, on the torus $\mathbb{T}^2$, WFM and RWEFM achieve comparable performance. This is expected as $\mathbb{T}^2$ is a flat manifold with zero curvature and the geometric advantage of RWEFM's Riemannian OT over WFM's Euclidean OT diminishes accordingly. Unsurprisingly, FM/RFM perform worst as they are unable to capture the dependencies between individual points required to generate meaningful point clouds.

\suppressfloats[t]
\begin{figure}[t]
\centering
\includegraphics[width=\linewidth]{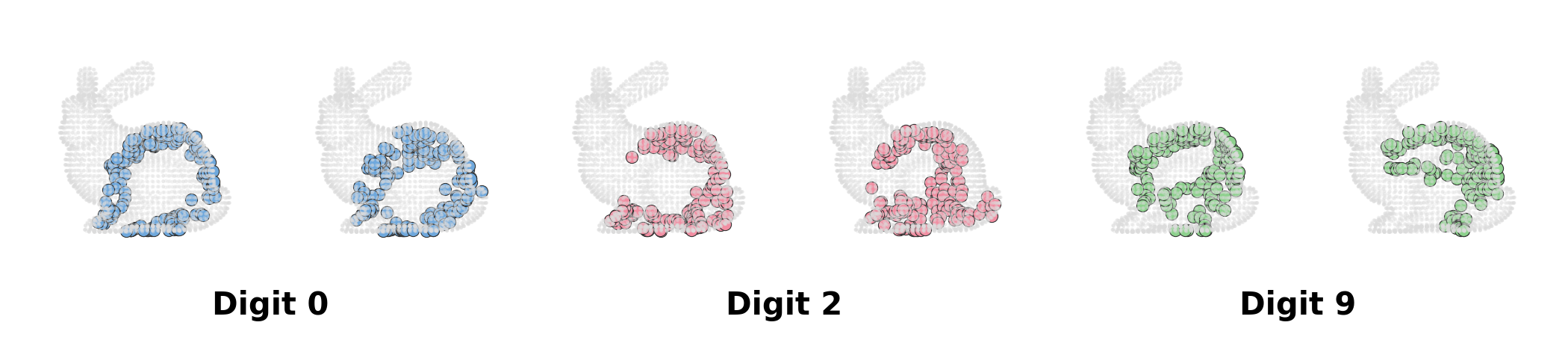}
\caption{\textbf{Distributions generated on the Stanford bunny.} RWEFM-generated empirical distributions on a general triangulated mesh with no closed-form geometry, showing two samples per digit class (0, 2, 9). Generated points lie on the curved surface and recover the digit morphology, illustrating RWEFM on arbitrary geometries (\cref{tab:bunny_mnist}).}
\label{fig:bunny_mnist}
\vspace{-4mm}
\end{figure}

\subsubsection{Beyond closed-form manifolds: distributions on a general triangulated mesh}
\label{sec:bunny}

The manifolds above admit closed-form geodesics, exponential and logarithm maps. Many scientific geometries---triangulated surfaces, learned metrics, implicit surfaces---do not.

RWEFM extends naturally to this setting: both the McCann interpolant and the Riemannian entropic map (Section~\ref{sec:REM}) can be evaluated from only a geodesic distance $d_g$ and a projection operator onto the manifold, without analytic $\exp$/$\log$ maps (Appendix~\ref{app:entropic_no_exp_no_log}).

As a proof of concept, we generate distributions on the surface of the Stanford bunny, a triangulated mesh with no analytic geometry, using MNIST digits as the empirical distributions. We endow the mesh with the spectral (biharmonic) premetric of \citet{chen2023flow}, computed from its smallest Laplace--Beltrami eigenpairs, and lay each digit onto a fixed tangent chart on the bunny's flank (Appendix~\ref{app:bunny_mesh}). \Cref{fig:bunny_mnist} shows point clouds generated by RWEFM for three digit classes and the samples lie on the curved surface and recover the digit morphology.

\begin{wraptable}{R}{0.5\textwidth}
\centering
\resizebox{\linewidth}{!}{%
\begin{tabular}{lcccccc}
\toprule
& \multicolumn{6}{c}{\textbf{Manifold: } $\mathcal{M}_{\mathrm{mesh}}$} \\
\cmidrule(lr){2-7}
& \multicolumn{2}{c}{Digit 0} & \multicolumn{2}{c}{Digit 2} & \multicolumn{2}{c}{Digit 9} \\
\cmidrule(lr){2-3}\cmidrule(lr){4-5}\cmidrule(lr){6-7}
Method & CD & EMD & CD & EMD & CD & EMD \\
\midrule
SetFM  & 0.50 & 0.50 & 0.50 & 0.50 & 0.50 & 0.50 \\
WFM    & 0.50 & 0.50 & 0.40 & 0.42 & 0.50 & 0.50 \\
\addlinespace
SetRFM & \textbf{0.30} & \textbf{0.26} & 0.29 & 0.23 & 0.37 & 0.24 \\
RWEFM  & \textbf{0.30} & 0.28 & \textbf{0.26} & \textbf{0.21} & \textbf{0.29} & \textbf{0.21} \\
\bottomrule
\end{tabular}}
\caption{\textbf{RWEFM on a general triangulated mesh (Stanford bunny).} Here $\mathcal{M}_{\mathrm{mesh}}$ denotes the Stanford bunny surface represented by a triangulated mesh. Classwise 1-NN deviation (1-NN-D; lower is better) under the mesh spectral metric for MNIST digits generated on the bunny (best per column bold; 3 seeds $\times$ 5 samplings). RWEFM and WFM use the sampled map (Appendix~\ref{app:bunny_mesh}); MMD in \cref{tab:bunny_mmd}.}
\label{tab:bunny_mnist}
\vspace{-4mm}
\end{wraptable}

Quantitatively (\cref{tab:bunny_mnist}), the mesh-aware methods (RWEFM, SetRFM) sharply outperform their Euclidean counterparts (WFM, SetFM), which ignore the surface and generally attain or approach the maximum classwise 1-NN deviation ($D_{\mathrm{1NN}}=0.5$), indicating large departures from 50\% classwise accuracy. Respecting the intrinsic geometry therefore remains decisive even when the manifold has no closed-form description, demonstrating that RWEFM applies to arbitrary geometries.

\subsubsection{Single-cell RNA-seq samples on Riemannian spaces}

In single-cell genomics, early generative models focused on synthesizing individual cells. A growing body of work now targets a fundamentally harder task: generating \emph{whole samples}---each sample being an empirical distribution of thousands of cells representing a patient or tissue. This sample-level perspective is essential for capturing inter-sample heterogeneity in disease modeling and perturbation studies \citep{boyeau2025deep,boiarsky2025diffusion,haviv2024wasserstein,atanackovic2024meta,klein2025cellflow}. Concretely, each training example is a point cloud $\mu = \{x_i\}_{i=1}^N$ of cell embeddings for one patient sample, and the goal is to generate new point clouds that faithfully reproduce the distribution of real samples.

\begin{table}[h]
\centering
\resizebox{0.6\linewidth}{!}{%
\begin{tabular}{lccccc}
\toprule
& \multicolumn{2}{c}{\textbf{Manifold: } $\mathbb{S}^{128-1}$} & \multicolumn{3}{c}{\textbf{Manifold: } $\mathbb{T}^2$} \\
\cmidrule(lr){2-3} \cmidrule(lr){4-6}
Method & 1-NN-D (CD) & MMD (EMD) & $W_2$ & $W_1$ & MMD \\
\midrule
SetFM  & 0.4460 & 0.1026 & 0.4472 & 0.4097 & 0.0114 \\
WFM    & 0.2945 & 0.0324 & 0.4059 & 0.4023 & 0.0120 \\
\addlinespace
SetRFM & 0.3754 & 0.0300 & 0.4608 & 0.4126 & 0.0119 \\
RWEFM  & \textbf{0.1889} & \textbf{0.0140} & \textbf{0.3970} & \textbf{0.3850} & \textbf{0.0097} \\
\bottomrule
\end{tabular}}
\caption{\textbf{Generation of cells and torsion angles on manifolds.} (Left) whole single-cell RNAseq samples on $\mathbb{S}^{128-1}$ (SCimilarity embeddings); (right) per-protein torsion-angle distributions on $\mathbb{T}^2$ (ESM-conditioned). Extended metrics in \cref{tab:scrnaseq_cd}; times in \cref{tab:timing}.}
\label{tab:scrnaseq_ot}
\label{tab:torsion_generation}
\vspace{-4mm}
\end{table}

A further challenge is that cell embeddings from foundation models such as SCimilarity \citep{heimberg2025cell} naturally live on non-Euclidean manifolds such as hyperspheres ($\mathbb{S}^{d-1}$) or hyperbolic spaces \citep{ding2021deep}, so naive Euclidean generation distorts the learned geometry. RWEFM directly addresses both challenges: it generates \emph{distributions} (whole samples) on the \emph{manifold} by learning a flow on the Wasserstein space $\mathcal{P}_2(\mathbb{S}^{128-1})$. We evaluate on generating whole single-cell blood samples in SCimilarity's hyperspherical embedding space and find that RWEFM substantially outperforms Euclidean and non-distributional baselines (Table~\ref{tab:scrnaseq_ot}, Figures~\ref{fig:pascientflow},\ref{fig:pascientflow_indv}). The generated samples accurately recapitulate real cell-type compositions and enable direct cell-typing analysis, demonstrating the utility of respecting the geometry of foundation model latent spaces.

\subsubsection{Protein torsion angle distributions on the torus}

\begin{figure}[h]
\centering
\includegraphics[width=\linewidth]{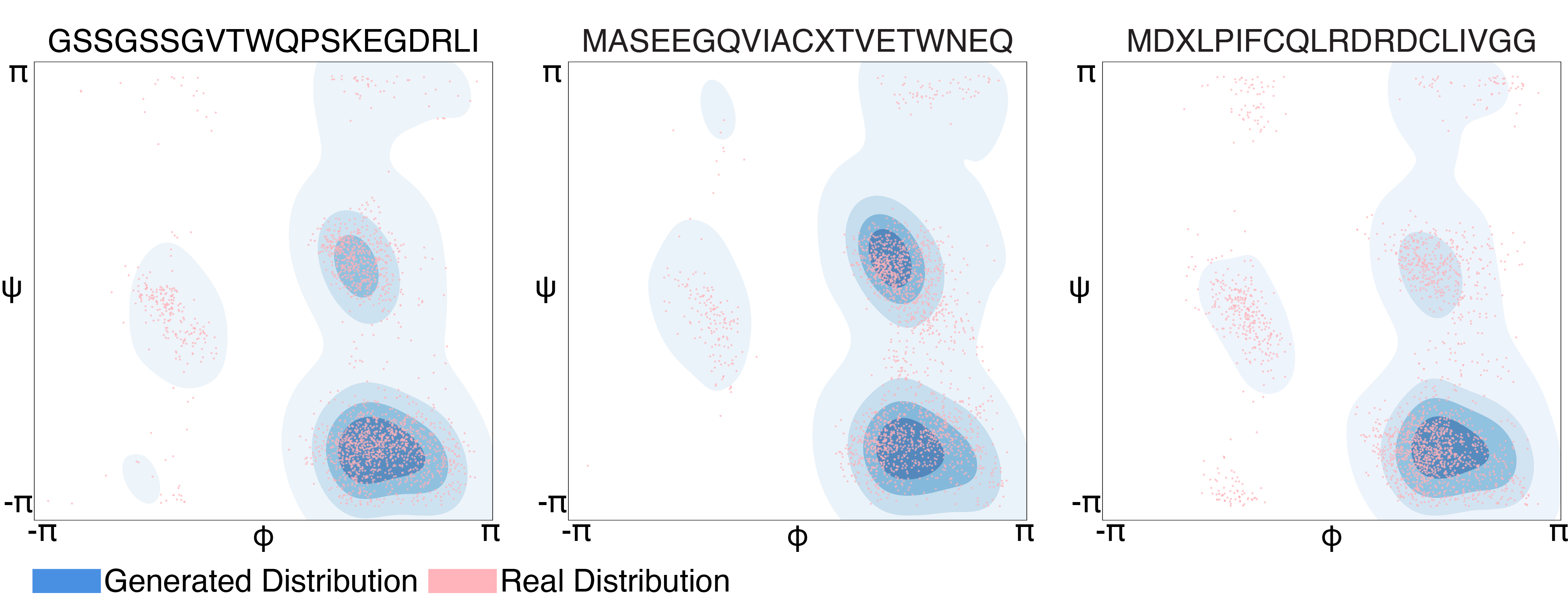}
\caption{\textbf{Torsion angle generation on the torus $\mathbb{T}^2$.} True and generated distributions of protein backbone torsion angles ($\phi$, $\psi$) for held-out mdCATH proteins, conditioned on ESM embeddings of their sequences (first $20$ amino acids shown in subplot titles). RWEFM accurately captures the multimodal structure of conformational distributions.}
\label{fig:torus_torsion}
\vspace{-4mm}
\end{figure}

Protein backbone conformations can be characterized by dihedral angles $\phi$ and $\psi$. The joint angular distribution is depicted using the Ramachandran plot, which reveals energetically favorable conformations and structural motifs such as $\alpha$-helices and $\beta$-sheets. Each protein is thus characterized by a \emph{distribution} of torsion angles. Since both angles are periodic, this distribution lives on the flat torus $\mathbb{T}^2$, a common setting for Riemannian generative models \citep{chen2023flow, davis2025generalised}.

We apply RWEFM to generate per-protein torsion angle distributions from mdCATH~\citep{mirarchi2024mdcath}. For each protein, we aggregate the $(\phi, \psi)$ angles across all residues during MD simulation, producing an empirical distribution over $\mathbb{T}^2$. Unlike prior work that generates individual angle pairs, we generate entire distributions conditioned on ESM embeddings~\citep{hayes2025simulating} of the protein sequence. On unseen test-set proteins, we find that RWEFM accurately captures the complex, multimodal structure of torsion angle distributions (Figure~\ref{fig:torus_torsion}). Quantitatively, RWEFM outperforms Euclidean and non-distributional baselines (Table~\ref{tab:torsion_generation}), demonstrating the benefits of integrating Riemannian geometry with flow matching on $\pspace$ for modeling protein conformational landscapes.

\section{Discussion}
We have introduced Riemannian Wasserstein Entropic Flow Matching (RWEFM), a novel framework for generative modeling of distributions on Riemannian manifolds. We showed that our framework is theoretically sound, computationally tractable, and empirically outperforms  previous approaches that either fail to capture the distributional or geometric aspects of the data. Our proposed Riemannian entropic map is an essential part of our approach, providing highly scalable and accurate estimation of the OT map. Our empirical benchmark spans multiple scientific applications, including single-cell genomics and protein conformation, highlighting the diversity and abundance of scientific problems that could benefit from explicitly incorporating the geometry and distributional nature of the data.

\newpage
\bibliographystyle{plainnat}
\bibliography{ref}


\newpage
\appendix
\onecolumn
\setcounter{figure}{0}
\setcounter{table}{0}
\renewcommand{\thefigure}{S\arabic{figure}}
\renewcommand{\thetable}{S\arabic{table}}

\etocdepthtag.toc{appendix}
\tableofcontents
\clearpage

\let\appendixaddcontentsline\addcontentsline
\renewcommand{\addcontentsline}[3]{%
  \def\contentsfile{#1}\def\tocfile{toc}%
  \ifx\contentsfile\tocfile\phantomsection\fi
  \appendixaddcontentsline{#1}{#2}{#3}}

\section{Entropic Optimal Transport in Riemannian spaces}
\label{app:EOT}

Recovering the optimal transport map $T_0$ between a source distribution $\mu$ and a target distribution $\nu$ is a central task in applied optimal transport. However, closed-form solutions for $T_0$ are exceedingly rare, typically existing only for univariate distributions or Gaussians in Euclidean space. Consequently, in the general setting of Riemannian manifolds, we must resort to statistical estimation from finite samples.

To estimate the map efficiently, we leverage Entropic Optimal Transport (EOT). Unlike unregularized OT, which requires cubic-time combinatorial solvers (e.g., the Hungarian algorithm), EOT can be solved using Sinkhorn's algorithm. On a Riemannian manifold, this simply requires computing the geodesic distance matrix between samples. Once the entropic coupling is obtained, we extend the estimator proposed by \cite{pooladian2021entropic} to the Riemannian setting.

\subsection{Problem Formulation}

Let $(\mathcal{M}, g)$ be a complete Riemannian manifold without boundary, and let $d(x,y)$ denote the geodesic distance. The classical Monge problem seeks a map $T$ minimizing the transport cost:
\begin{equation*}
    \inf_{T \in \mathcal{T}(\mu,\nu)} \int_{\mathcal{M}} \frac{1}{2}d^2(x, T(x)) d\mu(x).
\end{equation*}

The Brenier-McCann theorem shows that this problem is equivalent to 

\begin{equation}
    \sup_{\phi \in L^1(P)} \int_{\mathcal{M}} \phi(x) d\mu(x) + \int_{\mathcal{M}} \phi^c(x) d\nu(x).
\label{eq:brenier_mccann}
\end{equation}

where the c-transform is defined as 

\begin{equation*}
    \phi^c(y) =  \inf_{x\in \mathcal{M}} \{\frac{1}{2}d^2(x,y)-\phi(x) \}.
\end{equation*}

The optimal transport map is then given by $T_0(x) = \exp_x(-\nabla\phi_{0}(x))$ with $\phi_0$ being the maximizer of equation~\eqref{eq:brenier_mccann} and the optimal plan is concentrated where $\phi_{0}(x) + \phi_{0}^c(y) = \frac{1}{2}d^2(x,y)$.

The relaxation to the Kantorovich problem defines the 2-Wasserstein distance:
\begin{equation*}
    \frac{1}{2}W_2^2(\mu,\nu) := \min_{\pi \in \Pi(\mu,\nu)} \int_{\mathcal{M} \times \mathcal{M}} \frac{1}{2}d^2(x,y) d\pi(x,y),
\end{equation*}
where $\Pi(\mu,\nu)$ is the set of couplings with marginals $\mu$ and $\nu$.
For a regularization parameter $\varepsilon > 0$, the Entropic Optimal Transport objective is:
\begin{equation}
    S_\varepsilon(\mu,\nu) := \inf_{\pi \in \Pi(\mu,\nu)} \int_{\mathcal{M} \times \mathcal{M}} \frac{1}{2}d^2(x,y) d\pi(x,y) + \varepsilon D_{KL}(\pi \| \mu \otimes \nu).
\label{eq:entropicOT}
\end{equation}
The unique solution $\pi_\varepsilon$ has the form $d\pi_\varepsilon(x,y) = e^{(f(x) + g(y) - \frac{1}{2}d^2(x,y))/\varepsilon} d\mu(x)d\nu(y)$, where the potentials $f, g \in C(\mathcal{M})$ solve the dual problem:
\begin{equation*}
    \sup_{f, g} \int f d\mu + \int g d\nu - \varepsilon \int_{\mathcal{M} \times \mathcal{M}} e^{\frac{f(x) + g(y) - \frac{1}{2}d^2(x,y)}{\varepsilon}} d\mu(x) d\nu(y) + \varepsilon.
\end{equation*}

\paragraph{Schr\"odinger bridge} The Schr\"odinger bridge problem is closely linked to the Entropic Optimal Transport problem. Let $(X_t)_{t\in[0,1]}$ denote the Brownian motion on $(\mathcal{M},g)$ with generator $\frac{\varepsilon}{2}\Delta_g$, and let $R_\varepsilon$ be its law on path space $\Omega := C([0,1],\mathcal{M})$. Denote by $R_\varepsilon^{01}$ the joint law of the endpoints $(X_0,X_1)$ under $R_\varepsilon$; it admits a density with respect to the product of volume measures given by the heat kernel:
\begin{equation*}
    dR_\varepsilon^{01}(x,y) = p_\varepsilon(1,x,y)\, \dvolm(x)\, \dvolm(y),
\end{equation*}
where $p_\varepsilon(t,x,y)$ solves $\partial_t p_\varepsilon = \frac{\varepsilon}{2}\Delta_g p_\varepsilon$.

The Schr\"odinger bridge problem seeks a coupling $\pi$ that minimizes 
\begin{equation*}
    C_{\varepsilon}(\mu,\nu)
    = \inf_{\pi\in\Pi(\mu,\nu)} D_{\mathrm{KL}}(\pi \,\|\, R_\varepsilon^{01}).
\end{equation*}

\subsection{The Riemannian Entropic Map}

In the Euclidean setting, the entropic map is defined as the barycentric projection $T_\varepsilon(x) = \mathbb{E}_{\pi_\varepsilon}[Y|X=x]$. On a manifold, the direct expectation is not well-defined. Instead, we define the map via the Riemannian exponential map and the conditional expectation in the tangent space.

\begin{definition}[Riemannian Entropic Map]
    Let $(f_\varepsilon, g_\varepsilon)$ be the optimal entropic potentials. We define the Riemannian entropic map $T_\varepsilon: \mathcal{M} \to \mathcal{M}$ as:
    \begin{equation} \label{eq:riemannian_map_app}
        T_\varepsilon(x) := \exp_x\left( \int_{\mathcal{M}} \log_x(y) d\pi_\varepsilon^x(y) \right),
    \end{equation}

where $\log_x = \exp_x^{-1}$ is the Riemannian logarithm and $\pi_\varepsilon^x$ is the conditional distribution of $Y$ given $X=x$ under the optimal plan. With some abuse of notation, $\log_x$ and $\exp_x$ are the (Riemannian) logarithm and exponential maps at point $x$. When no subscript is given, $\log$ and $\exp$ denote the natural logarithm and exponential functions.

\end{definition}

This definition implies a straightforward computational procedure. To evaluate the map at a source point $x$, we rely on the optimal entropic coupling $\pi_\varepsilon$. The procedure involves three steps:
\begin{enumerate}
    \item \textbf{Lift to Tangent Space:} For the fixed source point $x$, we compute the Riemannian logarithm $\log_x(y)$ for every target point $y$ in the support of $\nu$. This maps the geometry from the manifold $\mathcal{M}$ onto the vector space $T_x\mathcal{M}$.
    \item \textbf{Euclidean Averaging:} We compute the weighted average of these tangent vectors, where the weights correspond to the conditional probability mass $\pi_\varepsilon(y|x)$. Because $T_x\mathcal{M}$ is a linear space, this is a simple Euclidean average.
    \item \textbf{Retract to Manifold:} We apply the exponential map $\exp_x$ to this average vector to project the result back onto the manifold, yielding the estimated transport location.
\end{enumerate}

We now establish the relationship between this map and the dual potential gradient. 

\begin{theorem}
    Let $(f_\varepsilon, g_\varepsilon)$ be optimal entropic potentials. The primal feasibility of $\pi_\varepsilon$ requires that its first marginal is $\mu$. In terms of the potentials, this constraint implies:
    \begin{equation} \label{eq:dual_constraint}
        \int_{\mathcal{M}} e^{\frac{f_\varepsilon(x) + g_\varepsilon(y) - \frac{1}{2}d^2(x,y)}{\varepsilon}} d\nu(y) = 1, \quad \forall x \in \text{supp}(\mu).
    \end{equation}
    Under this condition, the Riemannian entropic map satisfies:
    \begin{equation*}
        T_\varepsilon(x) = \exp_x\left( -\nabla f_\varepsilon(x) \right).
    \end{equation*}
\end{theorem}

\begin{proof}
    We assume the potentials satisfy \eqref{eq:dual_constraint}. Taking the logarithm, we isolate $f_\varepsilon(x)$:
    \begin{equation*}
        f_\varepsilon(x) = -\varepsilon \log \int_{\mathcal{M}} \exp\left( \frac{g_\varepsilon(y) - \frac{1}{2}d^2(x,y)}{\varepsilon} \right) d\nu(y).
    \end{equation*}
    Let $h(x,y) = \exp\left( \frac{g_\varepsilon(y) - \frac{1}{2}d^2(x,y)}{\varepsilon} \right)$. We compute the Riemannian gradient $\nabla f_\varepsilon(x)$:
    \begin{equation*}
        \nabla f_\varepsilon(x) = -\varepsilon \frac{\nabla_x \left( \int_{\mathcal{M}} h(x,y) d\nu(y) \right)}{\int_{\mathcal{M}} h(x,y) d\nu(y)}.
    \end{equation*}
    Using the identity $\nabla_x (\frac{1}{2}d^2(x,y)) = -\log_x(y)$, the gradient of the integrand is:
    \begin{equation*}
        \nabla_x h(x,y) = \frac{1}{\varepsilon} h(x,y) \log_x(y).
    \end{equation*}
    Substituting this back, and identifying the conditional density $d\pi_\varepsilon^x(y) = \frac{h(x,y)d\nu(y)}{\int h(x,z)d\nu(z)}$, we obtain:
    \begin{equation*}
        \nabla f_\varepsilon(x) = - \int_{\mathcal{M}} \log_x(y) d\pi_\varepsilon^x(y).
    \end{equation*}
    The integral term is exactly the argument of the exponential map in our definition of $T_\varepsilon$. Thus, $T_\varepsilon(x) = \exp_x(-\nabla f_\varepsilon(x))$.
\end{proof}

\subsection{Properties of the Estimator}

\paragraph{Connection to the Fréchet Mean.}
The notion of a ``center of mass'' on a Riemannian manifold is formalized by the Fréchet mean. Given the conditional distribution $\pi_\varepsilon^x$ of the target $Y$ given $X=x$, the true barycentric projection is the point $m^*$ that minimizes the expected squared distance:
\begin{equation*}
    m^* = \arg\min_{z \in \mathcal{M}} F(z), \quad \text{where } F(z) := \frac{1}{2}\int_{\mathcal{M}} d^2(z,y) d\pi_\varepsilon^x(y).
\end{equation*}
Finding $m^*$ generally requires an iterative optimization procedure, as the gradient of this objective is $\nabla F(z) = - \int_{\mathcal{M}} \log_z(y) d\pi_\varepsilon^x(y)$, leading to the implicit condition $\int \log_{m^*}(y) d\pi_\varepsilon^x(y) = 0$.

Our estimator $T_\varepsilon(x)$ is computed directly as the exponential of the weighted average of tangent vectors, as defined in \eqref{eq:riemannian_map_app}. This definition possesses a geometric connection to the true center of mass as the vector we compute, $v = \int_{\mathcal{M}} \log_x(y) d\pi_\varepsilon^x(y)$, is exactly the negative gradient of the Fréchet objective evaluated at the source point $x$:
\begin{equation*}
    v = -\nabla F(x).
\end{equation*}
Consequently, the operation $T_\varepsilon(x) = \exp_x(v)$ geometrically corresponds to taking \textbf{a single gradient descent step} on the objective $F(z)$, initialized at $x$ with a step size of 1. We stress that our estimator is the correct generalization of the entropic map to the Riemannian setting, and we found the connection to the Fréchet mean to be a useful geometric intuition. 

\paragraph{Recovery of the Euclidean Case.}
If we reduce $\mathcal{M}$ to the Euclidean space $\mathbb{R}^d$ equipped with the standard metric, the geometry simplifies: $d(x,y) = \|x-y\|$, the exponential map becomes translation $\exp_x(v) = x + v$, and the logarithm becomes subtraction $\log_x(y) = y - x$. Substituting these into our definition:
\begin{equation*}
    T_\varepsilon(x) = \exp_x\left( \int (y-x) d\pi_\varepsilon^x(y) \right) = x + \left( \int y d\pi_\varepsilon^x(y) - x \int d\pi_\varepsilon^x(y) \right).
\end{equation*}

Since $\pi_\varepsilon^x$ is a probability distribution, it integrates to 1. Thus:
\begin{equation*}
    T_\varepsilon(x) = \int y d\pi_\varepsilon^x(y) = \mathbb{E}_{\pi_\varepsilon}[Y|X=x].
\end{equation*}

This recovers the standard barycentric projection estimator from~\citet{pooladian2021entropic}.

\subsection{McCann Interpolation via the Entropic Map}

A fundamental concept in optimal transport geometry is the \textit{displacement interpolation}, or McCann interpolation, which describes the geodesic path between probability measures in Wasserstein space. Specifically, if $T_0$ is the optimal transport map pushing $\mu$ to $\nu$, the interpolation at time $t \in [0,1]$ is the distribution $\mu_t = (T_{0,t})_\# \mu$, where $T_{0,t}(x)$ moves the mass at $x$ a fraction $t$ of the way along the geodesic toward $T_0(x)$.

Our Riemannian entropic map offers a constructive way to approximate this interpolation. Because the map $T_\varepsilon(x)$ is constructed via the exponential map of a tangent vector $v_x = \int \log_x(y) d\pi_\varepsilon^x(y)$, the geodesic connecting $x$ to $T_\varepsilon(x)$ is simply the curve $\gamma(t) = \exp_x(t \cdot v_x)$.

\begin{definition}[Entropic Interpolant]
    For any time $t \in [0,1]$, we define the time-$t$ entropic map $T_{\varepsilon, t}: \mathcal{M} \to \mathcal{M}$ as:
    \begin{equation*}
        T_{\varepsilon, t}(x) := \exp_x\left( t \cdot \int_{\mathcal{M}} \log_x(y) d\pi_\varepsilon^x(y) \right).
    \end{equation*}
    The estimated McCann interpolation at time $t$ is the pushforward measure $\hat{\mu}_t := (T_{\varepsilon, t})_\# \mu$.
\end{definition}

This formulation is geometrically intuitive: we compute the aggregate direction of transport $v_x$ in the tangent space and simply scale it by $t$ before projecting back to the manifold. At $t=0$, the argument to the exponential map is the zero vector, recovering the identity map ($T_{\varepsilon, 0}(x) = x$). At $t=1$, we recover the full Riemannian entropic map ($T_{\varepsilon, 1}(x) = T_\varepsilon(x)$). For intermediate $t$, this generates a distribution supported on the geodesic flow between the source and the estimated target.

\subsection{Out-of-Sample Estimation}

A critical feature of the proposed estimator is its ability to generalize to points outside the initial training set. Let $\hat{\nu} = \sum_{j=1}^n \mathbf{b}_j \delta_{Y_j}$ be the discrete empirical measure of the target distribution, and let $\mathbf{g} \in \mathbb{R}^n$ be the optimal dual potential vector obtained from Sinkhorn's algorithm on training samples.

We seek to evaluate the map $T_\varepsilon(x')$ for a new source point $x' \in \mathcal{M}$ that was not present during training. This evaluation is not an ad-hoc interpolation but a direct consequence of the \textit{extension principle} in semi-discrete optimal transport.

\paragraph{The $(c, \varepsilon)$-Transform and Dual Extension.}
In classical Optimal Transport ($\varepsilon=0$), the relationship between optimal potentials is governed by the $c$-transform (or $c$-conjugate), which generalizes the Legendre-Fenchel transform. The source potential $f$ is obtained from the target potential $g$ via a ``hard'' infimum:
\begin{equation*}
    f(x) = g^c(x) := \inf_{y \in \mathcal{M}} \left( \frac{1}{2}d^2(x,y) - g(y) \right).
\end{equation*}

In Entropic Optimal Transport, this non-smooth operator is replaced by a ``soft'' minimum known as the \textbf{$(c, \varepsilon)$-transform}:
\begin{equation*}
    f_\varepsilon(x) = g^{(c, \varepsilon)}(x) := -\varepsilon \log \left( \int_{\mathcal{M}} \exp\left( \frac{g(y) - \frac{1}{2}d^2(x,y)}{\varepsilon} \right) d\nu(y) \right).
\end{equation*}

In our estimation setting, we solve the dual problem using the discrete empirical measure $\hat{\nu}$. While the resulting optimal potential $\mathbf{g}$ is a vector in $\mathbb{R}^n$, the $(c, \varepsilon)$-transform interprets these values as coefficients of a Riemannian kernel expansion. Substituting the discrete measure $\hat{Q}$ into the integral definition yields a globally defined, smooth potential $f_\varepsilon: \mathcal{M} \to \mathbb{R}$:
\begin{equation} \label{eq:f_global}
    f_\varepsilon(x') = -\varepsilon \log \left( \sum_{j=1}^n \mathbf{b}_j \exp\left( \frac{\mathbf{g}_j - \frac{1}{2}d^2(x', Y_j)}{\varepsilon} \right) \right),
\end{equation}
where $\mathbf{b}_j$ are the target weights (typically $1/n$).

\paragraph{Derivation of the Out-of-Sample Map.}
We apply our main theorem, $T_\varepsilon(x') = \exp_{x'}(-\nabla f_\varepsilon(x'))$, to this extended potential.
Differentiating \eqref{eq:f_global} with respect to $x'$ involves the gradient of the log-sum-exp function. By the chain rule:
\begin{equation*}
\begin{split}
    \nabla f_\varepsilon(x') &= -\varepsilon \frac{\sum_{j=1}^n \nabla_{x'} \left[ \mathbf{b}_j \exp\left( \frac{\mathbf{g}_j - \frac{1}{2}d^2(x', Y_j)}{\varepsilon} \right) \right]}{\sum_{k=1}^n \mathbf{b}_k \exp\left( \frac{\mathbf{g}_k - \frac{1}{2}d^2(x', Y_k)}{\varepsilon} \right)} \\
    &= -\varepsilon \sum_{j=1}^n \frac{\mathbf{b}_j \exp(\frac{\mathbf{g}_j - \frac{1}{2}d^2(x', Y_j)}{\varepsilon})}{\sum \mathbf{b}_k \exp(\frac{\mathbf{g}_k - \frac{1}{2}d^2(x', Y_k)}{\varepsilon})} \cdot \nabla_{x'} \left( \frac{\mathbf{g}_j - \frac{1}{2}d^2(x', Y_j)}{\varepsilon} \right).
\end{split}
\end{equation*}
Since $\mathbf{g}_j$ is constant with respect to $x'$, the inner gradient is simply $-\frac{1}{2\varepsilon}\nabla_{x'} d^2(x', Y_j) = \frac{1}{\varepsilon}\log_{x'}(Y_j)$. Substituting this back yields:
\begin{equation*}
    -\nabla f_\varepsilon(x') = \sum_{j=1}^n w_j(x') \log_{x'}(Y_j),
\end{equation*}
where the attention weights $w_j(x')$ are given by the softmax function:
\begin{equation*}
    w_j(x') = \frac{\mathbf{b}_j \exp\left( \frac{\mathbf{g}_j - \frac{1}{2}d^2(x', Y_j)}{\varepsilon} \right)}{\sum_{k=1}^n \mathbf{b}_k \exp\left( \frac{\mathbf{g}_k - \frac{1}{2}d^2(x', Y_k)}{\varepsilon} \right)}.
\end{equation*}
Finally, the transport map for the new point $x'$ is computed by projecting this weighted average back onto the manifold:
\begin{equation*}
    T_\varepsilon(x') = \exp_{x'} \left( \sum_{j=1}^n w_j(x') \log_{x'}(Y_j) \right).
\end{equation*}

\subsection{Computing the Riemannian entropic map without analytic $\exp$ and $\log$ maps}
\label{app:entropic_no_exp_no_log}
When closed-form $\log$ and $\exp$ maps are not available, the entropic estimator remains computationally viable as long as the geodesic distance function $d_g(\cdot,\cdot)$ on $\mathcal{M}$ is accessible. Specifically, we assume we have access to: (i) the geodesic distance $d_g(x,y)$ and its extrinsic gradient $\nabla_x d_g(x,y) \in \mathbb{R}^d$, (ii) a projection operator $\Pi_{\mathcal{M}}: \mathbb{R}^d \to \mathcal{M}$ mapping ambient points to the closest point on $\mathcal{M}$, and (iii) the Jacobian of the projection operator $J_{\Pi}(x) \in \mathbb{R}^{d\times d}$, which at points $x\in\mathcal{M}$ projects ambient vectors onto the tangent space $T_x\mathcal{M}$. 

\paragraph{Approximating the logarithmic map.}
The $N^2$ logarithmic map evaluations in the lift step can be obtained from the gradient of the squared distance. Define $D(x,y) = \frac{1}{2}d_g(x,y)^2$. Its extrinsic gradient with respect to $x$ is $\nabla_x D(x,y) = d_g(x,y)\,\nabla_x d_g(x,y)$. Projecting onto the tangent space at $x$ via the Jacobian of the manifold projection gives:
\begin{equation*}
    \log_x(y) = -J_{\Pi}(x)\big(d_g(x,y)\,\nabla_x d_g(x,y)\big).
\end{equation*}
Substituting into the Riemannian entropic map, the full lift-average-retract computation becomes:
\begin{equation*}
    T_\varepsilon(x_i) = \exp_{x_i}\!\left( -\sum_{j=1}^N \pi_\varepsilon(y_j|x_i)\, J_{\Pi}(x_i)\big(d_g(x_i,y_j)\,\nabla_{x_i} d_g(x_i,y_j)\big) \right).
\end{equation*}
Since the Sinkhorn algorithm already requires the cost matrix $C_{ij} = \frac{1}{2}d_g^2(x_i,y_j)$, the distance gradients can be obtained by a backward pass through the same computation at negligible additional cost and the tangent projection can be performed with a Jacobian-vector product. 

\paragraph{Approximating the exponential map via projected Euler integration.}
The $N$ exponential map evaluations in the retract step can be approximated by numerically integrating the geodesic using only the projection operator. Given $x\in\mathcal{M}$ and $v\in T_x\mathcal{M}$, we seek $y = \exp_x(v)$. For a given number of steps $K$, we can approximate this geodesic by iteratively taking small steps in the ambient space and projecting the position and velocity back onto the manifold.

\begin{enumerate}[leftmargin=*] 
    \item \textbf{Euler step:} $\tilde{y}_{k+1} = y_k + \frac{1}{K} v_k$
    \item \textbf{Position projection:} $y_{k+1} = \Pi_{\mathcal{M}}(\tilde{y}_{k+1})$
    \item \textbf{Velocity projection:} $\tilde{v}_{k+1} = J_{\Pi}(y_{k+1}) v_k$
    \item \textbf{Speed renormalization:} $v_{k+1} = \lVert v \rVert \cdot \tilde{v}_{k+1} / \lVert \tilde{v}_{k+1} \rVert$
\end{enumerate}

Where we start with $y_0 = x$ and $v_0 = v$. The last step is to ensure that the speed of the geodesic is preserved. This procedure requires only the projection operator and its Jacobian-vector product. This formulation opens the door to applying RWEFM on geometries where analytic $\exp$ and $\log$ maps are unavailable, such as triangulated meshes, learned Riemannian metrics, or implicit surfaces; we demonstrate a proof of concept on a triangulated mesh (the Stanford bunny) in Section~\ref{sec:bunny}.

\subsection{Statistical performance of the estimator}
\label{sec:app-statistical-performance}

We now establish a finite-sample guarantee for the Riemannian entropic map.

Throughout this section, we write
\[
    c(x,y) := \frac12 d_g^2(x,y),
\]
and let \((\varphi_0,\varphi_0^c)\) be a pair of optimal Kantorovich
potentials for the unregularized problem. For \(x\in\Omega\), define
\[
    x^\star := T_0(x),
    \qquad
    v_0(x) := \log_x T_0(x),
\]
and the Kantorovich duality gap
\begin{equation}
    D_x(y)
    := c(x,y)-\varphi_0(x)-\varphi_0^c(y).
    \label{eq:A7-duality-gap}
\end{equation}
By optimality, \(D_x(y)\ge 0\) and \(D_x(T_0(x))=0\).

\subsubsection{Regularity assumptions on \(M\) and \(T_0\)}
\label{sec:A7-assumptions}

\begin{assumption}[Geometry and convexity]
\label{ass:A7-geometry}
There exists a compact, strongly geodesically convex set \(\Omega\subset M\)
such that \(\operatorname{supp}(\mu),\operatorname{supp}(\nu)\subset\Omega\).
For every \(x\in\Omega\), the logarithmic map
\(\log_x:\Omega\to T_xM\) is single-valued and smooth. Moreover, on the
relevant tangent balls the exponential map obeys the uniform Lipschitz bound
\[
    d_g\!\left(\exp_x(u),\exp_x(v)\right)
    \le C_{\exp}\,\|u-v\|_g,
    \qquad x\in\Omega.
\]
\end{assumption}

\begin{assumption}[Density bounds]
\label{ass:A7-density}
The measures \(\mu\) and \(\nu\) admit densities \(f_\mu,f_\nu\) with
respect to Riemannian volume on \(\Omega\), and
\[
    0<m_\rho\le f_\mu(x),f_\nu(x)\le M_\rho<\infty,
    \qquad x\in\Omega.
\]
\end{assumption}

\begin{assumption}[Regularity and quadratic detachment of the Monge map]
\label{ass:A7-monge}
The Monge problem for the cost \(c(x,y)=\frac12d_g^2(x,y)\) admits a unique
optimal map \(T_0:\Omega\to\Omega\), which is a diffeomorphism. There exist
constants \(0<\lambda\le \Lambda<\infty\) such that, uniformly for
\(x,y\in\Omega\),
\begin{equation}
    \frac{\lambda}{2}d_g^2(y,T_0(x))
    \le D_x(y)
    \le \frac{\Lambda}{2}d_g^2(y,T_0(x)).
    \label{eq:A7-quadratic-detachment}
\end{equation}
\end{assumption}

Assumption~\ref{ass:A7-geometry} implies that \(c\) is smooth on the compact
set \(\Omega\times\Omega\), hence all derivatives of \(c\) that appear below
are uniformly bounded. It also implies the following two basic geometric
facts. First, there is a constant \(C_{\log}\) such that
\begin{equation}
    \|\log_x(y)-\log_x(z)\|_g
    \le C_{\log} d_g(y,z),
    \qquad x,y,z\in\Omega.
    \label{eq:A7-log-lipschitz}
\end{equation}
Second, since \(\Omega\) is a compact subset of a \(d\)-dimensional smooth
manifold, it admits measurable partitions into at most \(C\delta^{-d}\) sets
of geodesic diameter at most \(\delta\), for every sufficiently small
\(\delta>0\).

Let
\[
    \mu_n=\frac1n\sum_{i=1}^n\delta_{X_i},
    \qquad
    \nu_n=\frac1n\sum_{j=1}^n\delta_{Y_j},
\]
where \(X_i\stackrel{\mathrm{iid}}\sim\mu\) and
\(Y_j\stackrel{\mathrm{iid}}\sim\nu\), with the two samples independent.
We denote by \(\pi_{\varepsilon,n}\) the optimal entropic plan between
\(\mu\) and \(\nu_n\), and by \(T_{\varepsilon,n}\) its Riemannian
barycentric map,
\[
    T_{\varepsilon,n}(x)
    := \exp_x\!\left(b_{\varepsilon,n}(x)\right),
    \qquad
    b_{\varepsilon,n}(x)
    := \int_\Omega \log_x(y)\,d\pi_{\varepsilon,n}^x(y).
\]

For the two-sample estimator, let
\((\widehat f,\widehat g)\) be optimal entropic potentials for
\((\mu_n,\nu_n)\). We use the canonical out-of-sample extension already
described in Section~A.5: for every \(x\in\Omega\),
\begin{equation}
    \widehat f(x)
    := -\varepsilon\log\!\left[
        \frac1n\sum_{j=1}^n
        \exp\!\left(\frac{\widehat g(Y_j)-c(x,Y_j)}{\varepsilon}\right)
    \right],
    \label{eq:A7-oos-f}
\end{equation}
and
\begin{align}
    \widehat w_j(x)
    &:=
    \frac{\exp\!\left((\widehat g(Y_j)-c(x,Y_j))/\varepsilon\right)}
    {\sum_{k=1}^n
     \exp\!\left((\widehat g(Y_k)-c(x,Y_k))/\varepsilon\right)},
    \label{eq:A7-oos-weights}\\
    \widehat b_{\varepsilon,(n,n)}(x)
    &:= \sum_{j=1}^n \widehat w_j(x)\log_x(Y_j),
    \qquad
    \widehat T_{\varepsilon,(n,n)}(x)
    := \exp_x\!\left(\widehat b_{\varepsilon,(n,n)}(x)\right).
    \label{eq:A7-oos-map}
\end{align}
At the sampled source points, this agrees with the barycentric projection of
the empirical optimal entropic coupling.

Define
\begin{equation}
    s_d:=\max\left\{1,\frac d2\right\}.
    \label{eq:A7-sd}
\end{equation}
The use of \(s_d\) only matters in dimension one; for every \(d\ge 2\),
\(s_d=d/2\).

\begin{theorem}[Statistical performance of the Riemannian entropic map]
\label{thm:A7-statistical-performance}
Under Assumptions~\ref{ass:A7-geometry}--\ref{ass:A7-monge}, there exist
constants \(C<\infty\) and \(\varepsilon_0\in(0,1]\), depending only on the
regularity constants and the geometry of \(\Omega\), such that for every
\(0<\varepsilon\le\varepsilon_0\) and \(n\ge2\),
\begin{equation}
\begin{split}
    \mathbb E\int_\Omega
    d_g^2\!\left(\widehat T_{\varepsilon,(n,n)}(x),T_0(x)\right)
    \,d\mu(x)
    \le C\left[
        \varepsilon\log\!\left(\frac{e}{\varepsilon}\right)
        + \varepsilon^{-s_d}\frac{\log(n+1)}{\sqrt n}
    \right].
\end{split}
\label{eq:A7-main-bound}
\end{equation}
In particular, for \(d\ge2\),
\begin{equation}
    \mathbb E\int_\Omega
    d_g^2\!\left(\widehat T_{\varepsilon,(n,n)}(x),T_0(x)\right)
    \,d\mu(x)
    \lesssim
    \varepsilon\log\!\left(\frac{e}{\varepsilon}\right)
    + \varepsilon^{-d/2}\frac{\log(n+1)}{\sqrt n}.
    \label{eq:A7-main-bound-dge2}
\end{equation}
Consequently, choosing \(\varepsilon\asymp n^{-1/(d+2)}\) for \(d\ge2\)
gives, up to logarithmic factors,
\[
    \mathbb E\int_\Omega
    d_g^2\!\left(\widehat T_{\varepsilon,(n,n)}(x),T_0(x)\right)
    \,d\mu(x)
    = \widetilde O\!\left(n^{-1/(d+2)}\right).
\]
\end{theorem}

\subsubsection{Deterministic and empirical ingredients}
\label{sec:A7-ingredients}

We first record a deterministic observation converting the Kantorovich duality
gap into map error.

\begin{lemma}[Duality gap controls barycentric map error]
\label{lem:A7-gap-controls-map}
Let \(\pi\in\Pi(\mu,\eta)\) be any coupling whose second marginal \(\eta\)
is supported in \(\Omega\). Define
\[
    b_\pi(x):=\int_\Omega\log_x(y)\,d\pi^x(y),
    \qquad
    T_\pi(x):=\exp_x(b_\pi(x)).
\]
Then
\begin{equation}
    \int_\Omega d_g^2(T_\pi(x),T_0(x))\,d\mu(x)
    \le C\int_{\Omega\times\Omega}D_x(y)\,d\pi(x,y).
    \label{eq:A7-gap-map}
\end{equation}
\end{lemma}

\begin{proof}
By Assumption~\ref{ass:A7-geometry} and Jensen's inequality,
\begin{align*}
    d_g^2(T_\pi(x),T_0(x))
    &\le C_{\exp}^2\|b_\pi(x)-v_0(x)\|_g^2\\
    &\le C_{\exp}^2
    \int_\Omega
    \|\log_x(y)-\log_x(T_0(x))\|_g^2\,d\pi^x(y).
\end{align*}
Using \eqref{eq:A7-log-lipschitz} and the lower bound in
\eqref{eq:A7-quadratic-detachment},
\[
    \|\log_x(y)-\log_x(T_0(x))\|_g^2
    \le C d_g^2(y,T_0(x))
    \le C D_x(y).
\]
Integrating first in \(y\) and then in \(x\) proves the claim.
\end{proof}

The next lemma bounds the regularization bias without any heat-kernel
representation.

\begin{lemma}[Entropic cost bias]
\label{lem:A7-entropic-bias}
There exists \(C<\infty\) such that, for all sufficiently small
\(\varepsilon\in(0,1]\),
\begin{equation}
    0\le
    S_\varepsilon(\mu,\nu)-\frac12W_2^2(\mu,\nu)
    \le C\varepsilon\log\!\left(\frac{e}{\varepsilon}\right).
    \label{eq:A7-entropic-bias}
\end{equation}
\end{lemma}

\begin{proof}
The lower bound is immediate because the transport part of the entropic
objective is at least the unregularized optimal cost and the relative entropy
is nonnegative.

For the upper bound, fix \(0<\delta<1\). By compactness of \(\Omega\), choose
a measurable partition \(\{B_k\}_{k=1}^N\) of \(\Omega\) such that
\[
    \operatorname{diam}_g(B_k)\le\delta,
    \qquad
    N\le C\delta^{-d}.
\]
Set \(A_k:=T_0^{-1}(B_k)\) and
\(p_k:=\mu(A_k)=\nu(B_k)\). Ignoring cells with \(p_k=0\), define
\begin{equation}
    \pi^\delta
    :=\sum_{k=1}^N
    \frac{1}{p_k}\,\mu|_{A_k}\otimes\nu|_{B_k}.
    \label{eq:A7-block-coupling}
\end{equation}
Then \(\pi^\delta\in\Pi(\mu,\nu)\) and
\begin{equation}
    D_{\mathrm{KL}}(\pi^\delta\|\mu\otimes\nu)
    =\sum_{k=1}^N p_k\log\frac1{p_k}
    \le \log N
    \le C+d\log\frac1\delta.
    \label{eq:A7-block-kl}
\end{equation}
Moreover, if \((x,y)\in A_k\times B_k\), then both \(T_0(x)\) and \(y\)
belong to \(B_k\), so the upper bound in
\eqref{eq:A7-quadratic-detachment} gives \(D_x(y)\le C\delta^2\). Since
\(\pi^\delta\) and the optimal Monge plan have the same marginals,
Kantorovich duality yields
\begin{equation}
    \int c\,d\pi^\delta-\frac12W_2^2(\mu,\nu)
    =\int D_x(y)\,d\pi^\delta(x,y)
    \le C\delta^2.
    \label{eq:A7-block-cost}
\end{equation}
Using \(\pi^\delta\) as a competitor in the entropic problem therefore gives
\[
    S_\varepsilon(\mu,\nu)-\frac12W_2^2(\mu,\nu)
    \le C\delta^2+C\varepsilon
       +d\varepsilon\log\frac1\delta.
\]
Taking \(\delta=\sqrt\varepsilon\) proves
\eqref{eq:A7-entropic-bias}.
\end{proof}

We next state the empirical-process estimate used twice below. We include the
argument because it is also the point at which the intrinsic dimension \(d\)
enters the statistical rate.

\begin{lemma}[Empirical process bound for entropic transforms]
\label{lem:A7-empirical-transform}
Let \(\rho\) be any probability measure supported on \(\Omega\), and let
\(\rho_n\) be its empirical measure based on \(n\) iid observations. Consider
functions of the form
\begin{equation}
    F(x)
    =-\varepsilon\log\int_\Omega
    \exp\!\left(\frac{a(y)-c(x,y)}{\varepsilon}\right)d\xi(y),
    \label{eq:A7-entropic-transform}
\end{equation}
where \(\xi\) is an arbitrary probability measure on \(\Omega\) and
\(a:\Omega\to\mathbb R\) is bounded and measurable. Since adding a constant to \(F\) does
not change \((\rho_n-\rho)F\), normalize all such functions by fixing
\(F(x_\circ)=0\) at one reference point \(x_\circ\in\Omega\), and denote the
resulting class by \(\mathcal F_\varepsilon\). Then
\begin{equation}
    \mathbb E\sup_{F\in\mathcal F_\varepsilon}
    \left|\int_\Omega F\,d(\rho_n-\rho)\right|
    \le C\left(1+\varepsilon^{1-s_d}\right)
    \frac{\log(n+1)}{\sqrt n}.
    \label{eq:A7-emp-process}
\end{equation}
The constant is uniform over \(\rho\), \(\xi\), and \(a\).
\end{lemma}

\begin{proof}
Because \(c\) is smooth on the compact set \(\Omega\times\Omega\), all of its
partial derivatives are uniformly bounded. Differentiating
\eqref{eq:A7-entropic-transform} shows first that
\[
    \nabla F(x)
    =\int_\Omega \nabla_x c(x,y)\,d\omega_x(y),
\]
where \(\omega_x\) is the Gibbs probability measure proportional to
\(\exp((a(y)-c(x,y))/\varepsilon)d\xi(y)\). Repeated differentiation gives,
for every integer \(k\ge1\),
\begin{equation}
    \|F\|_{C^k(\Omega)}
    \le C_k\left(1+\varepsilon^{1-k}\right).
    \label{eq:A7-derivative-bound}
\end{equation}
Indeed, the \(k\)-th derivative is a finite sum of products of derivatives of
\(c\) and centered moments under \(\omega_x\), with at most \(k-1\) powers of
\(\varepsilon^{-1}\). Interpolation between consecutive integer orders gives
the corresponding estimate for noninteger Hölder exponents. Thus, with
\(s_d\) from \eqref{eq:A7-sd},
\begin{equation}
    \sup_{F\in\mathcal F_\varepsilon}
    \|F\|_{C^{s_d}(\Omega)}
    \le C\left(1+\varepsilon^{1-s_d}\right).
    \label{eq:A7-holder-bound}
\end{equation}
The normalization \(F(x_\circ)=0\), together with the uniform first-derivative
bound, also gives a uniform envelope \(\|F\|_\infty\le C\).

A finite smooth atlas reduces the metric entropy calculation to the standard
one for Hölder balls on bounded subsets of \(\mathbb R^d\)~\citep{huang2022error}. Consequently,
with \(M_\varepsilon=C(1+\varepsilon^{1-s_d})\),
\begin{equation}
    \log N\!\left(\eta,\mathcal F_\varepsilon,\|\cdot\|_\infty\right)
    \le C\left(\frac{M_\varepsilon}{\eta}\right)^{d/s_d}.
    \label{eq:A7-entropy-bound}
\end{equation}
By definition of \(s_d\), the exponent \(d/s_d\) is at most \(2\). Standard
symmetrization followed by the truncated Dudley entropy integral therefore
gives
\[
    \mathbb E\sup_{F\in\mathcal F_\varepsilon}
    |(\rho_n-\rho)F|
    \le C M_\varepsilon\frac{\log(n+1)}{\sqrt n}.
\]
This is \eqref{eq:A7-emp-process}.
\end{proof}

As a direct consequence, the one-sample entropic cost is stable under
empirical replacement of one marginal.

\begin{corollary}[One-sample entropic-cost deviation]
\label{cor:A7-cost-deviation}
For \(0<\varepsilon\le1\),
\begin{equation}
    \mathbb E\left|
    S_\varepsilon(\mu,\nu_n)-S_\varepsilon(\mu,\nu)
    \right|
    \le
    C\varepsilon^{-s_d}\frac{\log(n+1)}{\sqrt n}.
    \label{eq:A7-cost-deviation}
\end{equation}
\end{corollary}

\begin{proof}
Use the semi-dual representation with the first marginal \(\mu\) fixed. For a
source-side function \(f\), optimize the dual objective over the target
potential to obtain
\[
    f^{(c,\varepsilon)}(y)
    :=-\varepsilon\log\int_\Omega
    \exp\!\left(\frac{f(x)-c(x,y)}{\varepsilon}\right)d\mu(x).
\]
Then
\[
    S_\varepsilon(\mu,\eta)
    =\sup_f\left\{\int f\,d\mu+
    \int f^{(c,\varepsilon)}\,d\eta\right\}.
\]
Hence
\[
    \left|S_\varepsilon(\mu,\nu_n)-S_\varepsilon(\mu,\nu)\right|
    \le
    \sup_f\left|\int f^{(c,\varepsilon)}\,d(\nu_n-\nu)\right|.
\]
The transformed functions belong, up to additive constants, to the class in
Lemma~\ref{lem:A7-empirical-transform}. Therefore
\[
    \mathbb E\left|
    S_\varepsilon(\mu,\nu_n)-S_\varepsilon(\mu,\nu)
    \right|
    \le C(1+\varepsilon^{1-s_d})
       \frac{\log(n+1)}{\sqrt n}.
\]
Since \(0<\varepsilon\le1\) and \(s_d\ge1\), the right-hand side is bounded
by the one in \eqref{eq:A7-cost-deviation}.
\end{proof}

For completeness, we also record the modified duality inequality needed to
compare the one- and two-sample maps. Importantly, it is valid for any cost
function and does not use Euclidean linear structure.

\begin{lemma}[Modified entropic duality]
\label{lem:A7-modified-duality}
Let \(P,Q\) be probability measures on \(\Omega\), let \(\pi_\varepsilon\)
be their optimal entropic plan, and let \(c\) be any bounded measurable cost.
Then
\begin{equation}
    S_\varepsilon(P,Q)
    \ge
    \int \eta\,d\pi_\varepsilon
    -\varepsilon\iint
      \exp\!\left(\frac{\eta(x,y)-c(x,y)}{\varepsilon}\right)
      dP(x)dQ(y)
    +\varepsilon
    \label{eq:A7-modified-duality}
\end{equation}
for every \(\eta\in L^1(\pi_\varepsilon)\).
\end{lemma}

\begin{proof}
Let \(\gamma=d\pi_\varepsilon/d(P\otimes Q)\). The elementary inequality
\[
    a\log a\ge ab-e^b+a,
    \qquad a\ge0,\ b\in\mathbb R,
\]
applied with
\[
    b=\frac{\eta(x,y)-c(x,y)}{\varepsilon}
\]
and integrated against \(P\otimes Q\) gives the claim after using
\[
    S_\varepsilon(P,Q)
    =\int c\,d\pi_\varepsilon
     +\varepsilon\int\log\gamma\,d\pi_\varepsilon.
\]
\end{proof}

\begin{proposition}[Stability to empirical sampling of the source]
\label{prop:A7-source-stability}
Let \(T_{\varepsilon,n}\) be the entropic map from \(\mu\) to \(\nu_n\), and
let \(\widehat T_{\varepsilon,(n,n)}\) be the out-of-sample extension in
\eqref{eq:A7-oos-map}. Then, for \(0<\varepsilon\le1\),
\begin{equation}
    \mathbb E\int_\Omega
    d_g^2\!\left(
       \widehat T_{\varepsilon,(n,n)}(x),T_{\varepsilon,n}(x)
    \right)d\mu(x)
    \le
    C\varepsilon^{-s_d}\frac{\log(n+1)}{\sqrt n}.
    \label{eq:A7-source-stability}
\end{equation}
\end{proposition}

\begin{proof}
Let \((f_{\varepsilon,n},g_{\varepsilon,n})\) be optimal entropic potentials
for \((\mu,\nu_n)\), normalized so that
\[
    \int_\Omega
    \exp\!\left(
      \frac{f_{\varepsilon,n}(x)+g_{\varepsilon,n}(y)-c(x,y)}{\varepsilon}
    \right)d\nu_n(y)=1
\]
for every \(x\in\Omega\). Likewise, by the definition
\eqref{eq:A7-oos-f},
\[
    \widehat\gamma(x,y)
    :=\exp\!\left(
      \frac{\widehat f(x)+\widehat g(y)-c(x,y)}{\varepsilon}
    \right)
\]
satisfies
\begin{equation}
    \int_\Omega \widehat\gamma(x,y)\,d\nu_n(y)=1
    \qquad\text{for every }x\in\Omega.
    \label{eq:A7-gamma-normalized}
\end{equation}
On the support of \(\mu_n\otimes\nu_n\), \(\widehat\gamma\) is the density
of the empirical optimal entropic plan. In addition,
\[
    \widehat b_{\varepsilon,(n,n)}(x)
    =\int_\Omega \log_x(y)\widehat\gamma(x,y)\,d\nu_n(y).
\]

Apply Lemma~\ref{lem:A7-modified-duality} to the optimal plan
\(\pi_{\varepsilon,n}\) between \(\mu\) and \(\nu_n\), with
\[
    \eta(x,y)
    =\varepsilon\chi(x,y)+\widehat f(x)+\widehat g(y).
\]
Using \eqref{eq:A7-gamma-normalized} gives, for every integrable \(\chi\),
\begin{equation}
\begin{split}
    &\int\chi\,d\pi_{\varepsilon,n}
    -\iint(e^{\chi(x,y)}-1)
      \widehat\gamma(x,y)\,d\mu(x)d\nu_n(y)\\
    &\qquad\le
    \varepsilon^{-1}\left[
       S_\varepsilon(\mu,\nu_n)
       -\int\widehat f\,d\mu
       -\int\widehat g\,d\nu_n
    \right].
    \label{eq:A7-stability-duality}
\end{split}
\end{equation}

We now control the right-hand side. Because
\((\widehat f,\widehat g)\) is optimal for \((\mu_n,\nu_n)\), while
\((f_{\varepsilon,n},g_{\varepsilon,n})\) is an admissible dual pair for that
same empirical problem,
\[
    \int\widehat f\,d\mu_n+
    \int\widehat g\,d\nu_n
    \ge
    \int f_{\varepsilon,n}\,d\mu_n+
    \int g_{\varepsilon,n}\,d\nu_n.
\]
Here no exponential correction remains because both source potentials are
chosen as the entropic \((c,\varepsilon)\)-transform of their corresponding
target potential, so their exponential term integrates to one pointwise in
\(x\). Also,
\[
    S_\varepsilon(\mu,\nu_n)
    =\int f_{\varepsilon,n}\,d\mu+
     \int g_{\varepsilon,n}\,d\nu_n.
\]
Therefore
\begin{equation}
\begin{split}
    &S_\varepsilon(\mu,\nu_n)
       -\int\widehat f\,d\mu
       -\int\widehat g\,d\nu_n\\
    &\qquad\le
      \int (f_{\varepsilon,n}-\widehat f)\,d(\mu-\mu_n).
    \label{eq:A7-potential-comparison}
\end{split}
\end{equation}
Conditionally on \(\nu_n\), the function \(f_{\varepsilon,n}\) is
independent of \(\mu_n\), and hence its contribution has expectation zero.
The random function \(\widehat f\), after an irrelevant additive
normalization, belongs pathwise to the class
\(\mathcal F_\varepsilon\) from
Lemma~\ref{lem:A7-empirical-transform}. Taking expectations in
\eqref{eq:A7-stability-duality}--\eqref{eq:A7-potential-comparison} therefore
gives
\begin{equation}
\begin{split}
    \mathbb E\sup_\chi
    \Bigg\{
    &\int\chi\,d\pi_{\varepsilon,n}
    -\iint(e^\chi-1)\widehat\gamma\,d\mu d\nu_n
    \Bigg\}\\
    &\le
    C\varepsilon^{-1}
    \left(1+\varepsilon^{1-s_d}\right)
    \frac{\log(n+1)}{\sqrt n}\\
    &\le
    C\varepsilon^{-s_d}\frac{\log(n+1)}{\sqrt n}.
    \label{eq:A7-functional-stability}
\end{split}
\end{equation}

It remains to extract the map error from this functional inequality. For a
measurable tangent vector field \(h(x)\in T_xM\), set
\begin{equation}
    \chi_h(x,y)
    :=\left\langle h(x),
      \log_x(y)-\widehat b_{\varepsilon,(n,n)}(x)
      \right\rangle_g
      -a\|h(x)\|_g^2,
    \label{eq:A7-chi-h}
\end{equation}
where \(a>0\) is a sufficiently large geometric constant. Because \(\Omega\)
is compact and \(\log_x(y)\) is uniformly bounded on
\(\Omega\times\Omega\), Hoeffding's lemma applied conditionally in \(x\)
shows that \(a\) can be chosen so that
\begin{equation}
    \int_\Omega
    (e^{\chi_h(x,y)}-1)\widehat\gamma(x,y)\,d\nu_n(y)
    \le0
    \qquad\text{for every }x\in\Omega.
    \label{eq:A7-hoeffding}
\end{equation}
On the other hand, disintegrating \(\pi_{\varepsilon,n}\) gives
\[
    \int\chi_h\,d\pi_{\varepsilon,n}
    =\int_\Omega
       \left[
       \left\langle h(x),
       b_{\varepsilon,n}(x)-
       \widehat b_{\varepsilon,(n,n)}(x)
       \right\rangle_g
       -a\|h(x)\|_g^2
       \right]d\mu(x).
\]
Taking the pointwise supremum over \(h\) yields
\begin{equation}
    \sup_h\int\chi_h\,d\pi_{\varepsilon,n}
    =\frac{1}{4a}
      \int_\Omega
      \|b_{\varepsilon,n}(x)-
      \widehat b_{\varepsilon,(n,n)}(x)\|_g^2d\mu(x).
    \label{eq:A7-tangent-stability}
\end{equation}
Combining \eqref{eq:A7-functional-stability}--\eqref{eq:A7-tangent-stability}
and finally applying the Lipschitz bound for the exponential map from
Assumption~\ref{ass:A7-geometry} proves
\eqref{eq:A7-source-stability}.
\end{proof}

\subsubsection{Proof of Theorem~\ref{thm:A7-statistical-performance}}
\label{sec:A7-proof-main}

We first bound the one-sample error. Applying
Lemma~\ref{lem:A7-gap-controls-map} to \(\pi_{\varepsilon,n}\) gives
\begin{equation}
    \int d_g^2(T_{\varepsilon,n}(x),T_0(x))\,d\mu(x)
    \le C\int D_x(y)\,d\pi_{\varepsilon,n}(x,y).
    \label{eq:A7-one-sample-gap}
\end{equation}
Since \(\pi_{\varepsilon,n}\) is optimal for the entropic problem between
\(\mu\) and \(\nu_n\),
\begin{align}
    \int D_x(y)\,d\pi_{\varepsilon,n}(x,y)
    &\le
    S_\varepsilon(\mu,\nu_n)
    -\int\varphi_0\,d\mu
    -\int\varphi_0^c\,d\nu_n.
    \label{eq:A7-gap-by-entropic-value}
\end{align}
Indeed, the difference between the right-hand side and the left-hand side is
exactly \(\varepsilon D_{\mathrm{KL}}(\pi_{\varepsilon,n}\|\mu\otimes\nu_n)\),
which is nonnegative. Taking expectations and using
\(\mathbb E\nu_n=\nu\) together with Kantorovich duality yields
\begin{align}
    \mathbb E\int D_x(y)\,d\pi_{\varepsilon,n}(x,y)
    &\le
    \mathbb E\left[
       S_\varepsilon(\mu,\nu_n)-S_\varepsilon(\mu,\nu)
    \right]\\
    &\quad+
    S_\varepsilon(\mu,\nu)-\frac12W_2^2(\mu,\nu)\\
    &\le
    C\varepsilon^{-s_d}\frac{\log(n+1)}{\sqrt n}
    +C\varepsilon\log\!\left(\frac e\varepsilon\right),
    \label{eq:A7-one-sample-final}
\end{align}
where the last line follows from
Corollary~\ref{cor:A7-cost-deviation} and
Lemma~\ref{lem:A7-entropic-bias}. Combining
\eqref{eq:A7-one-sample-gap} and \eqref{eq:A7-one-sample-final},
\begin{equation}
    \mathbb E\int
    d_g^2(T_{\varepsilon,n}(x),T_0(x))\,d\mu(x)
    \le
    C\left[
      \varepsilon\log\!\left(\frac e\varepsilon\right)
      +\varepsilon^{-s_d}\frac{\log(n+1)}{\sqrt n}
    \right].
    \label{eq:A7-one-sample-risk}
\end{equation}

Finally, the squared triangle inequality and
Proposition~\ref{prop:A7-source-stability} give
\begin{align*}
    &\mathbb E\int
    d_g^2(\widehat T_{\varepsilon,(n,n)}(x),T_0(x))\,d\mu(x)\\
    &\qquad\le
    2\mathbb E\int
    d_g^2(\widehat T_{\varepsilon,(n,n)}(x),T_{\varepsilon,n}(x))\,d\mu(x)
    +2\mathbb E\int
    d_g^2(T_{\varepsilon,n}(x),T_0(x))\,d\mu(x)\\
    &\qquad\le
    C\left[
      \varepsilon\log\!\left(\frac e\varepsilon\right)
      +\varepsilon^{-s_d}\frac{\log(n+1)}{\sqrt n}
    \right],
\end{align*}
which proves Theorem~\ref{thm:A7-statistical-performance}.

\section{A short introduction to flow matching}

In this section, we provide a very concise introduction to flow matching and highlight the theoretical steps needed to show the validity of the procedure on the space of probability distributions. We refer the reader to \cite{lipman2024flow} for a complete introduction to flow matching.

\subsubsection{Vector fields generating probability paths}
At its core, flow matching aims to find a tractable way to generate samples from a target distribution $\nu$ from a source distribution $\mu$. One way to do so is by learning a vector field $u_t$ that \emph{generates} a probability path with the right boundary conditions. By \emph{generation}, we understand the Lagrangian interpretation. That is, $u_t$ generates a probability path $\mu_t$ iff, for all $t\in[0,1]$
\begin{equation*}
\begin{split}
    \partial_t\phi_t(x) &= u_t(\phi_t(x)) \\
    X_t &= \phi_t(X_0) \sim \mu_t.
\end{split}
\end{equation*}

With such a $u_t$, one can then sample from $\mu$ and integrate over time to generate samples from distributions $\mu_t$. An important result to verify that $u_t$ generates a given probability path is the mass conservation formula:
\begin{equation*}
    u_t \text{ generates } \mu_t \Leftrightarrow \partial_t \mu_t(x) + \nabla \cdot (\mu_t u_t)(x) = 0
\end{equation*}

That is, $(u_t,\mu_t)$ solving the continuity equation is equivalent to $u_t$ generating $\mu_t$. However, directly using this relation is impractical as we typically don't have information about the target distribution $\nu$.

\subsubsection{Conditional vector fields and probability paths}

A more promising strategy is to construct probability paths using conditional probability paths:
\begin{equation*}
    \mu_t(x) = \int \mu_{t|z}(x|z) d\pi(z)
\end{equation*}

Of course, one needs to design the conditional probability paths and marginal distribution for the conditioning variable such that the boundaries coincide with the constraints ($\mu$ and $\nu$).

If $\mu$ is known (\emph{e.g.} a standard normal), one can choose
\begin{equation*}
    d\pi(z) = d\nu(x_1), \quad p_{t|x_1} = \mathcal{N}(tx_1, (1-t)^2)
\end{equation*}

We verify that in that case: 
\begin{equation*}
\begin{split}
    \int \mathcal{N}(x;0,1) d\nu(x_1) &= \mathcal{N}(0,1)\\
    \int \delta_{x_1} d\nu(x_1) &= \nu
\end{split}
\end{equation*}

When $\mu$ is an arbitrary (unknown) distribution, one can instead choose

\begin{equation*}
\begin{split}
    &d\pi(z) = d\pi(x_0,x_1), \quad \int d\pi(x_0,x_1) dx_1 = d\mu(x_0), \\
    &\int d\pi(x_0,x_1) dx_0 = d\nu(x_1), \quad p_{t|x_0,x_1} = \delta_{(1-t)x_0 + tx_1}(x)
\end{split}
\end{equation*}

In which case, we verify again that the boundary conditions are respected
\begin{equation*}
\begin{split}
    \iint \delta_{x_0} d\pi(x_0,x_1) &= \int \delta_{x_0} d\mu(x_0) = \mu(x)\\
    \iint \delta_{x_1} d\pi(x_0,x_1) &= \int \delta_{x_1} d\nu(x_1) = \nu(x)
\end{split}
\end{equation*}

Given the simple form of the conditional probability paths, one can easily derive corresponding vector fields that generate the conditional distributions, $v_t(x\mid z)$.

The first crucial realization of flow matching is that the marginal conditional vector field $v_t$ also generates the marginal probability path $\mu_t$
\begin{equation*}
    \partial_t \mu_t(x) + \nabla\cdot(\mu_t v_t)(x) = 0
\end{equation*}

where 
\begin{equation*}
\begin{split}
    v_t(x) &= \int v_t(x\mid z) \frac{\mu_{t|z}(x|z)\pi(z)}{\mu_t(x)}\\
    v_t(x) &= \mathbb{E}[v_t(X_t| Z)| X_t = x  ]
\end{split}
\end{equation*}

That is, from simple conditional vector fields that generate simple conditional probability paths, we can construct the marginal vector field that generates the probability path of interest. Now, of course, the equation above is not tractable as $P(Z \mid X_t)$ is generally not available. The second realization of flow matching is that learning $v_t^{\theta}(x)$ by regressing it to conditional vector fields has the same minimizer than regressing $v_t^{\theta}(x)$ against its true value, as shown next.

\subsubsection{Flow matching loss}

Our intended flow matching loss is to regress a learnable vector field to the marginal vector field:
\begin{equation*}
    \mathcal{L}_{FM} = \mathbb{E}_{t,X_t}[\lVert v_t(X_t)-v_t^{\theta}(X_t)\rVert^2]
\end{equation*}

Instead we can use a loss that involves only tractable distributions:
\begin{equation*}
    \mathcal{L}_{CFM} = \mathbb{E}_{t,Z\sim \pi(Z),X_t\sim \mu_t(X_t\mid Z)}[\lVert v_t(X_t|Z)-v_t^{\theta}(X_t) \rVert^2]
\end{equation*}

Crucially, one can show that the gradients of both losses are identical.
\begin{equation*}
    \nabla_{\theta}\mathcal{L}_{CFM} = \nabla_{\theta}\mathcal{L}_{FM}
\end{equation*}

\subsubsection{Roadmap to lifting to empirical distributions}

In this work, we want to leverage the same strategy for training our generative model but operating on the space of probability distributions defined on a Riemannian manifold $\mathcal{M}$: $\mathcal{P}_2 (\mathcal{M})$, which is an infinite dimensional space. For this to be possible, we need to show the three following statements apply on that space:

\begin{itemize}
    \item (C1) The existence of an equivalent of the mass conservation formula on $\ppspace$ (Section~\ref{app:continuity_p2})
    \item (C2) Showing that the marginal vector field $v_t$ generates $\mu_t$ (Section~\ref{sec:marginal})
    \item (C3) The gradient of the conditional flow matching loss is identical to the gradient of the flow matching loss (Section~\ref{app:gradloss})
\end{itemize}

\section{Lifting Flow Matching to $\pspace$}
\label{app:lifting_FM}

\subsubsection*{Setup and Definitions}

We work on the Wasserstein space $\pspace$, the space of probability measures on a manifold $\mathcal{M}$ with finite second moments. We endow $\pspace$ with the Borel $\sigma-$algebra generated by the $W_2$ metric (with the underlying geodesic distance).

\subsection{Continuity equation and Lagrangian flows on $\pspace$}
\label{app:continuity_p2}

\paragraph{Objectives of this section}

In the following, we show that if $(\mathbb{P}_t,V_t)$ solve the continuity equation on $\pspace$(\emph{i.e.} with test functionals on $\pspace$), then $V_t$ generates $\mathbb{P}_t$ in the Lagrangian sense $\mathbb{P}_t = (\Phi_t)_{\#}\mathbb{P}_0$. That is, conceptually, one can generate the probability path $\mathbb{P}_t$ by flowing individual elements from $0$ to $t$ according to the vector field $V_t$, which underlies the generation procedure used in flow matching.

\subsubsection{Definitions}

\paragraph{Tangent space and path space of continuous curves}
For each $\mu \in \pspace$, we identify the tangent space $T_{\mu}\pspace$ with the closure of the gradient vector fields in $L^2(\mu;T\mathcal{M})$ and endow it with the norm $\lVert v\rVert^2_{L^2(\mu)} = \int_{\mathcal{M}}\lVert v(x)\rVert_g^2 d\mu(x)$.  We further write $L^2(\mathbb{P}_t, T\pspace)$ the space of functions $V_t :\pspace \rightarrow T\pspace$ such that $\int_{\pspace} \lVert V_t(\mu) \rVert^2_{L^2(\mu)}(\mu) d\mathbb{P}_t < \infty$. We write $\Gamma_T := C([0,T]; \pspace)$ the path space of continuous curves $\gamma$ from $[0,T]$ to $\pspace$.

We define the elevation maps $e_t$ as
\begin{equation*}
e_t: (\gamma) \in \Gamma_T \rightarrow \gamma(t) \in \pspace \quad \text{for }t\in[0,T]
\end{equation*}

\paragraph{Cylinder Functions}

\begin{definition} A functional $\mathcal{F}:\pspace\to \mathbb{R}$ is called a cylinder function if there exists an integer $k \ge 1$, a smooth function with compact support $F \in C_c^\infty(\mathbb{R}^k)$, and a set of smooth functions with compact support $\{V_1, \dots, V_k\} \subset C_c^\infty(\mathcal{M})$, such that for any measure $\mu\in\pspace$:

\begin{equation*}
\mathcal{F}(\mu) = F\left(\int_{\mathcal{M}} V_1 d\mu, \dots, \int_{\mathcal{M}} V_k d\mu\right)
\end{equation*}
\end{definition}

This definition can be extended to time-dependent functionals $\varphi_t(\mu) = \varphi(t, \mu)$, where $\varphi(t, \mu) = F(t, \int V_1 d\mu, \dots, \int V_k d\mu)$ for some $F \in C_c^\infty(I \times \mathbb{R}^k)$.

In this text, we will assume that time-dependent cylinder functions have a compact support in time $(0,T)$, such that $\varphi_0(x) = \varphi_T(x) = 0$ for all $x$. 

Using the chain rule, the gradient of a cylinder functional $\mathcal{F}$, the Wasserstein gradient, at a measure $\mu$ can be computed.

\begin{definition}  The Wasserstein gradient of a cylinder function $\mathcal{F}$, denoted $\nabla_{\mathcal{W}}\mathcal{F}(\mu)$, is the vector field on $\mathcal{M}$:

\begin{equation*}
\nabla_{\mathcal{W}}\mathcal{F}(\mu) = \sum_{i=1}^k \frac{\partial F}{\partial x_i}\left(\int V_1 d\mu, \dots, \int V_k d\mu\right) \nabla V_i
\end{equation*}
\end{definition}

Here, $\frac{\partial F}{\partial x_i}$ is the partial derivative of $F$ with respect to its $i$-th argument, and $\nabla V_i$ is the gradient of the function $V_i$ on the manifold $\mathcal{M}$.

\subsubsection{Continuity equation and superposition principle on $\pspace$}

We give precise sufficient hypotheses for the Lagrangian interpretation
used in the main text. The argument adapts the superposition principle
for random measures of \citet[Theorem~1.2]{pinzi2025nested} through a
smooth embedding.

Recall that $(\mathbb P_t,V_t)$ satisfies the weak continuity
equation on $\pspace$ if
\begin{equation}
\begin{aligned}
&\int_0^T\int_{\pspace}\Bigl[\partial_t\varphi_t(\mu)\\
&\qquad+\bigl\langle\nabla_{\mathcal W}\varphi_t(\mu),
V_t(\mu)\bigr\rangle_{L^2(\mu)}\Bigr]\,
d\mathbb P_t(\mu)\,dt=0,
\end{aligned}
\tag{Weak CE}
\label{eq:weak_CE_app}
\end{equation}
for every smooth cylinder functional $\varphi_t$ with compact
support in time in $(0,T)$.

\begin{theorem}[Weak continuity equation and superposition in $\pspace$]
\label{thm:weak_continuity_superposition}
Let $(\mathcal M,g)$ be a connected, complete, smooth Riemannian manifold
without boundary, and let $(\mathbb P_t)_{t\in[0,T]}$ be an absolutely
continuous curve in $\mathcal P_2(\pspace)$. Assume there is a compact
set $K\subset\mathcal M$ such that
$\mathbb P_t(\mathcal P(K))=1$ for every $t$, where $\mathcal P(K)$
denotes the probability measures supported in $K$.
Let $V_t(\mu)\in T_\mu\pspace$ admit a jointly Borel representative
$b(t,x,\mu)=V_t(\mu)(x)\in T_x\mathcal M$, and assume
\begin{equation*}
\int_0^T\int_{\pspace}
\|V_t(\mu)\|_{L^2(\mu)}^2\,d\mathbb P_t(\mu)\,dt<\infty.
\end{equation*}
Suppose $(\mathbb P_t,V_t)$ satisfies the weak continuity
equation~\eqref{eq:weak_CE_app}. Then:
\begin{enumerate}
\item There exists a probability measure $\Gamma$ on
$\Gamma_T=C([0,T];\pspace)$, concentrated on
$W_2$-absolutely continuous curves taking values in $\mathcal P(K)$,
such that
\begin{equation*}
(e_t)_\#\Gamma=\mathbb P_t,\qquad t\in[0,T].
\end{equation*}

\item For $\Gamma$-almost every trajectory $\gamma$, the
measure-valued continuity equation
\begin{equation}
\partial_t\gamma(t)+\operatorname{div}_{\mathcal M}
\bigl(\gamma(t)V_t(\gamma(t))\bigr)=0
\tag{Measure CE}
\label{eq:measure_characteristic}
\end{equation}
holds weakly on $\mathcal M$. Consequently, for each smooth
time-dependent cylinder functional $\varphi_t$, along
$\Gamma$-almost every trajectory and for almost every $t$,
\begin{equation*}
\frac{d}{dt}\varphi_t(\gamma(t))
=\partial_t\varphi_t(\gamma(t))
+\bigl\langle\nabla_{\mathcal W}\varphi_t(\gamma(t)),
V_t(\gamma(t))\bigr\rangle_{L^2(\gamma(t))}.
\end{equation*}
This is the weak interpretation of
$\dot\gamma(t)=V_t(\gamma(t))$ used here.

\item Assume additionally that, for $\mathbb P_0$-almost every
$\mu$, equation~\eqref{eq:measure_characteristic} has at most one
$W_2$-absolutely continuous solution in $\mathcal P(K)$ with
$\gamma(0)=\mu$. Then, for $\mathbb P_0$-almost every $\mu$,
there exists a unique $W_2$-absolutely continuous curve
$\gamma_\mu:[0,T]\to\mathcal P(K)\subset\pspace$ solving
\begin{equation*}
\begin{aligned}
\dot\gamma_\mu(t)&=V_t(\gamma_\mu(t)),
\qquad\text{for a.e. }t\in[0,T],\\
\gamma_\mu(0)&=\mu,
\end{aligned}
\end{equation*}
where the evolution equation is understood in the weak sense
of~\eqref{eq:measure_characteristic}, equivalently through the
cylinder chain rule above. The map $\Phi_t(\mu):=\gamma_\mu(t)$
defines a measurable flow, up to $\mathbb P_0$-null sets, such that
\begin{equation*}
\mathbb P_t=(\Phi_t)_\#\mathbb P_0,\qquad t\in[0,T].
\end{equation*}
Thus $V_t$ generates the probability path $\mathbb P_t$ by
evolving each initial measure along its trajectory $\gamma_\mu$.
\end{enumerate}
\end{theorem}

\begin{proof}[Proof]
Choose a smooth embedding $\jmath:\mathcal M\to\mathbb R^D$ and
write $J(\mu)=\jmath_\#\mu$ for $\mu\in\mathcal P(K)$. Set
$\widetilde{\mathbb P}_t=J_\#\mathbb P_t$ and define
\begin{equation*}
\widetilde b(t,\jmath(x),J(\mu))
=d\jmath_x\,b(t,x,\mu),
\qquad x\in K,\quad\mu\in\mathcal P(K),
\end{equation*}
extending $\widetilde b$ by zero elsewhere. On the compact set $K$,
the embedding has bounded differential and its inverse has bounded
metric distortion. Pulling back cylinder tests, using a smooth cutoff
equal to one near $K$, transfers the weak continuity equation to
$(\widetilde{\mathbb P}_t,\widetilde b)$ on
$\mathcal P(\mathbb R^D)$. The integrated squared-velocity bound is
preserved up to a constant and implies the integrability required by
\citet[Theorem~1.2]{pinzi2025nested}.

That theorem gives a measure $\widetilde\Gamma$ on measure-valued
trajectories with marginals $\widetilde{\mathbb P}_t$, concentrated
on solutions of the continuity equation driven by $\widetilde b$.
These trajectories remain in $\mathcal P(\jmath(K))$: the marginal
identities imply this simultaneously at rational times, and continuity
and closedness extend it to every time. We can therefore pull them
back through $J^{-1}$ to obtain $\Gamma$ with the desired marginals.
Extension of smooth tests from the embedded manifold and the chain
rule give~\eqref{eq:measure_characteristic}~\citep{pinzi2026study}.

By Fubini and the marginal identities, the assumed energy is finite
along almost every represented trajectory. The velocity-energy
estimate for the Euclidean continuity equation gives $W_2$ absolute
continuity there; metric comparison on $K$ transfers this to the
intrinsic $W_2$ metric. Applying the cylinder chain rule to
\eqref{eq:measure_characteristic} proves the second claim.

\paragraph{Disintegration and the deterministic flow.}
For the third claim, consider the joint law of the initial measure
and the entire trajectory:
\begin{equation*}
\widehat{\Gamma}:=(e_0,\operatorname{id})_\#\Gamma
\in\mathcal P(\pspace\times\Gamma_T).
\end{equation*}
Its first marginal is $(e_0)_\#\Gamma=\mathbb P_0$, and its
second marginal is $\Gamma$. Since $\pspace$ and $\Gamma_T$ are
Polish spaces, disintegration with respect to the first marginal
gives a measurable family of conditional path measures
$\{\Gamma_\mu\}_{\mu\in\pspace}$ such that
\begin{equation*}
\widehat{\Gamma}
=\int_{\pspace}\delta_\mu\otimes\Gamma_\mu\,d\mathbb P_0(\mu).
\end{equation*}
Here $\Gamma_\mu$ is the conditional law of the trajectory given
its initial value $\mu$. In particular, for $\mathbb P_0$-almost
every $\mu$,
\begin{equation*}
\Gamma_\mu\bigl(\{\gamma\in\Gamma_T:\gamma(0)=\mu\}\bigr)=1.
\end{equation*}
The concentration properties of $\Gamma$ also hold under
$\Gamma_\mu$ for $\mathbb P_0$-almost every $\mu$: its trajectories
are $W_2$-absolutely continuous, remain in $\mathcal P(K)$,
and solve~\eqref{eq:measure_characteristic}.

Under the additional uniqueness assumption, there is at most one
such trajectory starting from $\mu$. Since $\Gamma_\mu$ is a
probability measure concentrated on these trajectories, there
is exactly one, denoted $\gamma_\mu$, and
\begin{equation*}
\Gamma_\mu=\delta_{\gamma_\mu}
\qquad\text{for }\mathbb P_0\text{-almost every }\mu.
\end{equation*}
The measurability of the conditional measures therefore gives
a measurable trajectory map $\mu\mapsto\gamma_\mu$ up to null
sets. Define $\Phi_t(\mu):=\gamma_\mu(t)$, so that
$\Phi_0(\mu)=\mu$ for $\mathbb P_0$-almost every $\mu$.

To identify the law at time $t$, take any bounded Borel
functional $\varphi:\pspace\to\mathbb R$. The marginal identity
and disintegration give
\begin{align*}
\int_{\pspace}\varphi(\nu)\,d\mathbb P_t(\nu)
&=\int_{\Gamma_T}\varphi(\gamma(t))\,d\Gamma(\gamma)\\
&=\int_{\pspace}\left(
  \int_{\Gamma_T}\varphi(\gamma(t))\,d\Gamma_\mu(\gamma)
  \right)d\mathbb P_0(\mu)\\
&=\int_{\pspace}\varphi(\gamma_\mu(t))\,d\mathbb P_0(\mu)\\
&=\int_{\pspace}\varphi(\Phi_t(\mu))\,d\mathbb P_0(\mu).
\end{align*}
This is precisely the pushforward identity
\begin{equation*}
\mathbb P_t=(\Phi_t)_\#\mathbb P_0,\qquad t\in[0,T],
\end{equation*}
which establishes the deterministic Lagrangian representation.
\end{proof}

\subsubsection{Example: Flowing between two diracs centered at $\mu_0$ and $\mu_1$, with $\mathcal{M}=\mathbb{R}^d$}

We consider a probability path in $\ppspace$ between two dirac distributions, centered at $\mu_0$ and $\mu_1$. In that case, we have $\mathbb{P}_t = \delta_{\mu_t}$. We first consider the case where the underlying manifold is the Euclidean space. We treat the general case where $\mathcal{M}$ is an arbitrary Riemannian manifold in the next section. 

We take the path $\gamma$ to be the constant-speed geodesics connecting $\mu_0$ and $\mu_1$:
\begin{equation*}
\gamma(t) = \mu_t = (\psi_t)_{\#}\mu_0
\end{equation*}

\begin{equation*}
\psi_{t}(x):=(1-t)x+t\,T_0(x),\qquad t\in[0,1],
\end{equation*}

and $T_0$ the optimal transport map on $\mathbb{R}^d$ between $\mu_0$ and $\mu_1$.

The minimal norm tangent vector is the vector field $v_t \in T_{\mu_t}\pspace$:

\begin{equation*}
v_t(z) = \left.\partial_t\psi_t(x)\right|_{x=S_t(z)} = T(S_t(z)) - S_t(z) = (T-Id)(\psi_t^{-1}(z)),
\end{equation*}

with $S_t(x) = \psi_t^{-1}$.

We now need to define a vector $V_t$ on $\pspace$ that generates this path:
\begin{equation*}
V_t(\mu) =
\begin{cases}
v_t \quad \text{if } \mu = \mu_t \\
0 \quad \text{otherwise}
\end{cases}
\end{equation*}

We now show that $(\delta_{\mu_t},V_t)$ solves the continuity equation on $\pspace$. We first show that $(\mu_t,v_t)$ follows the continuity equation on $\mathbb{R}^d$, then lift it to show that it follows the continuity equation on $\mathcal{P}_2(\mathbb{R}^d)$.

\paragraph{$(\mu_t,v_t)$ on $\mathbb{R}^d$}

From the definition of $\mu_t$ and the change of variable formula, we have that 

\begin{equation*}
\int_{\mathbb{R}^d} \xi(z) d\mu_t(z) = \int_{\mathbb{R}^d} \xi(\psi_t(x))d\mu_0(x)
\end{equation*}

for any continuous test function $\xi$. Differentiating with respect to $t$ (and noting that $\dot{\psi}_t = T(x)-x$), we obtain

\begin{equation*}
\begin{split}
    \frac{d}{dt} \int_{\mathbb{R}^d} \xi(z)d\mu_t(z) &= \int_{\mathbb{R}^d} \nabla\xi(\psi_t(x))\cdot(T(x)-x)d\mu_0(x) \\
    &= \int_{\mathbb{R}^d} \nabla\xi(z) \cdot (T-Id)(\psi_t^{-1}(z))d\mu_t(z) = \int_{\mathbb{R}^d} \nabla \xi(z) \cdot v_t(z) d\mu_t(z)
\end{split}
\end{equation*}

which is the weak formulation of the continuity equation on $\mathbb{R}^d$.

We can now adapt it to time cylindrical functions of the form
\[
\varphi_t(\mu) = F\!\Bigl(t,\int_{\mathbb{R}^d}\phi_1\, d\mu,\ldots,\int_{\mathbb{R}^d}\phi_k\, d\mu\Bigr).
\]
Using the identity above, we have

\begin{equation*}
\frac{d}{dt} \int_{\mathbb{R}^d} \phi_i(z)d\mu_t(z) =  \int_{\mathbb{R}^d} \nabla \phi_i(z) \cdot v_t(z) d\mu_t(z)
\end{equation*}

 We thus have 

\begin{equation*}
\frac{d}{dt}\varphi_t(\mu_t) = \partial_tF(t,G(t)) + \sum_{i=1}^k \partial_{x_i} F(t,\int_{\mathbb{R}^d}\phi_1 d\mu,...,\int_{\mathbb{R}^d}\phi_k d\mu) \cdot \int_{\mathbb{R}^d} \nabla \phi_i(z) \cdot v_t(z) d\mu_t(z)
\end{equation*}

Using the definition of the Wasserstein gradient of the cylinder function, we write

\begin{equation*}
\begin{split}
\frac{d}{dt}\varphi_t(\mu_t) &= \partial_t \varphi_t(\mu_t) + \int_{\mathbb{R}^d} \nabla_W\varphi_t(\mu_t) \cdot v_t d\mu_t\\
&= \partial_t \varphi_t(\mu_t) + \langle \nabla_W\varphi_t(\mu_t),v_t  \rangle_{L^2(\mu_t)}
\end{split}
\end{equation*}

which is the our desired results on $\mathbb{R}^d$ that we now need to lift to $\mathcal{P}_2(\mathbb{R}^d)$.

\paragraph{$(\mathbb{P}_t,V_t)$ on $\mathcal{P}_2(\mathbb{R}^d)$}

We have

\begin{equation*}
\begin{split}
&\int_0^1\int_{\mathcal{P}_2(\mathbb{R}^d)} \left(\partial_t\varphi_t(\mu) + \langle \nabla_{\mathcal{W}}\varphi_t(\mu), V_t(\mu)\rangle_{L^2(\mu)}\right)\ d\mathbb{P}_t(\mu)dt = \\
&\int_0^1 \left(\partial_t\varphi_t(\mu_t) + \langle \nabla_{\mathcal{W}}\varphi_t(\mu_t), V_t(\mu_t)\rangle_{L^2(\mu_t)}\right) dt
\end{split}
\end{equation*}

which is equal to $\int_0^1 \frac{d}{dt}\varphi_t(\mu_t) dt = \varphi_1(\mu_1) - \varphi_0(\mu_0)$ and vanishes with the assumption that $\varphi_t$ has a compact support in time. Hence, we have

\begin{equation*}
\int_0^1\int_{\mathcal{P}_2(\mathbb{R}^d)} \left(\partial_t\varphi_t(\mu) + \langle \nabla_{\mathcal{W}}\varphi_t(\mu), V_t(\mu)\rangle_{L^2(\mu)}\right)\ d\mathbb{P}_t(\mu)dt = 0,
\end{equation*}

which shows that $(\mathbb{P}_t,V_t)$ solves the weak continuity equation on $\mathcal{P}_2(\mathbb{R}^d)$.

\subsubsection{Example: Flowing between two diracs on arbitrary Riemannian manifolds}

We now generalize the example above to general Riemannian manifolds. Considering arbitrary manifolds $(\mathcal{M},g)$, and writing geodesics between $x$ and $y$ $\in \mathcal{M}$.

\begin{equation*}
\alpha_{t}(x,y)\;:=\;\exp_{x}\!\bigl(t\,\log_{x}y\bigr),\qquad t\in[0,1],
\end{equation*}

where,
\begin{itemize}
\item $\exp_{x}:T_{x}M\!\to\!M$ is the exponential map,
\item $\log_{x}y\in T_{x}M$ is the inverse ($\exp_{x}(\log_{x}y)=y$).
\end{itemize}

The McCann displacement interpolation on $M$ is defined as 

\begin{equation*}
\psi_{t}(x):=\alpha_{t}(x,T(x))
           =\exp_{x}(t\,\log_x T(x)),
\qquad
\mu_{t}\;:=\;(\psi_{t})_{\#}\mu_{0},
\end{equation*}

which is the Riemannian analogue of the Euclidean straight-line interpolation and is still a constant-speed $W_{2}$-geodesic on $\mathcal P_{2}(M)$.

The velocity field then writes
\begin{equation*}
v_t(\cdot| \mu_0, \mu_1)=\frac{d}{dt}\psi_t(x).
\end{equation*}

We define $\mathbb{P}_t = \delta_{\mu_t}$ and $V_t(\mu) = v_t$ if $\mu=\mu_t$ and $0$ otherwise.

We can now follow the same procedure as in the Euclidean case to show that $(V_t,\mathbb{P}_t)$ solves the weak continuity equation on $\pspace$. We first verify it on $\mathcal{M}$ and then lift it to $\pspace$.

\paragraph{$(\mu_t,v_t)$ on $\mathcal{M}$}

From the definition of $\mu_t$ and the change of variable formula, we have that 

\begin{equation*}
\int_{\mathcal{M}} \xi(z) d\mu_t(z) = \int_{\mathcal{M}} \xi(\psi_t(x))d\mu_0(x)
\end{equation*}

for any continuous test function $\xi$. Differentiating with respect to $t$ (and using that $v_t(\psi_t(x)) = \dot{\psi}_t(x)$), we obtain

\begin{equation*}
\begin{split}
    \frac{d}{dt} \int_{\mathcal{M}} \xi(z)d\mu_t(z) &= \int_{\mathcal{M}} \langle \nabla\xi(\psi_t(x)), \dot{\psi}_t \rangle_g d\mu_0(x) \\
    &= \int_{\mathcal{M}} \langle \nabla\xi(\psi_t(x)), v_t(\psi_t(x))\rangle_g d\mu_0(x) \\
    &= \int_{\mathcal{M}} \langle \nabla\xi(z), v_t(z)\rangle_g d\mu_t(z)
\end{split}
\end{equation*}

This can again be extended to time-dependent cylinder functions $\varphi_t$ to yield

\begin{equation*}
\begin{split}
\frac{d}{dt}\varphi_t(\mu_t) &= \partial_t \varphi_t(\mu_t) + \int_{\mathcal{M}} \langle \nabla_W\varphi_t(\mu_t), v_t \rangle_g d\mu_t\\
&= \partial_t \varphi_t(\mu_t) + \langle \nabla_W\varphi_t(\mu_t),v_t  \rangle_{L^2(\mu_t)}
\end{split}
\end{equation*}

where $\langle\cdot,\cdot\rangle_{L^2(\mu_t)}$ implicitly encodes the Riemannian metric in the inner product.

\paragraph{$(\mathbb{P}_t,V_t)$ on $\pspace$}

Since $\mathbb{P}_t$ is concentrated on $\mu_t$, we have:

\begin{equation*}
\begin{split}
&\int_0^1\int_{\pspace} \left(\partial_t\varphi_t(\mu) + \langle \nabla_{\mathcal{W}}\varphi_t(\mu), V_t(\mu)\rangle_{L^2(\mu)}\right)\ d\mathbb{P}_t(\mu)dt = \\
&\int_0^1 \left(\partial_t\varphi_t(\mu_t) + \langle \nabla_{\mathcal{W}}\varphi_t(\mu_t), V_t(\mu_t)\rangle_{L^2(\mu_t)}\right) dt
\end{split}
\end{equation*}

which, from our derivation above, is equal to $\int_0^1 \frac{d}{dt}\varphi_t(\mu_t) dt = \varphi_1(\mu_1) - \varphi_0(\mu_0)$ and vanishes with the assumption that $\varphi_t$ has a compact support in time. Hence, we have

\begin{equation*}
\int_0^1\int_{\pspace} \left(\partial_t\varphi_t(\mu) + \langle \nabla_{\mathcal{W}}\varphi_t(\mu), V_t(\mu)\rangle_{L^2(\mu)}\right)\ d\mathbb{P}_t(\mu)dt = 0,
\end{equation*}

which shows that $(\mathbb{P}_t,V_t)$ solves the weak continuity equation on $\pspace$.

\subsection{Marginal vector field generates marginal probability path}
\label{sec:marginal}

In this section, we show that marginalizing conditional velocity fields preserves the weak continuity equation. Under the hypotheses of Theorem~\ref{thm:weak_continuity_superposition} for the marginal pair, including trajectory uniqueness, the marginal vector field also generates the marginal probability path.

\begin{assumption}[Conditional velocity and probability path solve continuity equation]
\label{ass:conditional_continuity}
We assume that for any given boundary measures $\mu_0, \mu_1$, the conditional flow $(\mathbb{P}_t(\cdot |\mu_0,\mu_1))_t$ and its velocity field $V_t(\mu | \mu_0, \mu_1)$ satisfy the weak continuity equation. This means for any suitable test (cylinder) function $\varphi_t(\mu)$, the following holds:

\begin{equation*}
\int_0^T\int_{\pspace} \left(\partial_t\varphi_t(\mu) + \langle \nabla_{\mathcal{W}}\varphi_t(\mu), V_t(\mu|\mu_0,\mu_1)\rangle_{L^2(\mu)}\right)\ d\mathbb{P}_t(\mu|\mu_0,\mu_1)dt = 0.
\end{equation*}

\end{assumption}

\begin{definition}[Marginal velocity field]
\label{def:marginal_v}
Analogously to classical flow matching, we define the marginal velocity field $V_t(\mu)$ as the conditional expectation of the conditional velocity fields $V_t(\mu|\mu_0,\mu_1)$:

\begin{equation*}
V_t(\mu) := \mathbb{E}_{\mathbb{P}_t(\mu_0,\mu_1|\mu)}[V_t(\mu|\mu_0,\mu_1)] = \int_{\pspace} V_t(\mu|\mu_0,\mu_1)\ d\mathbb{P}_t(\mu_0,\mu_1|\mu)
\end{equation*}

where $d\mathbb{P}_t(\mu_0,\mu_1|\mu) = \frac{d\mathbb{P}_t(\mu|\mu_0,\mu_1)d\Pi(\mu_0,\mu_1)}{d\mathbb{P}_t(\mu)}$ is the conditional probability measure. 

\end{definition}

This definition implies the key relation:
\begin{equation*}
V_t(\mu)d\mathbb{P}_t(\mu) = \int_{\pspace\times \pspace } V_t(\mu|\mu_0,\mu_1)\ d\mathbb{P}_t(\mu|\mu_0,\mu_1) d\Pi(\mu_0,\mu_1)
\end{equation*}

\begin{theorem}[Marginal vector field generates the marginal probability path]
\label{thm:marginalization}
Assume the conditional pair $(\mathbb{P}_t(\cdot|\mu_0,\mu_1),V_t(\cdot|\mu_0,\mu_1))$ satisfies Assumption~\ref{ass:conditional_continuity}. Then the marginal vector field $V_t$ from Definition~\ref{def:marginal_v} and the marginal probability path $\mathbb{P}_t$ solve~\eqref{eq:weak_CE_app}. Under the hypotheses of the superposition theorem (Theorem~\ref{thm:weak_continuity_superposition}) for the marginal pair, including trajectory uniqueness, $V_t$ generates the marginal probability path $\mathbb{P}_t$.
\end{theorem}

\begin{proof}[Proof of Theorem~\ref{thm:marginalization}]

We want to prove that the marginal flow $\mathbb{P}_t$ also satisfies the weak continuity equation with an appropriately defined marginal velocity field $V_t(\mu)$. That is, we aim to show that:

\begin{equation*}
\int_0^T\int_{\pspace} \left(\partial_t\varphi_t(\mu) + \langle \nabla_{\mathcal{W}}\varphi_t(\mu), V_t(\mu)\rangle_{L^2(\mu)}\right)\ d\mathbb{P}_t(\mu)dt = 0.
\end{equation*}

Let's evaluate the left-hand side of the target equation by substituting the definition of the marginal flow $d\mathbb{P}_t(\mu)$. We can split the expression into two parts.

\paragraph{Temporal derivative term}

We begin with the term containing the partial derivative with respect to time, $\partial_t\varphi_t(\mu)$. We have

 \begin{equation*}
 \int_0^T \int_{\pspace} \partial_t\varphi_t(\mu)\ d\mathbb{P}_t(\mu)dt
 \end{equation*}

1.  Substitute the definition of the marginal measure $d\mathbb{P}_t(\mu)$: \\
    \begin{equation*}
    = \int_0^T \int_{\pspace} \partial_t\varphi_t(\mu) \left( \int_{\pspace \times \pspace} d\mathbb{P}_t(\mu|\mu_0,\mu_1) d\Pi(\mu_0,\mu_1) \right) dt
    \end{equation*}

2.  By Fubini's theorem, we can exchange the order of integration: \\
    \begin{equation*}
    = \int_{\pspace\times \pspace} \left( \int_0^T \int_{\pspace} \partial_t\varphi_t(\mu)\ d\mathbb{P}_t(\mu|\mu_0,\mu_1) dt \right) d\Pi(\mu_0,\mu_1)
    \end{equation*}

3.  Using the assumed conditional continuity equation, we replace the inner integral:\\
    \begin{multline*}
    = \int_{\pspace\times \pspace} \Bigl( - \int_0^T \int_{\pspace} \langle \nabla_{\mathcal{W}}\varphi_t(\mu), V_t(\mu|\mu_0,\mu_1)\rangle_{L^2(\mu)}\\
    \cdot\, d\mathbb{P}_t(\mu|\mu_0,\mu_1)\, dt \Bigr) d\Pi(\mu_0,\mu_1)
    \end{multline*}

4.  Combining the integrals gives our final expression for the first part: \\
    \begin{multline*}
    \text{Part 1} = - \int_{\pspace\times \pspace} \int_0^T \int_{\pspace} \langle \nabla_{\mathcal{W}}\varphi_t(\mu), V_t(\mu|\mu_0,\mu_1)\rangle_{L^2(\mu)}\\
    \cdot\, d\mathbb{P}_t(\mu|\mu_0,\mu_1)\,dt\, d\Pi(\mu_0,\mu_1)
    \end{multline*}

\paragraph{Velocity term}

Now we analyze the second term involving the marginal velocity field $V_t(\mu)$:
\begin{equation*}
\text{Part 2} = \int_0^T \int_{\pspace} \langle \nabla_{\mathcal{W}}\varphi_t(\mu), V_t(\mu)\rangle_{L^2(\mu)}\ d\mathbb{P}_t(\mu)dt
\end{equation*}

1.  Substituting the definition of the marginal velocity field: \\
\begin{multline*}
\text{Part 2} = \int_0^T \int_{\pspace} \left\langle \nabla_{\mathcal{W}}\varphi_t(\mu), \int_{\pspace \times \pspace} V_t(\mu|\mu_0,\mu_1)\ d\mathbb{P}_t(\mu|\mu_0,\mu_1) \right\rangle_{L^2(\mu)}\\
\cdot\, d\Pi(\mu_0,\mu_1)\,dt
\end{multline*}

2.  Using the linearity of the inner product and the definition of conditional probability, we can combine the integrals:\\
    \begin{equation*}
    = \int_0^T \int_{\pspace} \int_{\pspace \times \pspace} \langle \nabla_{\mathcal{W}}\varphi_t(\mu), V_t(\mu|\mu_0,\mu_1) \rangle_{L^2(\mu)}\ d\mathbb{P}_t(\mu|\mu_0,\mu_1) d\Pi(\mu_0,\mu_1) dt
    \end{equation*}

Hence, $\text{Part 2} = - \text{Part 1}$, establishing the weak continuity equation for the marginal pair. Under the stated additional hypotheses, Theorem~\ref{thm:weak_continuity_superposition} then gives $\mathbb P_t=(\Phi_t)_\#\mathbb P_0$, completing the proof.
\end{proof}

We have thus shown that if we define the marginal velocity field as the conditional expectation of the conditional velocity fields, the resulting marginal flow satisfies the weak continuity equation on the Wasserstein manifold.

\subsection{Equivalence of conditional and non-conditional flow matching losses}
\label{app:gradloss}

A natural loss for the flow matching objective is 
\begin{equation*}
    \mathcal{L}_{FM} = \mathbb{E}_{t,\mu_t\sim \mathbb{P}_t}[\lVert V_t(\mu_t)-V_t^{\theta}(\mu_t)\rVert^2_{L^2(\mu_t)}]
\end{equation*}

where $\lVert V_t(\mu)\rVert^2_{L^2(\mu)} = \int_{\mathcal{M}} \lVert V_t(\mu)(x)\rVert_g^2 d\mu(x)$.

The conditional version would then be
\begin{equation*}
    \mathcal{L}_{CFM} = \mathbb{E}_{t,\mu_0,\mu_1\sim \Pi, \mu_t\sim \mathbb{P}_t(\cdot\mid \mu_0,\mu_1)} [\lVert V_t(\mu_t\mid \mu_0, \mu_1)-V_t^{\theta}(\mu_t)\rVert^2_{L^2(\mu_t)}]
\end{equation*}

Following the original flow matching proof \cite{lipman2024flow}, we obtain
\begin{equation*}
\begin{split}
    &\nabla_{\theta} \mathcal{L}_{CFM} = \mathbb{E}_{t,\mu_0,\mu_1\sim \Pi, \mu_t\sim \mathbb{P}_t(\cdot\mid \mu_0,\mu_1)} [\nabla_{\theta}\lVert V_t(\mu_t\mid \mu_0, \mu_1)-V_t^{\theta}(\mu_t)\rVert^2_{L^2(\mu_t)}]\\
    &= \mathbb{E}_{t,\mu_0,\mu_1\sim \Pi, \mu_t\sim \mathbb{P}_t(\cdot\mid \mu_0,\mu_1)} [2 \langle \nabla_{\theta} V_t^{\theta}(\mu_t), V_t^{\theta}(\mu_t)\rangle - 2\langle V_t(\mu_t\mid \mu_0, \mu_1),\nabla_{\theta} V_t^{\theta}(\mu_t)\rangle] \\
\end{split}
\end{equation*}

Focusing on the right handside in the expectation, we have

\begin{align*}
    &\mathbb{E}_{t,\mu_0,\mu_1\sim \Pi, \mu_t\sim \mathbb{P}_t(\cdot\mid \mu_0,\mu_1)} [\langle V_t(\mu_t\mid \mu_0, \mu_1),\nabla_{\theta} V_t^{\theta}(\mu_t)\rangle] \\
    &= \mathbb{E}_{t,\mu_t \sim \mathbb{P}_t} \bigl[\mathbb{E}_{\mu_0,\mu_1 \sim \mathbb{P}(\cdot\mid \mu_t)} [\langle V_t(\mu_t\mid \mu_0, \mu_1),\nabla_{\theta} V_t^{\theta}(\mu_t)\rangle] \bigr]\\
    &= \mathbb{E}_{t,\mu_t \sim \mathbb{P}_t} [\langle \mathbb{E}_{\mu_0,\mu_1 \sim \mathbb{P}(\cdot\mid \mu_t)} [ V_t(\mu_t\mid \mu_0, \mu_1)], \nabla_{\theta} V_t^{\theta}(\mu_t)\rangle]\\
    &= \mathbb{E}_{t,\mu_t \sim \mathbb{P}_t} [\langle V_t(\mu_t), \nabla_{\theta} V_t^{\theta}(\mu_t)\rangle]
\end{align*}

Hence,
\begin{equation}
    \nabla_{\theta} \mathcal{L}_{CFM} = \mathbb{E}_{t,\mu_t \sim \mathbb{P}_t} [2 \langle \nabla_{\theta} V_t^{\theta}(\mu_t), V_t^{\theta}(\mu_t)\rangle - 2 \langle V_t(\mu_t), \nabla_{\theta} V_t^{\theta}(\mu_t)\rangle ]
    \label{eq:gradlossCFM}
\end{equation}

To finish the proof, we derive the gradient of $\mathcal{L}_{FM}$
\begin{equation*}
\begin{split}
    \nabla_{\theta} \mathcal{L}_{FM} &= \mathbb{E}_{t,\mu_t \sim \mathbb{P}_t} [\nabla_{\theta}\lVert V_t(\mu_t)-V_t^{\theta}(\mu_t)\rVert^2_{L^2(\mu_t)}]\\
    &= \mathbb{E}_{t,\mu_t \sim \mathbb{P}_t} [2 \langle \nabla_{\theta} V_t^{\theta}(\mu_t), V_t^{\theta}(\mu_t)\rangle - 2\langle V_t(\mu_t),\nabla_{\theta} V_t^{\theta}(\mu_t)\rangle],
\end{split}
\end{equation*}

which coincides exactly with~\eqref{eq:gradlossCFM}.


\section{Experimental Details and additional results}\label{appendix:hyperparameters}

Here we lay out the settings and hyperparameters used in our experiments. We discuss how datasets were processed, errors and metrics were computed, and model architectures and training procedures.

\subsection{Neural Network Architecture \& Training}

We parameterize the velocity field as a neural network $v_\theta(X_t, t, c) : \mathcal{M}^N \times [0,1] \times \mathcal{C} \to T_{X_t}\mathcal{M}^N$, which maps a point cloud $X_t = \{x_1, \ldots, x_N\}$ with $x_i \in \mathcal{M}$, time $t \in [0,1]$, and optional condition $c \in \mathcal{C}$ to a set of tangent vectors at $X_t$. The network architecture is designed to respect the permutation equivariance of point clouds and optimal transport maps.

The architecture consists of three components: (1) an embedding layer that maps each point $x_i \in \mathcal{M}$ to a latent representation, (2) six multi-head self-attention blocks that process the embedded point cloud, and (3) an unembedding layer that projects back to the tangent space $T_{X_t}\mathcal{M}^N$. Each self-attention block processes the latent representation $h \in \mathbb{R}^{N \times d}$ as follows. First, we incorporate temporal and conditional information by adding learned embeddings to form $\tilde{h} = h + \text{Embed}_t(\phi(t)) + \text{Embed}_c(c)$, where $\phi(t)$ denotes Fourier features of time $t$. The block then applies:
\begin{align*}
h' &\leftarrow h + \text{MHA}(\text{LayerNorm}(\tilde{h})) \\
h &\leftarrow h' + \text{MLP}(\text{LayerNorm}(h'))
\end{align*}
where the multi-head attention (MHA) uses 4 heads, and the residual connections are to the original stream $h$. After the final attention block, the unembedding layer projects the latent representation back to the ambient dimension of $\mathcal{M}$ to produce the tangent velocity vectors.

\paragraph{Classifier-Free Guidance} For conditional generation, we employ classifier-free guidance to improve sample quality and controllability. The conditioning embedding $\text{Embed}_c(c)$ projects the condition $c$ to the unit sphere in the embedding dimension, which creates a meaningful semantic separation between null and real conditions. Real conditions lie on the unit sphere while the null condition is represented by the zero vector at the origin. During training, we randomly drop conditions with a probability of $p$, and for these null conditions, we set the embedding $\text{Embed}_c(c)$ to the zero vector (not the input $c$ itself). At inference time, we use the classifier-free guidance formula:
\begin{equation*}
v_\theta^{\text{CFG}}(X_t, t, c) = v_\theta(X_t, t, \emptyset) + w \cdot \left(v_\theta(X_t, t, c) - v_\theta(X_t, t, \emptyset)\right)
\end{equation*}
where $w \geq 1$ is the guidance weight, and $\emptyset$ denotes the null condition. By default we set $p=0.1$ and $w=2$.

\paragraph{Training Details} By default, we train all models for 500,000 steps using the Adam optimizer \citep{kingma2014adam} with a learning rate of $3 \times 10^{-4}$. We apply learning rate decay with a factor of 0.99 every 5,000 steps. During training, we sample point clouds of $\min(N,1024)$ (where N is the number of particles in the empirical measure) points from each distribution with a batch size of 32. The time variable $t$ is sampled uniformly from $[0,1]$ for each training step, and generation is performed using integration via 1000 Euler steps. 

For computing the Riemannian entropic map, we set the entropic regularization parameter $\varepsilon = 0.002$, where all distance matrices are scaled to have a maximum value of 1. The number of Sinkhorn iterations is determined automatically before training by sampling 100 pairs of distributions and finding the minimum number of iterations required to achieve convergence for 95\% of the pairs.  

\paragraph{Mini-batch OT Coupling} For unconditional generation, we enhance training by matching samples within mini-batches using optimal transport \citep{pooladian2023multisample}. This approach effectively performs OT in the Wasserstein space itself \citep{emami2025optimal, bonet2025flowing}. Given source distributions $\{\mu_i\}_{i=1}^{B_s} \sim \mathbb{P}_0$ (noise) and target distributions $\{\nu_j\}_{j=1}^{B_t} \sim \mathbb{P}_1$ (data) within a mini-batch, we construct a cost matrix $C \in \mathbb{R}^{B_s \times B_t}$ and sample training pairs $(\mu_i, \nu_j) \sim \pi^*$ from the resulting optimal coupling $\pi^*$.

To ensure scalability, we compute the cost matrix using the geometric Chamfer distance rather than the Wasserstein distance:
\begin{equation*}
C_{i,j} = \text{CD}(\mu_i, \nu_j) = \frac{1}{|\mu_i|}\sum_{x \in \mu_i} \min_{y \in \nu_j} d_{\mathcal{M}}(x,y) + \frac{1}{|\nu_j|}\sum_{y \in \nu_j} \min_{x \in \mu_i} d_{\mathcal{M}}(x,y)
\end{equation*}
where $d_{\mathcal{M}}$ is the intrinsic geodesic distance on the manifold $\mathcal{M}$. The optimal transport plan $\pi^*$ is computed using the Sinkhorn algorithm with entropic regularization $\varepsilon = 0.001$, where distance matrices are scaled to maximum value 1 and the number of iterations is determined automatically as described above.

All models were implemented in JAX \citep{frostig2019compiling} using optimal transport tools from \citep{cuturi2022optimal} and trained on a single NVIDIA B200 GPU, taking 3-6 hours depending on the dataset and manifold geometry.

\paragraph{Source Noise Generation} 
While rigorously defined distributions exist on non-Euclidean domains (e.g., von Mises or wrapped normal distributions), flow matching only requires a source that is easy to sample from, without needing a closed-form likelihood. We note that naively sampling noise as $\text{proj}_{\mathcal{M}}(\mathcal{N}(0, I))$ (where $\text{proj}_{\mathcal{M}}$ is the projection onto the manifold) fails to produce meaningful variability. Our source distribution must be a distribution over \textit{distributions} on $\mathcal{M}$. Indeed, for large enough sample sizes $n$, all generated ``noise'' point clouds become nearly identical, resulting in a degenerate distribution. 

Instead, following \citet{haviv2024wasserstein}, we employ a two-level hierarchical sampling scheme. First, we randomly sample a mean $\mu$ and covariance $\Sigma$, then generate points from $\pi(\mathcal{N}(\mu, \Sigma))$. The distribution for $\mu$ is derived from the empirical mean and covariance of all points in the training data, while the distribution for $\Sigma$ is based on the per-sample covariances. Specifically, let $\{L_i\}$ be the Cholesky factors of the per-sample covariances. We compute the element-wise mean $\bar{L}$ and standard deviation $\sigma_L$ of these lower-triangular matrices, sample a random lower-triangular matrix $\tilde{L} \sim \mathcal{N}(\bar{L}, \text{diag}(\sigma_L^2))$, and set $\Sigma = \tilde{L}\tilde{L}^\top$. This construction ensures $\Sigma$ is positive semidefinite. Samples from the  resulting mean and covariance are then projected onto the manifold geometry. For high-dimensional datasets (e.g., single-cell data on $\mathbb{S}^{128-1}$), we only sample diagonal covariances for computational efficiency.

\subsection{Geometric Operations on Manifolds}

We provide explicit formulas for the geometric operations used in our method across different manifolds. For each geometry, we define: (1) the squared geodesic distance $d^{2}_{\mathcal{M}}(p, q)$, (2) the geodesic interpolant $\gamma(t; p_0, p_1)$ connecting $p_0$ to $p_1$ at time $t \in [0,1]$, (3) the tangent velocity $v(t; p_0, p_1) = \frac{d}{dt}\gamma(t; p_0, p_1)$, (4) the exponential map $\exp_p(v, \Delta t)$ that moves from point $p$ along tangent vector $v$ by time $\Delta t$, and (5) the tangent norm $\|\cdot\|_{T_p\mathcal{M}}^2$ used for computing the training loss between predicted and target velocities. 

\textbf{Important note on representation:} We always represent manifolds using extrinsic coordinates embedded in Euclidean space (e.g., $\mathbb{S}^2$ as 3D unit vectors in $\mathbb{R}^3$, $\mathbb{H}^2$ as 3D vectors in Lorentz space). This extrinsic representation significantly simplifies neural network modeling, as the network can operate on fixed-dimensional Euclidean vectors while geometric constraints are enforced through projection operations. These formulas correspond directly to our implementation and are included here for completeness and reproducibility.

\paragraph{Euclidean Space $\mathbb{R}^d$}
The standard flat geometry. Points $p \in \mathbb{R}^d$ have trivial tangent spaces $T_p\mathbb{R}^d \cong \mathbb{R}^d$. The geodesics are straight lines.

\noindent\textit{Distance and interpolation:}
\begin{equation*}
d^{2}_{\mathbb{R}^d}(p, q) = \|p - q\|_{2}^2, \qquad \gamma(t; p_0, p_1) = (1-t)p_0 + tp_1
\end{equation*}

\noindent\textit{Velocity and exponential map:}
\begin{equation*}
v(t; p_0, p_1) = p_1 - p_0, \qquad \exp_p(v, \Delta t) = p + v \cdot \Delta t
\end{equation*}

\noindent\textit{Tangent space:} $\|v - w\|_{T_p\mathbb{R}^d}^2 = \|v - w\|^2$. Projection: $\text{proj}(p) = p$.

\paragraph{$d$-Torus $\mathbb{T}^d$}
The $d$-dimensional torus $\mathbb{T}^d = (\mathbb{S}^1)^d$. Points are represented as $d$ angles $p \in [0, 2\pi)^d$. The geodesic distance accounts for periodic wraparound in each coordinate.

\noindent\textit{Distance and logarithmic map:}
\begin{align*}
d^{2}_{\mathbb{T}^d}(p, q) &= \sum_{i=1}^d \min(\lvert p_i - q_i \rvert, 2\pi - \lvert p_i - q_i \rvert)^2,\\
\log_{p_0}(p_1) &= \arctan2(\sin(p_1 - p_0), \cos(p_1 - p_0))
\end{align*}

\noindent\textit{Interpolation and velocity:}
\begin{equation*}
\gamma(t; p_0, p_1) = (p_0 + t \cdot \log_{p_0}(p_1)) \mod 2\pi, \qquad v(t; p_0, p_1) = \log_{p_0}(p_1)
\end{equation*}

\noindent\textit{Exponential map:} $\exp_p(v, \Delta t) = (p + v \cdot \Delta t) \mod 2\pi$.

\noindent\textit{Tangent space:} $\|v - w\|_{T_p\mathbb{T}^d}^2 = \|v - w\|^2$. Projection: $\text{proj}(p) = p \mod 2\pi$.

\paragraph{$d$-Sphere $\mathbb{S}^d$}
The $d$-dimensional sphere embedded in $\mathbb{R}^{d+1}$. Points are unit vectors $p \in \mathbb{R}^{d+1}$ with $\|p\| = 1$. The tangent space at $p$ is $T_p\mathbb{S}^d = \{v \in \mathbb{R}^{d+1} : \langle v, p \rangle = 0\}$. We use spherical linear interpolation (SLERP).

\noindent\textit{Distance:} $d^{2}_{\mathbb{S}^d}(p, q) = \arccos(\text{clip}(\langle p, q \rangle, -1, 1))^2$, where $\theta = \arccos(\langle p_0, p_1 \rangle)$.

\noindent\textit{Interpolation:}
\begin{equation*}
\gamma(t; p_0, p_1) = \begin{cases}
\frac{\sin((1-t)\theta)}{\sin(\theta)} p_0 + \frac{\sin(t\theta)}{\sin(\theta)} p_1 & \text{if } \sin(\theta) \geq 10^{-6} \\
(1-t)p_0 + tp_1 & \text{otherwise}
\end{cases}
\end{equation*}

\noindent\textit{Velocity:}
\begin{equation*}
v(t; p_0, p_1) = \begin{cases}
-\frac{\theta \cos((1-t)\theta)}{\sin(\theta)} p_0 + \frac{\theta \cos(t\theta)}{\sin(\theta)} p_1 & \text{if } \sin(\theta) \geq 10^{-6} \\
-p_0 + p_1 & \text{otherwise}
\end{cases}
\end{equation*}

\noindent\textit{Exponential map:}
\begin{equation*}
\exp_p(v, \Delta t) = \begin{cases}
\cos(\|v\| \Delta t) p + \sin(\|v\| \Delta t) \frac{v}{\|v\|} & \text{if } \|v\| \geq 10^{-6} \\
\text{normalize}(p + v \Delta t) & \text{otherwise}
\end{cases}
\end{equation*}

\noindent\textit{Tangent space:} $\|v - w\|_{T_p\mathbb{S}^d}^2 = \|v_{\tan} - w_{\tan}\|^2$ where $v_{\tan} = v - \langle v, p \rangle p$. Projection: $\text{proj}(p) = p/\|p\|$.

\paragraph{Hyperbolic Space $\mathbb{H}^d$ (Lorentz Model)}
Hyperbolic space in the Lorentz (hyperboloid) model. Points lie on $\{x \in \mathbb{R}^{d+1} : \langle x, x \rangle_L = -1, x_0 > 0\}$ with Minkowski inner product $\langle x, y \rangle_L = -x_0 y_0 + \sum_{i=1}^d x_i y_i$. The tangent space at $p$ is $T_p\mathbb{H}^d = \{v : \langle v, p \rangle_L = 0\}$.

\noindent\textit{Distance:}
\[
d^{2}_{\mathbb{H}^d}(p, q) = \mathrm{arccosh}\!\bigl(\mathrm{clip}(-\langle p, q \rangle_L,\, 1 + 10^{-7},\, \infty)\bigr)^2,
\quad \omega = \mathrm{arccosh}(-\langle p_0, p_1 \rangle_L).
\]

\noindent\textit{Interpolation:}
\begin{equation*}
\gamma(t; p_0, p_1) = \begin{cases}
\frac{\sinh((1-t)\omega)}{\sinh(\omega)} p_0 + \frac{\sinh(t\omega)}{\sinh(\omega)} p_1 & \text{if } \sinh(\omega) \geq 10^{-6} \\
(1-t)p_0 + tp_1 & \text{otherwise}
\end{cases}
\end{equation*}

\noindent\textit{Velocity:}
\begin{equation*}
v(t; p_0, p_1) = \begin{cases}
\tfrac{-\omega \cosh((1-t)\omega)}{\sinh(\omega)} p_0 + \tfrac{\omega \cosh(t\omega)}{\sinh(\omega)} p_1 & \sinh(\omega) \geq 10^{-6} \\
-p_0 + p_1 & \text{otherwise}
\end{cases}
\end{equation*}

\noindent\textit{Exponential map} (with $c = \|v\|_L$, $\|v\|_L = \sqrt{\max(\langle v,v\rangle_L,0)}$):
\hfuzz=9pt
\begin{equation*}
\exp_p(v,\Delta t) = \begin{cases}
\cosh(c\Delta t)\,p + \tfrac{\sinh(c\Delta t)}{c}\,v & c\geq 10^{-6} \\
\mathrm{proj}(p+v\Delta t) & \text{otherwise}
\end{cases}
\end{equation*}
\hfuzz=0.1pt

\noindent\textit{Tangent space:} $\|v - w\|_{T_p\mathbb{H}^d}^2 = \langle v_{\tan} - w_{\tan}, v_{\tan} - w_{\tan} \rangle_L$ where $v_{\tan} = v + \langle v, p \rangle_L p$. Projection: $\text{proj}(p) = [\sqrt{1 + \|p_{1:d}\|^2}, p_1, \ldots, p_d]^\top$.

\begin{table}[p]
\centering
\scriptsize
\renewcommand{\arraystretch}{0.85}
\setlength{\tabcolsep}{3pt}
\begin{tabular}{llrrrlr}
\toprule
Dataset & Method & Train $N$ & Test $N$ & $|\mathcal{P}|$ & Sink.\ iters & Time \\
\midrule
  MNIST 3 & FM & 6,131 & 1,010 & 237 & -- & 2h\,15m \\
   & RFM &  &  &  & -- & 1h\,39m \\
   & Set-FM &  &  &  & -- & 4h\,17m \\
   & Set-RFM &  &  &  & -- & 3h\,08m \\
   & WFM  &  &  &  & 1790 & 9h\,12m \\
   & RWEFM  &  &  &  & 850 & 6h\,04m \\
   & PSF &  &  &  & -- & $\sim$12h \\
   & PVD &  &  &  & -- & $\sim$12h \\
  \cmidrule(l){1-7}
  MNIST 4 & FM & 5,842 & 982 & 215 & -- & 2h\,11m \\
   & RFM &  &  &  & -- & 1h\,36m \\
   & Set-FM &  &  &  & -- & 4h\,13m \\
   & Set-RFM &  &  &  & -- & 3h\,11m \\
   & WFM  &  &  &  & 2400 & 10h\,54m \\
   & RWEFM  &  &  &  & 930 & 6h\,00m \\
   & PSF &  &  &  & -- & $\sim$12h \\
   & PVD &  &  &  & -- & $\sim$12h \\
  \cmidrule(l){1-7}
  MNIST 8 & FM & 5,851 & 974 & 276 & -- & 2h\,12m \\
   & RFM &  &  &  & -- & 1h\,41m \\
   & Set-FM &  &  &  & -- & 4h\,02m \\
   & Set-RFM &  &  &  & -- & 3h\,12m \\
   & WFM  &  &  &  & 1710 & 10h\,29m \\
   & RWEFM  &  &  &  & 810 & 6h\,34m \\
   & PSF &  &  &  & -- & $\sim$12h \\
   & PVD &  &  &  & -- & $\sim$12h \\
\midrule
  EMNIST h & FM & 4,800 & 800 & 300 & -- & 46m \\
   & RFM &  &  &  & -- & 46m \\
   & Set-FM &  &  &  & -- & 2h\,03m \\
   & Set-RFM &  &  &  & -- & 1h\,58m \\
   & WFM  &  &  &  & 490 & 3h\,33m \\
   & RWEFM  &  &  &  & 600 & 4h\,06m \\
   & PSF &  &  &  & -- & $\sim$12h \\
   & PVD &  &  &  & -- & $\sim$12h \\
  \cmidrule(l){1-7}
  EMNIST w & FM & 4,800 & 800 & 306 & -- & 46m \\
   & RFM &  &  &  & -- & 47m \\
   & Set-FM &  &  &  & -- & 2h\,01m \\
   & Set-RFM &  &  &  & -- & 2h\,00m \\
   & WFM  &  &  &  & 530 & 3h\,52m \\
   & RWEFM  &  &  &  & 720 & 4h\,22m \\
   & PSF &  &  &  & -- & $\sim$12h \\
   & PVD &  &  &  & -- & $\sim$12h \\
  \cmidrule(l){1-7}
  EMNIST y & FM & 4,800 & 800 & 278 & -- & 43m \\
   & RFM &  &  &  & -- & 44m \\
   & Set-FM &  &  &  & -- & 2h\,09m \\
   & Set-RFM &  &  &  & -- & 1h\,50m \\
   & WFM  &  &  &  & 540 & 3h\,42m \\
   & RWEFM  &  &  &  & 650 & 4h\,00m \\
   & PSF &  &  &  & -- & $\sim$12h \\
   & PVD &  &  &  & -- & $\sim$12h \\
\midrule
  KMNIST ki & FM & 6,000 & 1,000 & 464 & -- & 1h\,30m \\
   & RFM &  &  &  & -- & 1h\,44m \\
   & Set-FM &  &  &  & -- & 2h\,13m \\
   & Set-RFM &  &  &  & -- & 2h\,15m \\
   & WFM  &  &  &  & 420 & 4h\,51m \\
   & RWEFM  &  &  &  & 440 & 5h\,11m \\
  \cmidrule(l){1-7}
  KMNIST na & FM & 6,000 & 1,000 & 418 & -- & 1h\,51m \\
   & RFM &  &  &  & -- & 2h\,03m \\
   & Set-FM &  &  &  & -- & 2h\,43m \\
   & Set-RFM &  &  &  & -- & 2h\,53m \\
   & WFM  &  &  &  & 350 & 4h\,23m \\
   & RWEFM  &  &  &  & 480 & 5h\,09m \\
  \cmidrule(l){1-7}
  KMNIST ma & FM & 6,000 & 1,000 & 428 & -- & 1h\,47m \\
   & RFM &  &  &  & -- & 1h\,37m \\
   & Set-FM &  &  &  & -- & 2h\,13m \\
   & Set-RFM &  &  &  & -- & 2h\,41m \\
   & WFM  &  &  &  & 420 & 4h\,37m \\
   & RWEFM  &  &  &  & 460 & 5h\,05m \\
\midrule
  MDcath & Set-FM & 3,432 & 858 & 127,000 & -- & 4h\,23m \\
   & Set-RFM &  &  &  & -- & 5h\,11m \\
   & WFM  &  &  &  & 600 & 7h\,32m \\
   & RWEFM  &  &  &  & 1090 & 9h\,15m \\
\midrule
  scRNA-seq & Set-FM & 1,200 & 240 & 5,000 & -- & 7h\,33m \\
   & Set-RFM &  &  &  & -- & 8h\,44m \\
   & WFM  &  &  &  & 420 & 16h\,08m \\
   & RWEFM  &  &  &  & 310 & 14h\,47m \\
\bottomrule
\end{tabular}
\caption{\textbf{Wall-clock training times} Set-FM/Set-RFM pay an extra self-attention cost over point-cloud pairs; WFM and RWEFM pay a per-step Sinkhorn OT cost whose iteration count (column 6) is determined automatically per experiment. Even on the same dataset, WFM and RWEFM require different iteration counts because they use different distance kernels (ambient Euclidean vs.\ Riemannian geodesic) leading to different convergence rates. Dashes (--): no OT coupling.}
\label{tab:timing}
\end{table}

\paragraph{Handling cut loci on the manifolds.} The validity of our theorems relies on the assumption that the $\log$ map is single valued and smooth. That said, for the manifolds considered above, the cut locus of a fixed point is a measure-zero set, such that ambiguity occurs with probability zero. In practice, our implementation numerically disambiguates these situations by returning a single value for the logarithmic map, as the above velocity expressions show.

\subsection{Benchmarking metrics for generation of distributions on Manifolds}
\label{sec:benchmark-metrics}

\begin{figure}[h]
\centering
\includegraphics[width=0.7\linewidth]{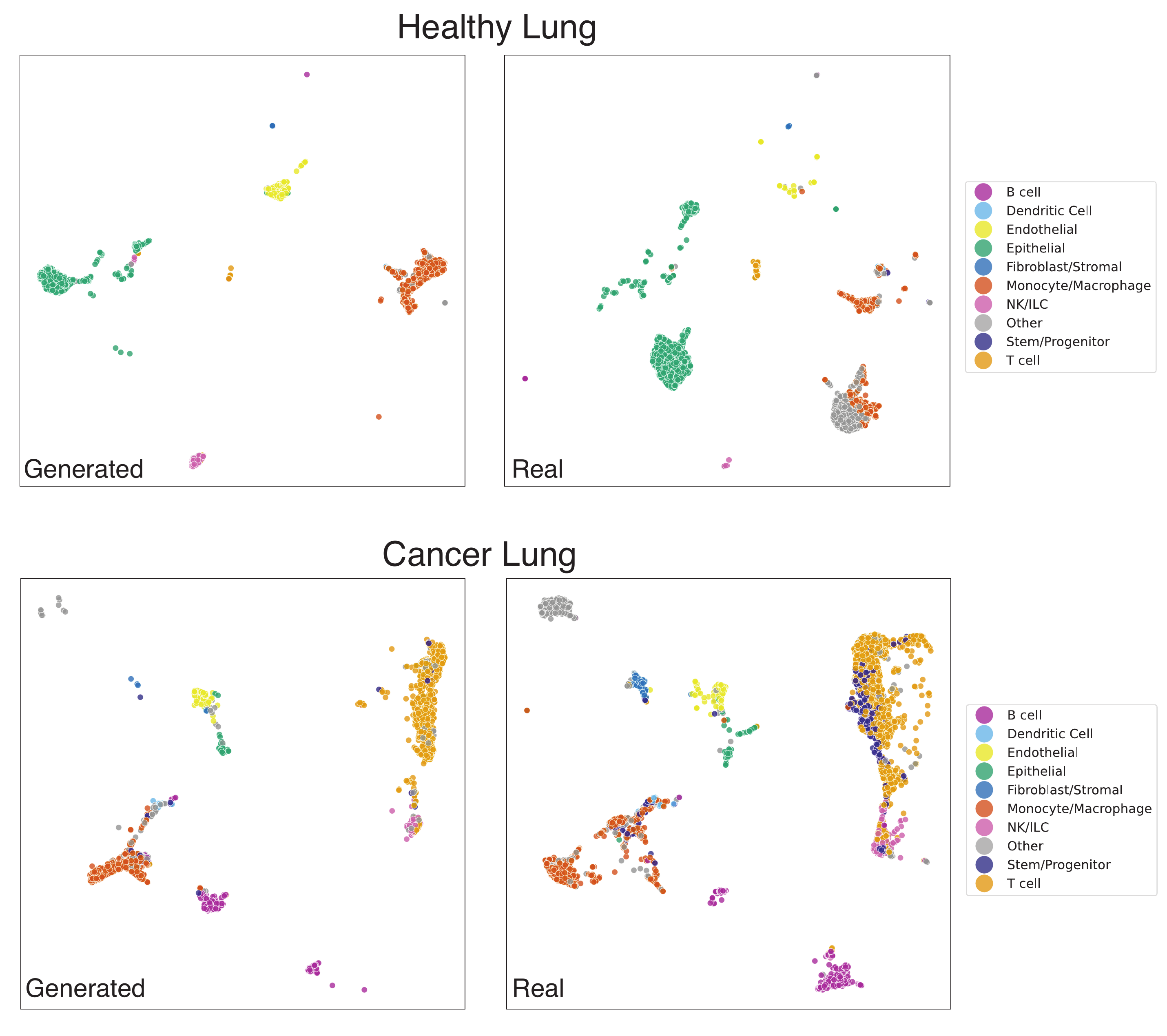}
\caption{\textbf{Individual whole single-cell samples generated by RWEFM.} We show additional individual samples generated by RWEFM in the latent space of SCimilarity \citep{heimberg2025cell}, a foundation model for single-cell data that operates in $\mathbb{S}^{128-1}$. Each panel shows the UMAP visualization of a single-cell sample, with cells colored by their cell type. Generated samples (left) closely match the cellular composition and structure of true samples (right).}
\label{fig:pascientflow_indv}
\end{figure}

For unconditional generation, we compare real and generated distributions using a classwise nearest-neighbor accuracy deviation and Maximum Mean Discrepancy (MMD), building on point-cloud evaluation measures such as those in~\citet{zhou20213d}. Each observation in this evaluation is an entire point cloud. We generate as many clouds as there are clouds in the held-out test set, pool the two sets, and classify each cloud as real or generated using the label of its nearest neighbor, excluding the cloud itself. Distances between clouds are computed using geometric CD or EMD.

Let $a_{\mathrm{real}}$ be the fraction of real clouds correctly classified as real and $a_{\mathrm{gen}}$ the fraction of generated clouds correctly classified as generated. We report the classwise 1-NN deviation abbreviated 1-NN-D in the tables. In particular, this is not the raw pooled classification accuracy. With equally sized classes, the pooled accuracy is $(a_{\mathrm{real}}+a_{\mathrm{gen}})/2$. Both $(a_{\mathrm{real}},a_{\mathrm{gen}})=(1,0)$ and $(0.5,0.5)$ give pooled accuracy $0.5$, although the first case labels every cloud as real. Our score distinguishes these cases, giving $0.5$ and $0$, respectively.

Lower values therefore indicate smaller classwise deviations from 50\% accuracy. The minimum value $0$ means that both empirical classwise accuracies equal $0.5$; the maximum is $0.5$. Finite-sample fluctuations can yield a nonzero score even when the real and generated distributions coincide.

We also compute the Maximum Mean Discrepancy (MMD) using both geometric CD and EMD kernels. The MMD is defined as:
\begin{equation*}
\text{MMD}^2(\mu, \nu) = \mathbb{E}_{x,x' \sim \mu}[k(x,x')] + \mathbb{E}_{y,y' \sim \nu}[k(y,y')] - 2\mathbb{E}_{x \sim \mu, y \sim \nu}[k(x,y)]
\end{equation*}
where $k(\cdot, \cdot)$ is a kernel function. We use $k(x,y) = \exp(-d(x,y)/\sigma)$ where $d$ is either the geometric CD or EMD distance, and $\sigma = 0.1$ is a scaling factor. We report both MMD-CD and MMD-EMD values.

For conditional generation, where we have a known ground-truth target distribution we are trying to match, we directly compare the generated distribution to the ground-truth using geometric Wasserstein distances $W_1$ and $W_2$, as well as MMD. All distances are computed using the intrinsic geometry of the underlying manifold.

\subsection{MNIST \& EMNIST on Sphere and Hyperbolic Space}

In \cref{fig:mnist_emnist_figures} we show how RWEFM can learn to generate distributions on the sphere and hyperbolic space. We use the MNIST digit dataset \citep{lecun1998gradient} for the sphere experiments and EMNIST letters \citep{cohen2017emnist} for the hyperbolic experiments. Both datasets consist of $28\times 28$ grayscale images of handwritten digits/letters. We convert each image to a point cloud in $\mathbb{R}^2$ by thresholding each pixel and normalizing the coordinates to $[-1,1]$. This transforms each image from a point in $\mathbb{R}^{28\times 28}$ to a distribution over $\mathbb{R}^2$.

For MNIST, we convert the point-cloud over $\mathbb{R}^2$ to a distribution over $\mathbb{S}^2$. We treat the $x$ and $y$ coordinates of each point as longitude and latitude on the sphere and produce the $3D$ coordinates via the spherical to Cartesian conversion. For EMNIST, we use the Lorentz model of hyperbolic space $\mathbb{H}^2$. We convert the point-cloud over $\mathbb{R}^2$ to a distribution over $\mathbb{H}^2$ by simply keeping the $x$ and $y$ coordinates as is and adding a $z$ coordinate such that each point lies on the hyperboloid defined by $-x^2 - y^2 + z^2 = 1$.

For the torus experiments, we use KMNIST \citep{kmnist2018}, a dataset of $28\times28$ grayscale images of 10 handwritten Kanji characters. As with MNIST and EMNIST, we threshold each pixel and normalize to obtain a point cloud in $\mathbb{R}^2$. To embed on the 2-torus $\mathbb{T}^2 = \mathbb{S}^1 \times \mathbb{S}^1$, we rescale the $x$ and $y$ coordinates from $[-1,1]$ to $[0,2\pi]$, treating each as an angular coordinate on the torus.

\begin{figure}[h]
\centering
\includegraphics[width=0.8\linewidth]{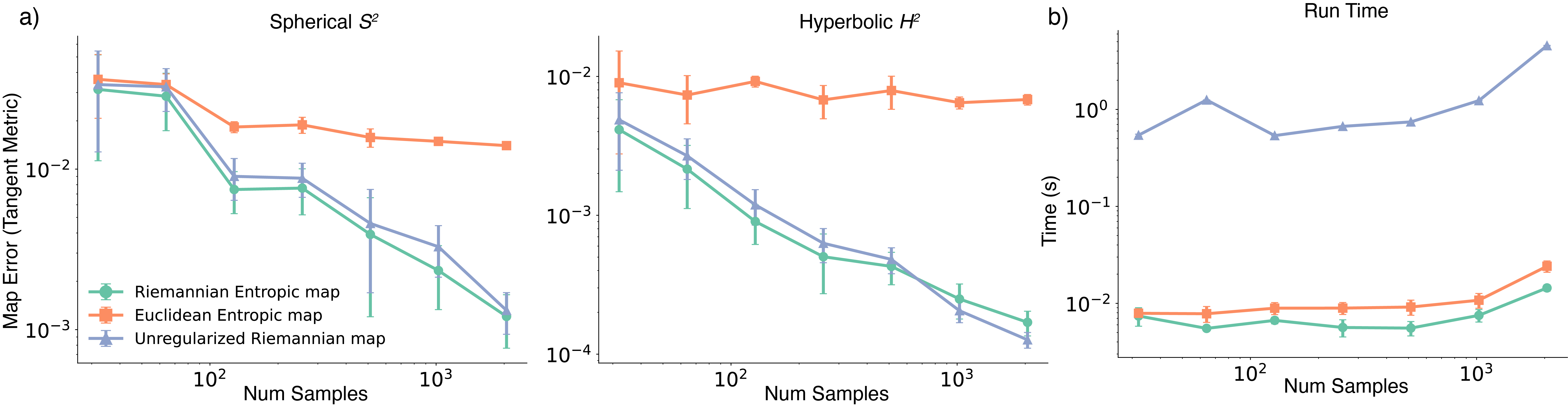}
\caption{\textbf{Entropic Map Benchmark.} \textbf{a.} We evaluate our Riemannian entropic map estimator on constructed examples on the sphere $\mathbb{S}^2$ and hyperbolic space $\mathbb{H}^2$ where the ground-truth map is known. This ground truth arises from shifting an original distribution via the Riemannian gradient of a convex function over each space. We compare the estimation error against the Euclidean entropic map and the unregularized OT map. \textbf{b.} Computational efficiency of our Riemannian entropic map implementation compared the Euclidean map and solving unregularized OT.}
\label{fig:entropic_map_benchmark}
\end{figure}

\subsection{Distributions on a General Mesh (Stanford Bunny)}
\label{app:bunny_mesh}

To demonstrate RWEFM on a geometry without closed-form $\exp$/$\log$ maps, we generate MNIST digit distributions on the surface of the Stanford bunny. We reduce the mesh resolution to $\sim$7k faces by vertex clustering, then center it at its centroid and rescale by the largest absolute coordinate to map it isotropically into $[-1,1]^3$. We equip the mesh with the spectral (biharmonic) premetric of \citet{chen2023flow}. Concretely, we build the discrete Laplace--Beltrami operator from the cotangent stiffness matrix $L$---with edge weights $\tfrac12(\cot\alpha_{ij}+\cot\beta_{ij})$, where $\alpha_{ij},\beta_{ij}$ are the two angles opposite edge $(i,j)$---and the lumped (barycentric) mass matrix $M=\mathrm{diag}(m_i)$, where $m_i=\tfrac13\sum_{f\ni i}A_f$ sums one third of the areas $A_f$ of the faces incident to vertex $i$. We solve the generalized eigenproblem $L\phi_i=\lambda_i M\phi_i$ for its $k=100$ smallest eigenpairs $(\lambda_i,\phi_i)$ and define the squared distance $d^2(x,y)=\sum_{i}\lambda_i^{-2}\,(\phi_i(x)-\phi_i(y))^2$, evaluated at surface points by barycentric interpolation over the nearest triangle. Because the spectral premetric is not geodesic ($\lVert\nabla d\rVert\neq 1$), we integrate the premetric conditional vector field of \citet{chen2023flow} directly for the interpolant, and evaluate the Riemannian entropic map from the distance function and mesh projection only (Appendix~\ref{app:entropic_no_exp_no_log}). Each MNIST image is binarized with Otsu's threshold; from the foreground pixels we sample a fixed $N=150$ points (with replacement when fewer than $150$ are available), normalize them to $[-1,1]^2$, and add small Gaussian jitter. Each resulting planar cloud is then laid onto a fixed tangent chart on the bunny's flank via the mesh exponential map.

We train one model per digit class $\{0,2,9\}$ for $100{,}000$ steps, using the same network architecture and optimizer as the closed-form experiments. \textbf{In this experiment, the methods denoted RWEFM and WFM use the sampled OT map} rather than the barycentric entropic map used elsewhere in the paper: instead of transporting each source particle to a weighted average of target particles, we assign it to a single target particle drawn from the entropic plan (the same sampled map used for the single-cell experiments). RWEFM and SetRFM operate on the mesh with the spectral metric, whereas WFM and SetFM operate in ambient $\mathbb{R}^3$ and their generated clouds are projected onto the surface for scoring; SetRFM and SetFM use random (identity) couplings without OT. Extended MMD metrics are reported in \cref{tab:bunny_mmd}.

\begin{table}[h!]
\centering
\small
\begin{tabular}{lcccccc}
\toprule
& \multicolumn{2}{c}{Digit 0} & \multicolumn{2}{c}{Digit 2} & \multicolumn{2}{c}{Digit 9} \\
\cmidrule(lr){2-3}\cmidrule(lr){4-5}\cmidrule(lr){6-7}
Method & CD & EMD & CD & EMD & CD & EMD \\
\midrule
SetFM  & 0.0414 & 0.2233 & 0.0229 & 0.1705 & 0.0287 & 0.1699 \\
WFM    & 0.0034 & 0.0799 & 0.0012 & 0.0604 & 0.0085 & 0.1157 \\
\addlinespace
SetRFM & 0.0008 & 0.0483 & 0.0010 & 0.0552 & 0.0010 & 0.0539 \\
RWEFM  & \textbf{0.0007} & \textbf{0.0468} & \textbf{0.0010} & \textbf{0.0538} & \textbf{0.0009} & \textbf{0.0527} \\
\bottomrule
\end{tabular}
\caption{\textbf{MMD on the Stanford bunny mesh} (lower is better), with Chamfer (CD) and Earth Mover's (EMD) ground metrics under the mesh spectral distance, for MNIST digits generated on the bunny. Companion to \cref{tab:bunny_mnist}; best per column in bold. RWEFM and WFM use the sampled OT map.}
\label{tab:bunny_mmd}
\end{table}

\subsection{de novo generation of Single-Cell Samples on Spherical spaces}

We demonstrate the ability of RWEFM to generate de-novo single-cell samples in the latent space of SCimilarity \citep{heimberg2025cell}, which is a foundation model for single-cell data that operates in $\mathbb{S}^{128-1}$. SCimilarity was used to embed a single-cell atlas from healthy blood samples of $1,200$ human donors, each consisting of $5,000$ cells profiled with $20,000$ genes across various patient conditions. For benchmarking (as in Tables \ref{tab:scrnaseq_cd} \& \ref{tab:scrnaseq_ot}), we perform unconditional generation, randomly holding out $240$ donors and evaluating the quality of $240$ generated samples against the true held-out samples. We note that due to the high dimensionality of the space, we used a \textit{sampled} map instead of the entropic. Briefly, instead of assigning each particle in the source distribution to a weighted average of all particles in the target distribution, we assign each particle to a single particle in the target distribution based on the optimal transport plan. We believe the curse of dimensionality makes the entropic map less effective in this setting, as the entropic map points to unrealistic barycenters of target samples, whereas the sampled map points to actual target samples.

Furthermore, we demonstrate class-conditional generation by training a flow to cell distribution conditioned on tissue status (healthy vs diseased) on an expanded dataset which included pathological samples. In \cref{fig:pascientflow}, we show that generated samples match the profiles of true samples, and display a marked shift in cellular composition between healthy and diseased samples. This in turn demonstrates the value of using a foundation model, as opposed to Euclidean generation on raw gene expression data or other lower-dimensional Euclidean embeddings. Since RWEFM operates directly in the spherical latent space of SCimilarity, it can leverage the functionality that the foundation model provides, such as accurate cell typing and batch effect robust embeddings. 

In Figure~\ref{fig:pascientflow_indv}, we present examples of individual samples generated with RWEFM for different tissue and disease combinations. For each generated sample, we match it with the closest sample in the real data by Riemannian Wasserstein distance. We observe that RWEFM can generate realistic tissue samples, showing great potential for downstream biological applications.

\subsection{Protein Torsion Angle Generation on Torus}

To generate the underlying data, we utilized the MD-CATH dataset \citep{mirarchi2024mdcath}, which contains Molecular Dynamics simulations for domain structures from the CATH database. For each protein in the dataset, we extracted the backbone torsion angles ($\phi, \psi$) for every residue at every time step of the simulation. We then aggregated these angle pairs across all residues and time points into a single collection for each protein. Since $\phi$ and $\psi$ are periodic, this process effectively converts each protein into a single empirical distribution over the flat torus $\mathbb{T}^2$.

We performed conditional generation by embedding the specific amino acid sequence of each protein using the ESM-2 protein language model to obtain a conditioning vector. We randomly withheld 20\% of the proteins as a test set. For these test proteins, we generated their torsion angle distributions (point clouds of size $N=2048$) conditioned on their ESM embeddings (\cref{fig:torus_torsion}). Finally, we compared the generated distributions to the ground-truth distributions derived from the MD simulations using standard distributional distance metrics on the torus (\cref{tab:torsion_generation}).

\begin{figure}[h]
\centering
\includegraphics[width=0.6\linewidth]{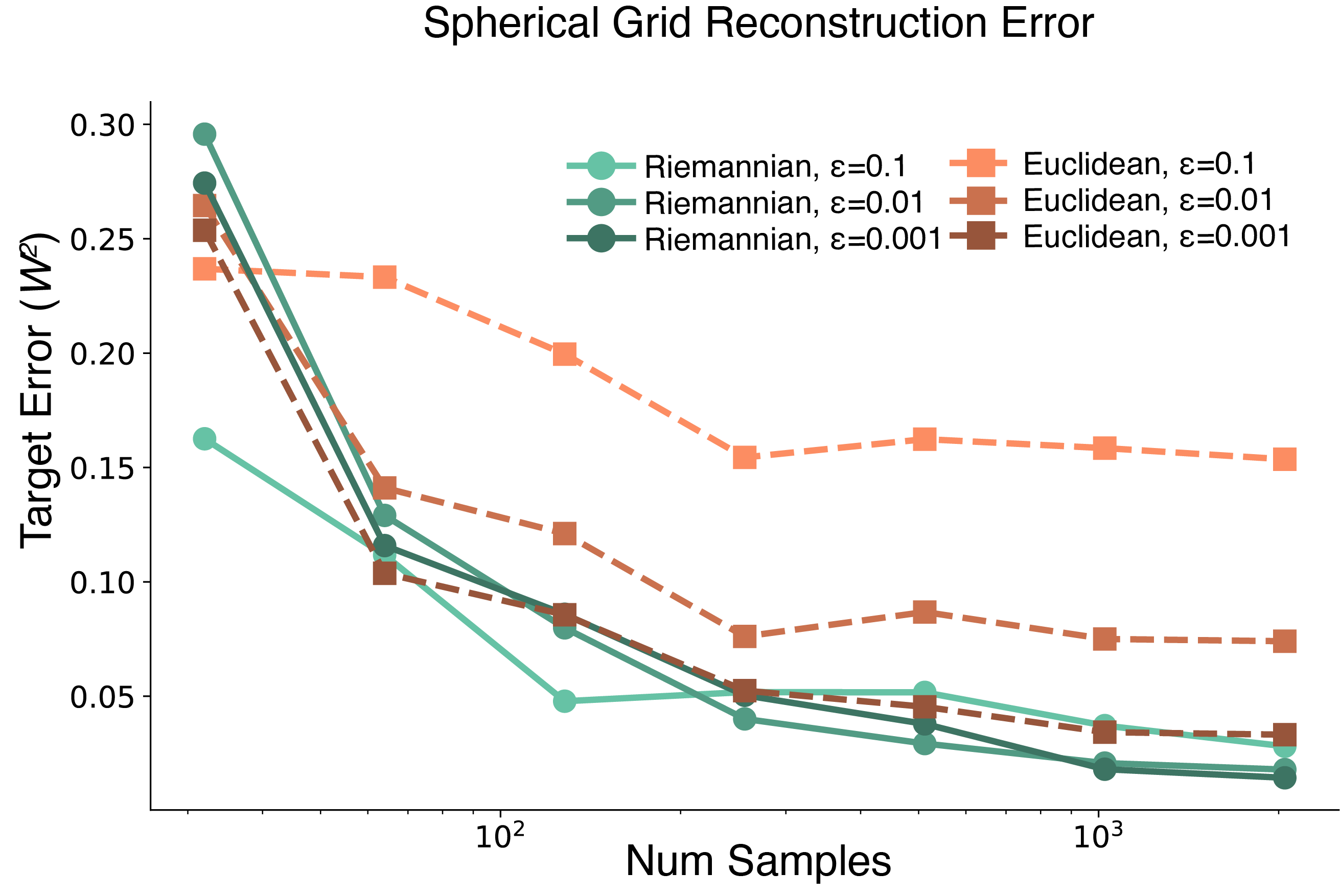}
\caption{\textbf{Entropic map reconstruction error} For the spherical grid example from \cref{fig:entropic_map_sphere_oos}, we benchmark the map quality of entropic and euclidean map estimators as a function of sample size and regularization strength. We report the (Spherical) EMD between the source samples pushed forward by the estimated map and the target samples.}
\label{fig:entropic_map_benchmark_map_error}
\end{figure}

\subsection{Benchmarking of Riemannian Entropic Map}
\label{app:benchmark_entropic}

In this manuscript we introduced the Riemannian analogue of the entropic optimal transport map estimator. In \cref{fig:entropic_map_benchmark}, we benchmark the accuracy and computational efficiency of this estimator against the Euclidean entropic map and the unregularized OT map on constructed examples on the sphere $\mathbb{S}^2$ and hyperbolic space $\mathbb{H}^2$. The ground-truth map is known in these examples, as it arises from shifting an original distribution via the Riemannian gradient of a convex function over each space. In both cases, the source is a (projected) Gaussian distribution centered at the north pole, and the target is obtained by applying the ground-truth map to the source samples. We vary the number of samples and the entropic regularization strength, and report the estimation error of each method in terms of the (spherical/hyperbolic) tangent norm between the estimated and ground-truth map velocities.

Next, we ask how does the quality of the pushed forward distribution vary as a function of map estimator, sample size and regularization strength. In \cref{fig:entropic_map_benchmark_map_error}, we report the (spherical) EMD between the source samples pushed forward by the estimated map and the target samples, for varying sample sizes and regularization strengths. Here too the source sample is a (projected) Gaussian centered at the north pole, and the target is the gridded sphere as shown in \cref{fig:entropic_map_sphere_oos}. In each experiment, we estimate the map from a limited number of samples, and use the out-of-sample extension of map to push forward the entire ($n=15,000$) set of source samples. We see that the Riemannian entropic map consistently outperforms the Euclidean entropic map across sample sizes and regularization strengths.

\subsection{Training Time versus Generation Quality Tradeoff}
\label{sec:tradeoff}

A key practical consideration when applying RWEFM is the tradeoff between computational cost and generation quality. To provide users with concrete guidance on this tradeoff, we conduct a comprehensive ablation study on the two primary hyperparameters that control both training efficiency and sample quality: the entropic regularization parameter $\varepsilon$ and the number of particles $n$ sampled from each distribution during training.

\begin{figure}[h]
\centering
\includegraphics[width=0.8\linewidth]{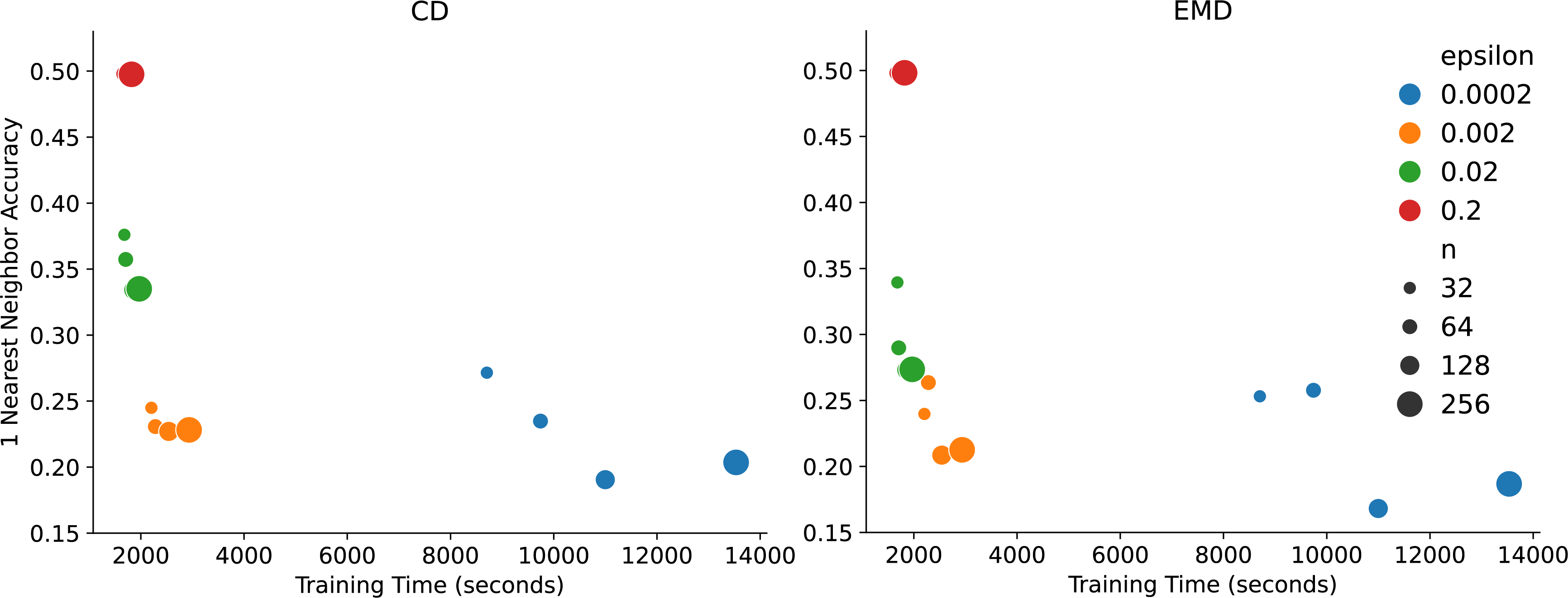}
\caption{\textbf{Training time versus generation quality tradeoff.} We benchmark RWEFM on generating MNIST digit 3 on $\mathbb{S}^2$, varying the entropic regularization parameter $\varepsilon$ and number of particles $n$ per distribution . The quantity plotted on both vertical axes is the classwise 1-NN deviation (1-NN-D), computed using Chamfer Distance and Earth Mover's Distance and shown against total training time in seconds. Lower scores indicate smaller classwise deviations from 50\% accuracy; the minimum is 0. As $\varepsilon$ decreases and $n$ increases, training time grows substantially due to increased Sinkhorn iterations required for convergence and larger network capacity needed to process more particles, but generation quality improves markedly as the estimated optimal transport map becomes more accurate. This demonstrates the concrete tradeoff between computational cost and sample quality that practitioners can tune based on their application requirements and computational budget.}
\label{fig:eps_n_benchmark}
\end{figure}

We benchmark RWEFM on the task of generating MNIST digit 3 on the sphere $\mathbb{S}^2$, systematically varying $\varepsilon \in \{0.0002, 0.002, 0.02, 0.2\}$ and $n \in \{32, 64, 128, 256\}$ across 16 experimental configurations. For each configuration, we train the model for 500,000 steps and measure both total wall-clock training time and generation quality on a held-out test set. Generation quality is assessed using classwise 1-NN deviation (1-NN-D) with both Chamfer Distance (CD) and Earth Mover's Distance (EMD) as ground metrics, where lower values indicate smaller classwise deviations from 50\% accuracy.

The results in \cref{fig:eps_n_benchmark} reveal a clear and predictable tradeoff. As $\varepsilon$ decreases, the entropic optimal transport map becomes less regularized and thus more accurate, requiring more Sinkhorn iterations to converge during training. Similarly, as $n$ increases, the model must process more particles per distribution, necessitating larger kernel sizes in the Sinkhorn algorithm and greater capacity in the transformer feedforward layers. Both factors increase training time—models with $\varepsilon = 0.0002$ and $n = 256$ take over 4 hours to train, compared to under an hour for $\varepsilon = 0.2$ and $n = 32$.

However, this computational investment yields substantial improvements in generation quality. The most accurate models use small $\varepsilon$ and large $n$, while the fastest models exhibit significantly worse quality. Interestingly, the relationship is roughly monotonic: intermediate configurations provide intermediate performance on both axes, allowing practitioners to select hyperparameters that balance their computational budget against their quality requirements. All experiments were conducted on a single NVIDIA B200 GPU, with the reported timings providing practitioners with realistic expectations for their own deployments. This analysis demonstrates that RWEFM offers a tunable spectrum of performance characteristics.

\subsection{Additional Metrics and Error Values}
\label{sec:additional_metrics}

These results complement the main-text benchmarks with additional ground metrics and variability estimates. We first report the overall MMD comparison and the single-cell results, then compare the sampled and entropic RWEFM variants across random seeds. Lower values are better for all metrics; 1-NN-D is reported as the classwise deviation from 0.5.

\begingroup
\raggedbottom
\setlength{\intextsep}{6pt}
\captionsetup[table]{font=small,skip=4pt}

\paragraph{Benchmark summaries}
\Cref{tab:mmd_mnist_emnist} complements the 1-NN-D results in \cref{tab:nna_mnist_emnist} with MMD under both CD and EMD ground metrics. The comparison includes point-cloud baselines alongside the flow-matching methods.

\begin{table}[H]
\centering
\caption{\textbf{MMD on MNIST, EMNIST, and KMNIST.} Extended metrics for \cref{tab:nna_mnist_emnist}, using Chamfer Distance (CD) and Earth Mover's Distance (EMD). Dashes indicate methods not evaluated on $\mathbb{T}^2$.}
\label{tab:mmd_mnist_emnist}
\resizebox{\textwidth}{!}{%
\begin{tabular}{lllcccccccc}
\toprule
& & & PVD & PSF & FM & SetFM & WFM & RFM & SetRFM & RWEFM \\
\midrule
$\mathbb{H}^2$ & h & CD  & $1.61 \cdot 10^{-3}$ & $3.50 \cdot 10^{-3}$ & $3.11 \cdot 10^{-2}$ & $7.79 \cdot 10^{-2}$ & $2.87 \cdot 10^{-2}$ & $4.10 \cdot 10^{-2}$ & $3.17 \cdot 10^{-3}$ & $3.39 \cdot 10^{-3}$ \\
               &   & EMD & $2.84 \cdot 10^{-3}$ & $5.26 \cdot 10^{-2}$ & $6.59 \cdot 10^{-2}$ & $5.94 \cdot 10^{-2}$ & $2.60 \cdot 10^{-2}$ & $5.68 \cdot 10^{-2}$ & $4.54 \cdot 10^{-3}$ & $2.81 \cdot 10^{-3}$ \\
\addlinespace
               & w & CD  & $2.23 \cdot 10^{-3}$ & $7.22 \cdot 10^{-3}$ & $6.77 \cdot 10^{-2}$ & $8.58 \cdot 10^{-2}$ & $6.87 \cdot 10^{-2}$ & $2.20 \cdot 10^{-2}$ & $1.04 \cdot 10^{-2}$ & $6.21 \cdot 10^{-3}$ \\
               &   & EMD & $6.22 \cdot 10^{-3}$ & $5.65 \cdot 10^{-2}$ & $1.29 \cdot 10^{-1}$ & $6.43 \cdot 10^{-2}$ & $7.54 \cdot 10^{-2}$ & $2.61 \cdot 10^{-2}$ & $9.91 \cdot 10^{-3}$ & $5.41 \cdot 10^{-3}$ \\
\addlinespace
               & y & CD  & $4.05 \cdot 10^{-4}$ & $3.88 \cdot 10^{-3}$ & $6.75 \cdot 10^{-2}$ & $5.42 \cdot 10^{-2}$ & $4.16 \cdot 10^{-2}$ & $6.76 \cdot 10^{-2}$ & $4.99 \cdot 10^{-3}$ & $2.33 \cdot 10^{-3}$ \\
               &   & EMD & $1.96 \cdot 10^{-3}$ & $6.22 \cdot 10^{-2}$ & $7.50 \cdot 10^{-2}$ & $4.12 \cdot 10^{-2}$ & $3.82 \cdot 10^{-2}$ & $7.13 \cdot 10^{-2}$ & $6.37 \cdot 10^{-3}$ & $2.28 \cdot 10^{-3}$ \\
\midrule
$\mathbb{S}^2$ & 3 & CD  & $4.05 \cdot 10^{-3}$ & $5.85 \cdot 10^{-3}$ & $2.17 \cdot 10^{-2}$ & $1.65 \cdot 10^{-2}$ & $1.20 \cdot 10^{-2}$ & $9.25 \cdot 10^{-3}$ & $2.47 \cdot 10^{-3}$ & $2.40 \cdot 10^{-3}$ \\
               &   & EMD & $7.81 \cdot 10^{-3}$ & $5.17 \cdot 10^{-2}$ & $2.15 \cdot 10^{-2}$ & $1.73 \cdot 10^{-2}$ & $1.16 \cdot 10^{-2}$ & $1.31 \cdot 10^{-2}$ & $3.16 \cdot 10^{-3}$ & $2.29 \cdot 10^{-3}$ \\
\addlinespace
               & 4 & CD  & $2.33 \cdot 10^{-3}$ & $4.53 \cdot 10^{-3}$ & $1.54 \cdot 10^{-2}$ & $2.80 \cdot 10^{-2}$ & $1.51 \cdot 10^{-2}$ & $9.23 \cdot 10^{-3}$ & $2.49 \cdot 10^{-3}$ & $2.55 \cdot 10^{-3}$ \\
               &   & EMD & $4.90 \cdot 10^{-3}$ & $2.24 \cdot 10^{-2}$ & $2.54 \cdot 10^{-2}$ & $2.55 \cdot 10^{-2}$ & $1.36 \cdot 10^{-2}$ & $2.13 \cdot 10^{-2}$ & $3.20 \cdot 10^{-3}$ & $2.38 \cdot 10^{-3}$ \\
\addlinespace
               & 8 & CD  & $5.92 \cdot 10^{-3}$ & $6.71 \cdot 10^{-3}$ & $1.63 \cdot 10^{-2}$ & $2.46 \cdot 10^{-2}$ & $1.88 \cdot 10^{-2}$ & $8.35 \cdot 10^{-3}$ & $2.48 \cdot 10^{-3}$ & $2.57 \cdot 10^{-3}$ \\
               &   & EMD & $1.03 \cdot 10^{-2}$ & $5.80 \cdot 10^{-2}$ & $2.60 \cdot 10^{-2}$ & $2.18 \cdot 10^{-2}$ & $1.60 \cdot 10^{-2}$ & $1.59 \cdot 10^{-2}$ & $3.81 \cdot 10^{-3}$ & $2.25 \cdot 10^{-3}$ \\
\midrule
$\mathbb{T}^2$ & ki & CD  & $-$ & $-$ & $4.57 \cdot 10^{-2}$ & $3.29 \cdot 10^{-2}$ & $1.18 \cdot 10^{-2}$ & $4.49 \cdot 10^{-2}$ & $1.28 \cdot 10^{-2}$ & $1.46 \cdot 10^{-2}$ \\
               &    & EMD & $-$ & $-$ & $4.51 \cdot 10^{-2}$ & $3.48 \cdot 10^{-2}$ & $1.39 \cdot 10^{-2}$ & $4.37 \cdot 10^{-2}$ & $1.26 \cdot 10^{-2}$ & $1.45 \cdot 10^{-2}$ \\
\addlinespace
               & na & CD  & $-$ & $-$ & $4.65 \cdot 10^{-2}$ & $5.91 \cdot 10^{-3}$ & $5.29 \cdot 10^{-3}$ & $4.50 \cdot 10^{-2}$ & $7.30 \cdot 10^{-3}$ & $1.12 \cdot 10^{-2}$ \\
               &    & EMD & $-$ & $-$ & $5.24 \cdot 10^{-2}$ & $5.72 \cdot 10^{-3}$ & $5.27 \cdot 10^{-3}$ & $5.28 \cdot 10^{-2}$ & $6.84 \cdot 10^{-3}$ & $1.04 \cdot 10^{-2}$ \\
\addlinespace
               & ma & CD  & $-$ & $-$ & $1.73 \cdot 10^{-2}$ & $6.47 \cdot 10^{-3}$ & $5.88 \cdot 10^{-3}$ & $1.63 \cdot 10^{-2}$ & $3.68 \cdot 10^{-2}$ & $7.12 \cdot 10^{-3}$ \\
               &    & EMD & $-$ & $-$ & $3.00 \cdot 10^{-2}$ & $9.49 \cdot 10^{-3}$ & $6.65 \cdot 10^{-3}$ & $2.95 \cdot 10^{-2}$ & $4.15 \cdot 10^{-2}$ & $7.92 \cdot 10^{-3}$ \\
\bottomrule
\end{tabular}%
}
\end{table}

\paragraph{Single-cell sample generation}
\Cref{tab:scrnaseq_cd} expands the single-cell benchmark in \cref{tab:scrnaseq_ot} to both ground metrics. Reporting 1-NN-D alongside MMD summarizes classwise nearest-neighbor label mixing and kernel-based distributional similarity.

\begin{table}[H]
\centering
\caption{\textbf{Extended metrics for scRNA-seq generation.} Whole-sample generation on $\mathbb{S}^{128-1}$, with mean $\pm$ standard deviation for classwise 1-NN deviation (1-NN-D) and MMD.}
\label{tab:scrnaseq_cd}
\small
\setlength{\tabcolsep}{4pt}
\begin{tabular}{lcccc}
\toprule
& \multicolumn{4}{c}{\textbf{Manifold: } $\mathbb{S}^{128-1}$} \\
\cmidrule(lr){2-5}
Method & 1-NN-D (EMD) & 1-NN-D (CD) & MMD (EMD) & MMD (CD) \\
\midrule
SetFM  & $0.2672 \pm 0.0218$ & $0.4460 \pm 0.0072$ & $0.1026 \pm 0.0410$ & $0.1174 \pm 0.0323$ \\
WFM    & $0.3714 \pm 0.0957$ & $0.2945 \pm 0.0442$ & $0.0324 \pm 0.0221$ & $0.0355 \pm 0.0162$ \\
\addlinespace
SetRFM & $0.4645 \pm 0.0251$ & $0.3754 \pm 0.0470$ & $0.0300 \pm 0.0068$ & $0.0341 \pm 0.0120$ \\
RWEFM  & $0.3039 \pm 0.0456$ & $0.1889 \pm 0.0588$ & $0.0140 \pm 0.0074$ & $0.0137 \pm 0.0072$ \\
\bottomrule
\end{tabular}%
\end{table}

\paragraph{Variability and transport-map variants}
\Cref{tab:nna_std_all,tab:mmd_std_all} report mean $\pm$ standard deviation over three seeds and five samplings per seed, with separate columns for sampled OT assignments and the barycentric entropic map.

\begin{table}[H]
\centering
\caption{\textbf{1-NN-D across manifolds.} Mean $\pm$ standard deviation of the classwise deviation from 0.5; lower is better. Dashes indicate unevaluated configurations.}
\label{tab:nna_std_all}
\resizebox{\textwidth}{!}{%
\begin{tabular}{lllcccccccc}
\toprule
& & & FM & RFM & SetFM & WFM & PSF & SetRFM & \shortstack{RWEFM\\(sample)} & \shortstack{RWEFM\\(entropic)} \\
\midrule
$\mathbb{H}^2$ & h & CD  & $0.492 \pm 0.003$ & $0.480 \pm 0.003$ & $0.372 \pm 0.055$ & $0.302 \pm 0.045$ & $0.179 \pm 0.002$ & $0.301 \pm 0.041$ & $0.192 \pm 0.019$ & $0.149 \pm 0.022$ \\
               &   & EMD & $0.487 \pm 0.003$ & $0.483 \pm 0.003$ & $0.355 \pm 0.060$ & $0.278 \pm 0.038$ & $0.325 \pm 0.001$ & $0.306 \pm 0.013$ & $0.157 \pm 0.019$ & $0.085 \pm 0.012$ \\
\addlinespace
               & w & CD  & $0.481 \pm 0.015$ & $0.436 \pm 0.011$ & $0.321 \pm 0.026$ & $0.412 \pm 0.014$ & $0.058 \pm 0.002$ & $0.337 \pm 0.009$ & $0.253 \pm 0.023$ & $0.229 \pm 0.021$ \\
               &   & EMD & $0.482 \pm 0.013$ & $0.446 \pm 0.007$ & $0.332 \pm 0.032$ & $0.393 \pm 0.024$ & $0.296 \pm 0.001$ & $0.329 \pm 0.011$ & $0.201 \pm 0.022$ & $0.104 \pm 0.023$ \\
\addlinespace
               & y & CD  & $0.478 \pm 0.004$ & $0.472 \pm 0.003$ & $0.337 \pm 0.026$ & $0.291 \pm 0.020$ & $0.206 \pm 0.002$ & $0.310 \pm 0.026$ & $0.222 \pm 0.013$ & $0.162 \pm 0.016$ \\
               &   & EMD & $0.483 \pm 0.004$ & $0.481 \pm 0.004$ & $0.326 \pm 0.016$ & $0.277 \pm 0.016$ & $0.309 \pm 0.000$ & $0.308 \pm 0.013$ & $0.201 \pm 0.012$ & $0.114 \pm 0.023$ \\
\midrule
$\mathbb{S}^2$ & 3 & CD  & $0.400 \pm 0.013$ & $0.295 \pm 0.016$ & $0.262 \pm 0.030$ & $0.290 \pm 0.015$ & $0.188 \pm 0.002$ & $0.377 \pm 0.023$ & $0.251 \pm 0.020$ & $0.182 \pm 0.017$ \\
               &   & EMD & $0.422 \pm 0.005$ & $0.382 \pm 0.007$ & $0.269 \pm 0.016$ & $0.285 \pm 0.011$ & $0.355 \pm 0.000$ & $0.379 \pm 0.011$ & $0.270 \pm 0.015$ & $0.179 \pm 0.015$ \\
\addlinespace
               & 4 & CD  & $0.452 \pm 0.011$ & $0.423 \pm 0.008$ & $0.291 \pm 0.029$ & $0.291 \pm 0.039$ & $0.166 \pm 0.001$ & $0.379 \pm 0.014$ & $0.245 \pm 0.017$ & $0.185 \pm 0.020$ \\
               &   & EMD & $0.464 \pm 0.006$ & $0.453 \pm 0.006$ & $0.289 \pm 0.028$ & $0.274 \pm 0.043$ & $0.290 \pm 0.001$ & $0.379 \pm 0.010$ & $0.265 \pm 0.016$ & $0.177 \pm 0.015$ \\
\addlinespace
               & 8 & CD  & $0.365 \pm 0.027$ & $0.332 \pm 0.027$ & $0.407 \pm 0.111$ & $0.308 \pm 0.075$ & $0.229 \pm 0.002$ & $0.392 \pm 0.025$ & $0.278 \pm 0.029$ & $0.158 \pm 0.011$ \\
               &   & EMD & $0.407 \pm 0.010$ & $0.383 \pm 0.007$ & $0.353 \pm 0.093$ & $0.367 \pm 0.079$ & $0.390 \pm 0.001$ & $0.395 \pm 0.014$ & $0.297 \pm 0.013$ & $0.163 \pm 0.024$ \\
\midrule
$\mathbb{T}^2$ & ki & CD  & $0.478 \pm 0.002$ & $0.475 \pm 0.003$ & $0.281 \pm 0.025$ & $0.219 \pm 0.014$ & $-$ & $0.261 \pm 0.035$ & $-$ & $0.171 \pm 0.018$ \\
               &    & EMD & $0.461 \pm 0.003$ & $0.459 \pm 0.003$ & $0.279 \pm 0.019$ & $0.172 \pm 0.013$ & $-$ & $0.279 \pm 0.023$ & $-$ & $0.176 \pm 0.015$ \\
\addlinespace
               & na & CD  & $0.490 \pm 0.001$ & $0.490 \pm 0.002$ & $0.240 \pm 0.014$ & $0.187 \pm 0.012$ & $-$ & $0.230 \pm 0.016$ & $-$ & $0.161 \pm 0.010$ \\
               &    & EMD & $0.483 \pm 0.001$ & $0.484 \pm 0.002$ & $0.235 \pm 0.017$ & $0.140 \pm 0.012$ & $-$ & $0.235 \pm 0.012$ & $-$ & $0.155 \pm 0.014$ \\
\addlinespace
               & ma & CD  & $0.469 \pm 0.003$ & $0.468 \pm 0.003$ & $0.276 \pm 0.020$ & $0.175 \pm 0.013$ & $-$ & $0.256 \pm 0.030$ & $-$ & $0.148 \pm 0.014$ \\
               &    & EMD & $0.471 \pm 0.003$ & $0.471 \pm 0.003$ & $0.294 \pm 0.011$ & $0.169 \pm 0.013$ & $-$ & $0.282 \pm 0.011$ & $-$ & $0.181 \pm 0.025$ \\
\bottomrule
\end{tabular}%
}
\end{table}

The MMD results below follow the same dataset and method ordering.

\begin{table}[H]
\centering
\caption{\textbf{MMD across manifolds.} Mean $\pm$ standard deviation under CD and EMD ground metrics; lower is better. Dashes indicate unevaluated configurations.}
\label{tab:mmd_std_all}
\resizebox{\textwidth}{!}{%
\begin{tabular}{lllcccccccc}
\toprule
& & & FM & RFM & SetFM & WFM & PSF & SetRFM & \shortstack{RWEFM\\(sample)} & \shortstack{RWEFM\\(entropic)} \\
\midrule
$\mathbb{H}^2$ & h & CD  & $0.0311 \pm 0.0104$ & $0.0410 \pm 0.0024$ & $0.0779 \pm 0.0424$ & $0.0287 \pm 0.0129$ & $0.0035 \pm 0.0000$ & $0.0032 \pm 0.0012$ & $0.0027 \pm 0.0017$ & $0.0034 \pm 0.0024$ \\
               &   & EMD & $0.0659 \pm 0.0053$ & $0.0568 \pm 0.0014$ & $0.0594 \pm 0.0323$ & $0.0260 \pm 0.0088$ & $0.0526 \pm 0.0000$ & $0.0045 \pm 0.0008$ & $0.0025 \pm 0.0020$ & $0.0028 \pm 0.0023$ \\
\addlinespace
               & w & CD  & $0.0677 \pm 0.0383$ & $0.0220 \pm 0.0038$ & $0.0858 \pm 0.0160$ & $0.0687 \pm 0.0050$ & $0.0072 \pm 0.0000$ & $0.0104 \pm 0.0059$ & $0.0051 \pm 0.0012$ & $0.0062 \pm 0.0030$ \\
               &   & EMD & $0.1289 \pm 0.1167$ & $0.0261 \pm 0.0015$ & $0.0643 \pm 0.0079$ & $0.0754 \pm 0.0089$ & $0.0565 \pm 0.0000$ & $0.0099 \pm 0.0037$ & $0.0039 \pm 0.0015$ & $0.0054 \pm 0.0020$ \\
\addlinespace
               & y & CD  & $0.0675 \pm 0.0178$ & $0.0676 \pm 0.0061$ & $0.0542 \pm 0.0166$ & $0.0416 \pm 0.0207$ & $0.0039 \pm 0.0000$ & $0.0050 \pm 0.0014$ & $0.0023 \pm 0.0006$ & $0.0023 \pm 0.0010$ \\
               &   & EMD & $0.0750 \pm 0.0102$ & $0.0713 \pm 0.0020$ & $0.0412 \pm 0.0096$ & $0.0382 \pm 0.0126$ & $0.0622 \pm 0.0000$ & $0.0064 \pm 0.0014$ & $0.0025 \pm 0.0016$ & $0.0023 \pm 0.0012$ \\
\midrule
$\mathbb{S}^2$ & 3 & CD  & $0.0217 \pm 0.0010$ & $0.0093 \pm 0.0009$ & $0.0165 \pm 0.0090$ & $0.0120 \pm 0.0034$ & $0.0059 \pm 0.0000$ & $0.0025 \pm 0.0007$ & $0.0034 \pm 0.0005$ & $0.0024 \pm 0.0003$ \\
               &   & EMD & $0.0215 \pm 0.0006$ & $0.0131 \pm 0.0009$ & $0.0173 \pm 0.0068$ & $0.0116 \pm 0.0038$ & $0.0517 \pm 0.0000$ & $0.0032 \pm 0.0005$ & $0.0037 \pm 0.0009$ & $0.0023 \pm 0.0007$ \\
\addlinespace
               & 4 & CD  & $0.0154 \pm 0.0035$ & $0.0092 \pm 0.0016$ & $0.0280 \pm 0.0111$ & $0.0151 \pm 0.0109$ & $0.0045 \pm 0.0000$ & $0.0025 \pm 0.0007$ & $0.0035 \pm 0.0005$ & $0.0025 \pm 0.0004$ \\
               &   & EMD & $0.0254 \pm 0.0019$ & $0.0213 \pm 0.0003$ & $0.0255 \pm 0.0092$ & $0.0136 \pm 0.0093$ & $0.0224 \pm 0.0000$ & $0.0032 \pm 0.0005$ & $0.0037 \pm 0.0009$ & $0.0024 \pm 0.0007$ \\
\addlinespace
               & 8 & CD  & $0.0163 \pm 0.0030$ & $0.0084 \pm 0.0025$ & $0.0246 \pm 0.0057$ & $0.0188 \pm 0.0040$ & $0.0067 \pm 0.0000$ & $0.0025 \pm 0.0006$ & $0.0031 \pm 0.0005$ & $0.0026 \pm 0.0003$ \\
               &   & EMD & $0.0260 \pm 0.0019$ & $0.0159 \pm 0.0009$ & $0.0218 \pm 0.0048$ & $0.0160 \pm 0.0066$ & $0.0580 \pm 0.0001$ & $0.0038 \pm 0.0006$ & $0.0035 \pm 0.0008$ & $0.0022 \pm 0.0004$ \\
\midrule
$\mathbb{T}^2$ & ki & CD  & $0.0457 \pm 0.0039$ & $0.0449 \pm 0.0031$ & $0.0329 \pm 0.0113$ & $0.0118 \pm 0.0049$ & $-$ & $0.0128 \pm 0.0040$ & $-$ & $0.0146 \pm 0.0058$ \\
               &    & EMD & $0.0451 \pm 0.0021$ & $0.0437 \pm 0.0020$ & $0.0348 \pm 0.0115$ & $0.0139 \pm 0.0058$ & $-$ & $0.0126 \pm 0.0062$ & $-$ & $0.0145 \pm 0.0049$ \\
\addlinespace
               & na & CD  & $0.0465 \pm 0.0020$ & $0.0450 \pm 0.0015$ & $0.0059 \pm 0.0026$ & $0.0053 \pm 0.0010$ & $-$ & $0.0073 \pm 0.0015$ & $-$ & $0.0112 \pm 0.0027$ \\
               &    & EMD & $0.0524 \pm 0.0009$ & $0.0528 \pm 0.0009$ & $0.0057 \pm 0.0022$ & $0.0053 \pm 0.0008$ & $-$ & $0.0068 \pm 0.0012$ & $-$ & $0.0104 \pm 0.0025$ \\
\addlinespace
               & ma & CD  & $0.0173 \pm 0.0018$ & $0.0163 \pm 0.0017$ & $0.0065 \pm 0.0015$ & $0.0059 \pm 0.0006$ & $-$ & $0.0368 \pm 0.0233$ & $-$ & $0.0071 \pm 0.0004$ \\
               &    & EMD & $0.0300 \pm 0.0012$ & $0.0295 \pm 0.0011$ & $0.0095 \pm 0.0011$ & $0.0067 \pm 0.0007$ & $-$ & $0.0415 \pm 0.0247$ & $-$ & $0.0079 \pm 0.0009$ \\
\bottomrule
\end{tabular}%
}
\end{table}

\endgroup
\clearpage

\end{document}